\documentclass{article}
\usepackage{ijcai21}

\usepackage{times}
\usepackage{soul}
\usepackage{url}
\usepackage[hidelinks]{hyperref}
\usepackage[utf8]{inputenc}
\usepackage[small]{caption}
\usepackage{graphicx}
\usepackage{amsmath}
\usepackage{amsthm}
\usepackage{booktabs}
\usepackage{algorithm}
\usepackage{algorithmic}
\newcommand{\boldtheta}{\boldsymbol{\theta}}
\newcommand{\boldphi}{\boldsymbol{\phi}}

\usepackage{mathtools}
\usepackage{enumitem}[shortlabels]
\usepackage{algorithm,algorithmic}
\usepackage{subcaption}
\usepackage{multirow}

\newtheorem{proposition}{Proposition}
\newtheorem*{proposition*}{Proposition}
\usepackage{amsfonts}
\usepackage{amssymb}
\usepackage{comment}
\usepackage[switch]{lineno}
\usepackage{natbib}

\usepackage{xcolor}
\usepackage{soul}

\usepackage{xcolor}

\usepackage{tikz}
\newcommand*\circled[1]{\tikz[baseline=(char.base)]{
            \node[shape=circle,draw,inner sep=.5pt] (char) {#1};}}

\usepackage{pifont}

\title{Monte Carlo Filtering Objectives: A New Family of Variational Objectives to Learn Generative Model and Neural Adaptive Proposal for Time Series}

\author{
Shuangshuang Chen$^{1,2}$\footnote{Contact Author}\and
Sihao Ding$^2$\and
Yiannis Karayiannidis$^{3}$\And
Mårten Björkman$^1$\\
\affiliations
$^1$ Royal Institute of Technology\\
$^2$ Volvo Car Corporation\\
$^3$ Chalmers University of Technology\\
\emails
\{shuche, celle\}@kth.se,
sihao.ding@volvocars.com,
yiannis@chalmers.se
}

\begin{document}
%\linenumbers

\maketitle

\begin{abstract}
Learning generative models and inferring latent trajectories have shown to be challenging for time series due to the intractable marginal likelihoods of flexible generative models. It can be addressed by surrogate objectives for optimization. We propose Monte Carlo filtering objectives (MCFOs), a family of variational objectives for jointly learning parametric generative models and amortized adaptive importance proposals of time series. MCFOs extend the choices of likelihood estimators beyond Sequential Monte Carlo in state-of-the-art objectives, possess important properties revealing the factors for the tightness of objectives, and allow for less biased and variant gradient estimates. We demonstrate that the proposed MCFOs and gradient estimations lead to efficient and stable model learning, and learned generative models well explain data and importance proposals are more sample efficient on various kinds of time series data.%Linear Gaussian State Space Models (LGSSMs), nonlinear, non-Gaussian, high dimensional SSMs and non-Markovian sequences. 
\end{abstract}

\frenchspacing

\section{Introduction}
Learning a generative model with latent variables for time series is of interest in many applications. However, exact inference and marginalization are often intractable for flexible generative models, making it challenging to learn. There are a few popular approaches to circumvent these difficulties: implicit methods that learn generative models by comparing generated samples to data distributions like Generative Adversarial Networks (GANs) \citep{Ian2014gan}; and explicit methods that define surrogate objectives of the intractable marginal log-likelihood like Variational Autoencoder (VAEs) \citep{kingma2014auto}, or tractable marginals by invertible transformations like Normalizing Flows (NFs) \citep{rezende2015variational}. %Although implicit methods have recently gained much popularity due to its successes especially on image generation and style transfer, 
Explicit methods are often preferable when latent/encoded information is of importance, e.g. filtering and smoothing problems for some subsequent tasks. In this work, we mainly focus on the second approach and propose a family of surrogate filtering objectives to learn generative models and adaptive importance proposal models %to infer latent trajectories 
for time series.

Researchers have introduced various surrogate objectives using variational approximations of intractable posterior for time series, known as evidence lower bounds (ELBOs), such as STONE \citep{bayer2014learning}, VRNN \citep{chung2015recurrent}, SRNN \citep{fraccaro2016sequential}, DKF \citep{krishnan2017structured}, KVAE \citep{fraccaro2017disentangled}%, and AVF \citep{marino2018general}
. However, they typically suffer from a general issue
caused by the limited flexibility of the variational approximations, thus restricting the learning of generative models.
%in variational inference that optimizing ELBOs restricts the capacity of generative models by the . 
To alleviate this constraint, IWAE \citep{burda2015importance} proposes a tighter objective by averaging importance weights of multiple samples drawn from a variational approximation. 
Monte Carlo objectives (MCOs) \citep{mnih2016variational} generalizes the IWAE objective and ELBOs for non-sequential data. AESMC \citep{le2017auto}, FIVO \citep{maddison2017filtering}, and VSMC \citep{naesseth2018variational} extend this idea for sequential data using the estimators by Sequential Monte Carlo (SMC) and propose closely related surrogates objectives for learning. %of the joint marginal log-likelihood using Sequential Monte Carlo (SMC).

Inspired by MCOs and the sequential variants, 
%In this work, 
we propose Monte Carlo filtering objectives (MCFOs), a new family of surrogate objectives for generative models of time series, that
\begin{itemize}
    %\item alleviate the limited choice of estimator
    \item broadens previously limited choices of estimators for time series other than SMC,
    \item possesses unique properties such as monotonic convergence and asymptotic bias, revealing the factors that determine a tighter objective: the number of samples and importance proposals,
    \item reduces high variance in gradient estimates of proposal models common in state-of-the-art algorithms,
    without introducing bias, which allows for faster convergence and sample efficient proposal models. %less biased and variant gradient estimation for the new objectives, which resolves %the problematic gradient estimator 
    %limitations .
    %\item can be extended to apply in an online learning manner and to handle arbitrarily long sequences efficiently with some adequate assumptions on the generative model.
\end{itemize}

%Although these methods report tighter objectives than IWAE, they are limited to SMC for likelihood estimation and naturally inherit its weaknesses. 

%We show empirically that heuristically omitting the high variance term in gradient estimates, as proposed, makes learning stable at the cost of introducing an extra bias that is detrimental for both generative and amortized importance proposal models. 

The paper is organized as follows: we first review the definition of MCOs and discuss common limitations of %state-of-the-art 
existing filtering variants. In Section \ref{sec:mcfo}, we derive MCFOs, %as a family of surrogate objectives for joint marginal log-likelihoods of sequences. Some
explain their relations to other objectives and important properties. We demonstrate two instances of MCFOs with SMC and Particle Independence Metropolis-Hasting (PIMH) to learn models on 1) Linear Gaussian State Space Models (LGSSMs), 2) nonlinear, non-Gaussian, high dimensional SSMs of video sequences, 3) non-Markovian music sequences.

\section{Background}
\subsection{Monte Carlo Objectives}
\label{sec:mco}
For a generative model with observation $\mathbf{x}$ and latent state $\mathbf{z}$, a Monte Carlo objective (MCO) \citep{mnih2016variational} is defined as an estimate of the marginal log-likelihood $\log p(\mathbf{x})$ by samples drawn from a proposal distribution $q$:%\vspace{-1mm}
%joint likelihood  could be considered as \textit{unnormalized} distribution $\Tilde{P}$, inference of latent variable $p(\mathbf{x}|\mathbf{z})$ as \textit{normalized} distribution, $P$, and marginal distribution  $p(\mathbf{x})$ is  normalization constant.
\begin{equation}
    \mathbb{E}_{q(\mathbf{z})}[\log R] =  \log p(\mathbf{x}) - \mathbb{E}_{q(\mathbf{z})}[\log \frac{p(\mathbf{x})}{R}] \leq \log p(\mathbf{x}),
    \label{eq:stochastic_bound}
\end{equation}
where $R$ is any unbiased estimator of $p(\mathbf{x})$ that $\mathbb{E}[R] = p(\mathbf{x})$. It is also a lower bound of $\log p(\mathbf{x})$ as can be shown using Jensen's inequality. When an estimator takes a single sample from $q$ and $R = {p(\mathbf{x}, \mathbf{z})}/{q(\mathbf{z})}$, the MCO %is known as the Evidence Lower Bounds (ELBOs) in variational inference.
can be identified as ELBO in variational inference \citep{mnih2014neural}. When an estimator averages importance weights from $K$ samples, $R^K=K^{-1}\sum_{i=1}^K {p(\mathbf{z}^i, \mathbf{x})}/{q(\mathbf{z}^i)}$, %R_K = \frac{1}{K}\sum_{i=1}^K \frac{p(\mathbf{z}^i, \mathbf{x})}{q(\mathbf{z}^i)}$. 
it yields an importance weighted ELBO (IW-ELBO) \citep{burda2015importance, domke2018importance}.
This bound is proven to be tighter with increasing $K$ and asymptotically converges to $\log p(\mathbf{x})$, as $K\rightarrow \infty$. %\citep{burda2015importance}. 

\subsection{Sequential Monte Carlo}
\label{sec:smc}
For a sequential observation $\mathbf{x}_{1:T}$ with latent trajectory $\mathbf{z}_{1:T}$, the generative process %of the observation
can be factorized as $p(\mathbf{z}_{1:T}, \mathbf{x}_{1:T}) = p(\mathbf{z}_1)p(\mathbf{x}_1|\mathbf{z}_1) \prod_{t=2}^T p(\mathbf{z}_{t}|\mathbf{z}_{1:t-1}, \mathbf{x}_{1:t-1}) p(\mathbf{x}_{t}|\mathbf{z}_{1:t}, \mathbf{x}_{1:t-1})$. Inferring latent trajectory, $p(\mathbf{z}_{1:T}|\mathbf{x}_{1:T})$, is of importance for marginalization and learning of generative models, however, it is usually intractable. 

Sequential Monte Carlo (SMC) \citep{doucet2009tutorial} approximates a target distribution, specifically $p(\mathbf{z}_{1:T}|\mathbf{x}_{1:T})$, %for the generative model
using a set of weighted sample trajectories $\{ \Tilde{w}_{T}^{i}, \hat{\mathbf{z}}_{1:T}^{i} \}_{i=1:K}$. It combines Sequential Importance Sampling (SIS) with resampling, and consists of four main steps: 
\begin{itemize}
    \item[] \textit{Sample} $K$ particles $\hat{\mathbf{z}}^{i}_t$ from proposal $q(\mathbf{z}_t|\bar{\mathbf{z}}^{i}_{1:t-1},\mathbf{x}_{1:t})$ with previously resampled trajectories $\bar{\mathbf{z}}^{i}_{1:t-1}$;
    \item[] \textit{Append} to trajectory $\hat{\mathbf{z}}^{i}_{1:t}=(\hat{\mathbf{z}}^{i}_t, \bar{\mathbf{z}}^{i}_{1:t-1})$; 
    \item[] \textit{Weight} trajectories with $\Tilde{w}_t^{i}={w_t^{i}}/{\sum_{j=1}^K w_t^{j}}$, where \\ $w_t^{i} = {p(\mathbf{x}_{1:t}, \hat{\mathbf{z}}_{1:t}^{i}|\mathbf{x}_{1:t-1}, \bar{\mathbf{z}}^{i}_{1:t-1})}/{q(\hat{\mathbf{z}}^{i}_t|\bar{\mathbf{z}}^{i}_{1:t-1},\mathbf{x}_{1:t})}$;
    \item[] \textit{Resample} 
    from $\{\Tilde{w}_t^{i}, \hat{\mathbf{z}}_{1:t}^{i} \}$ to obtain equally-weighted particles %$\{\frac{1}{K},
    $\bar{\mathbf{z}}^{i}_{1:t} = \hat{\mathbf{z}}_{1:t}^{A_{t-1}^i}$,
    with ancestral indices $A^i_{t-1}$. 
    %\}$.
\end{itemize}
This iteration continues until time $T$. Besides being an approximate inference, SMC also gives an unbiased estimate of the marginal likelihood $p(\mathbf{x}_{1:T})$ by the importance weights:  %\vspace{-1mm}
\begin{equation}
    \hat{p}(\mathbf{x}_{1:T}) = \prod_{t=1}^T \left( \frac{1}{K} \sum_{i=1}^K w_t^i \right).
    \label{eq:smc_estimator}
\end{equation}
%although resampling increases the intermediate variance, 
%\textit{degeneracy} problem in SIS that when only a single or few trajectories dominate, the total variance of estimates becomes extremely large. 
%Furthermore, 
The variance of this estimate, accessed by the so-called Effective Sample Size (ESS) for sample efficiency, is largely dependent on the proposal distributions $q$. % and it is usually difficult to scale up SMC to higher dimensions due to inefficient approximation of complex posterior distributions. 
We refer to \citep{doucet2009tutorial} for a more in-depth discussion.

\subsection{Variational Filtering Objectives}
To learn a generative model for time series data, various ELBO-like surrogate objectives have been proposed %such as STONE \citep{bayer2014learning}, VRNN \citep{chung2015recurrent}, SRNN \citep{fraccaro2016sequential}, DKF \citep{karl2016deep}, KVAE \citep{fraccaro2017disentangled} and AVF \citep{marino2018general}, 
% ALREADY MENTIONED
using different factorizations of generative models and approximations $q(\mathbf{z}_{1:T}|\mathbf{x}_{1:T})$. IW-ELBO %\citep{burda2015importance} 
can be extended to sequences by changing from Importance Sampling (IS) to SIS estimator $R^K=K^{-1}\sum_{i=1}^K {p(\mathbf{z}_{1:T}^i, \mathbf{x}_{1:T})}/{q(\mathbf{z}_{1:T}^i)}$. However, such an estimator suffers from exponential growth of variance with the %increasing 
length of sequences.%  in SIS.
% the high variance problem of SIS. 
To improve this, AESMC, FIVO and VSMC %AESMC \citep{le2017auto}, VSMC \citep{naesseth2018variational} and FIVO \citep{maddison2017filtering}, 
propose three closely related MCOs, exploiting the SMC estimators (\ref{eq:smc_estimator}): %\vspace{-1mm}
\begin{equation}
    \begin{aligned}
        & \text{ELBO}_{\text{SMC}} = \mathbb{E}_{Q_{\text{SMC}}} \bigg[ \sum_{t=1}^T \log  \underbrace{\left( \frac{1}{K} \sum_{i=1}^K w_t^i \right)}_{R_t^K} \bigg]  \\
        & Q_{\text{SMC}}(z_{1:T}^{1:K}) = \int \left( \prod_{i=1}^K q(z_1^i) \right) \\ & \quad \quad \quad \quad \prod_{t=2}^T \prod_{i=1}^K \left(q(z_t^i|z_{1:t-1}^{A_{t-1}^i}) \cdot 
        \frac{w_{t-1}^{A_{t-1}^{i}}}{\sum_j w_{t-1}^{j}}
        %\text{Discrete}(A_{t-1}^k|w_{t-1}^{1:K}) 
        \right) d  A_{1:T-1}^{1:K}.
    \end{aligned}
    \label{eq:ELBO_AESMC}
\end{equation}
%The main differences between these objectives lie in the implementation of SMCs; see details in Appendix A.1. % (move to appendix) AESMC implements the classic SMC \citep[Section~3.5]{doucet2009tutorial}; VSMC assumes Markovian latent variables and conditional independence of observation so that proposals and conditionals in the generative model are simplified; FIVO defines its objective by adaptive resampling SMC \citep[Section~3.5]{doucet2009tutorial}, but learns with the fixed schedule of resampling as AESMC. 
It is found that the learning of generative models via the objective suffers high variance in gradient estimation% of proposal parameters
, since %the categorical distribution of 
importance weight $w_t^i$ does not allow for smooth gradient computation. We show that simply ignoring some high variance term in the gradient estimate as the suggested solution in previous methods, introduces an extra bias and leads to non-optimality of proposal and generative parameters at convergence. To tackle these problems and extend variational filtering objectives% to more estimators
, we propose MCFOs, %a family of variational filtering objectives, 
and discuss their important properties in the following sections. %We examine their important properties and discuss the choice of importance proposal on the tightness of MCFOs and their relations to previous objectives and how to optimize MCFOs to learn parametric generative and adaptive proposal models from data.

\section{Monte Carlo Filtering Objectives}
\label{sec:mcfo}
\iffalse
In this section, we propose MCFOs, a family of variational filtering objectives, inspired by MCOs and FIVOs. We examine their important properties and discuss the choice of importance proposal on the tightness of MCFOs and their relations to previous objectives and how to optimize MCFOs to learn parametric generative and adaptive proposal models from data.
\fi
%\subsection{Definition and Properties}
Instead of constructing variational lower bound from (\ref{eq:smc_estimator}), we leverage the following decomposition of joint marginal log-likelihood:
%The key to efficiently compute maximum likelihoods of generative models for time series is to decompose %the joint marginal log-likelihood 
%it:
\vspace{-1mm}
\begin{equation}
    \log p(\mathbf{x}_{1:T}) = \log p(\mathbf{x}_1) + \sum_{t=2}^T \log p(\mathbf{x}_t|\mathbf{x}_{1:t-1}).
    \label{eq:recursion}
\end{equation}
%so that each $\log p(\mathbf{x}_t|\mathbf{x}_{1:t-1})$ can be evaluated recursively for arbitrarily long sequences. 
Nonetheless, $\log p(\mathbf{x}_t|\mathbf{x}_{1:t-1})$ is usually intractable, which makes learning a generative model by maximizing (\ref{eq:recursion}) only possible in some limited cases. % of arbitrary length.
%where the marginal likelihood $Z_t := p(\mathbf{x}_{1:t})$. Instead of using unbiased estimator by SMC to define lower bound on whole trajectory $p(\mathbf{x}_{1:T})$ as AESMC/VSMC/FIVO, 
Instead, we define $\mathcal{L}_t^K$, an MCO to $\log p(\mathbf{x}_t|\mathbf{x}_{1:t-1})$ with $K$ samples: %\vspace{-1mm}
\begin{equation}
\begin{aligned}
    &  \mathcal{L}_t^K = %\mathbb{E}_{Q_t^K}\left[
    %{p(\mathbf{z}_{1:t-1}^{i}|\mathbf{x}_{1:t-1}) q(\mathbf{z}_{t}^{i}|\mathbf{z}_{1:t-1}^{i},\mathbf{x}_{1:t})} 
    %\log \left( \frac{1}{K} \sum_{i=1}^K \frac{p(\mathbf{z}_{1:t}^{i},\mathbf{x}_{1:t})}{p(\mathbf{z}_{1:t-1}^{i},\mathbf{x}_{1:t-1})q(\mathbf{z}_{t}^{i}|\mathbf{z}_{1:t-1}^{i},\mathbf{x}_{1:t})} \right) \right],
    \mathbb{E}_{Q_t^K}[\log R_t^K], \\ 
   & Q_t^K(\mathbf{z}_{1:t}^{1:K}| \mathbf{x}_{1:t})
        = p(\mathbf{z}_{1:t-1}^{1:K}|\mathbf{x}_{1:t-1}) \cdot \prod_{i=1}^K  q(\mathbf{z}_{t}^{i}|\mathbf{z}_{1:t-1}^{i},\mathbf{x}_{1:t}), \\
    & R_t^K%(\mathbf{z}_{1:t}^{1:K}, \mathbf{x}_{1:t})
        = \frac{1}{K} \sum_{i=1}^K \frac{p(\mathbf{z}_{1:t}^{i},\mathbf{x}_{1:t})}{p(\mathbf{z}_{1:t-1}^{i},\mathbf{x}_{1:t-1})q(\mathbf{z}_{t}^{i}|\mathbf{z}_{1:t-1}^{i},\mathbf{x}_{1:t})},\\
    & \text{specifically $\mathcal{L}_1^K$ for $\log p(\mathbf{x}_1)$,}\\
    & Q_1^K(\mathbf{z}_{1}^{1:K}| \mathbf{x}_{1})
        = \prod_{i=1}^K q(\mathbf{z}_1^{i}|\mathbf{x}_1), ~~
        R_1^K%(\mathbf{z}_{1}^{1:K}, \mathbf{x}_{1})
        = \frac{1}{K} \sum_{i=1}^K \frac{p(\mathbf{z}_1^{i}, \mathbf{x}_1)}{q(\mathbf{z}_1^{i}|\mathbf{x}_1)},
\end{aligned}
\label{eq:L_t_k}
\end{equation}
\vspace{-2mm}

\noindent 
where %$Q_t^K$ %= {p(\mathbf{z}_{1:t-1}|\mathbf{x}_{1:t-1}) q(\mathbf{z}_{t}|\mathbf{z}_{1:t-1},\mathbf{x}_{1:t})}$ is the proposal distribution of the joint posterior $p(\mathbf{z}_{1:t}|\mathbf{x}_{1:t})$, and
$q(\mathbf{z}_{t}|\mathbf{z}_{1:t-1},\mathbf{x}_{1:t})$ and $q(\mathbf{z}_{1}|\mathbf{x}_{1})$ are the proposal distributions% of $p(\mathbf{z}_{t}|\mathbf{z}_{1:t-1},\mathbf{x}_{1:t})$
, $R_t^K$ and $R_1^K$ are the unbiased estimators of %conditional marginals 
$p(\mathbf{x}_{t}|\mathbf{x}_{1:t-1})$ and $p(\mathbf{x}_{1})$ using $K$ samples; see Appendix \ref{sec:appendix_lower_bound} for derivations. Summing up the series of MCOs in (\ref{eq:L_t_k}), we then define a Monte Carlo filtering objective (MCFO), $\mathcal{L}_{\text{MCFO}}$:
\begin{equation*}
    \begin{aligned}
        & \mathcal{L}_{\text{MCFO}}^K (\mathbf{x}_{1:T}, p, q) %\\
        %& \quad \quad \quad \quad \quad \quad 
        = \sum_{t=1}^T \mathcal{L}_t^K \leq \log p(\mathbf{x}_{1:T}) 
        , %= \sum_{t=1}^T \mathbb{E}_{Q_t^K}[\log R_t^K],\\
        %& Q_1^K%(\mathbf{z}_{1}^{1:K}, \mathbf{x}_{1}) = \prod_{i=1}^K q(\mathbf{z}_1^{i}|\mathbf{x}_1), %R_1^K%(\mathbf{z}_{1}^{1:K}, \mathbf{x}_{1}) = \frac{1}{K} \sum_{i=1}^K \frac{p(\mathbf{z}_1^{i}, \mathbf{x}_1)}{q(\mathbf{z}_1^{i}|\mathbf{x}_1)},
        \\ 
        %& Q_t^K%(\mathbf{z}_{1:t}^{1:K}, \mathbf{x}_{1:t}) = p(\mathbf{z}_{1:t-1}^{1:K}|\mathbf{x}_{1:t-1}) \cdot \prod_{i=1}^K  q(\mathbf{z}_{t}^{i}|\mathbf{z}_{1:t-1}^{i},\mathbf{x}_{1:t}). \\
        %& R_t^K%(\mathbf{z}_{1:t}^{1:K}, \mathbf{x}_{1:t})
        %= \frac{1}{K} \sum_{i=1}^K \frac{p(\mathbf{z}_{1:t}^{i},\mathbf{x}_{1:t})}{p(\mathbf{z}_{1:t-1}^{i},\mathbf{x}_{1:t-1})q(\mathbf{z}_{t}^{i}|\mathbf{z}_{1:t-1}^{i},\mathbf{x}_{1:t})}.
    \end{aligned} 
\end{equation*}
as the lower bound of $\log p(\mathbf{x}_{1:T})$. To avoid notation clutter, we leave out arguments on $\mathcal{L}_{\text{MCFO}}^K$ and $Q_t^K$ %$R_t^K$
whenever the context permits. 

Considering the filtering problem for which future observations have no impact on the current posterior, replacing $p(\mathbf{z}_{1:t-1}^{1:K}|\mathbf{x}_{1:t-1})$ in $Q_t^K$ %(\mathbf{z}_{1:t}^{1:K}|\mathbf{x}_{1:t})$ 
with $K$ sample approximations $\hat{p}(\mathbf{z}_{1:t-1}^{1:K}|\mathbf{x}_{1:t-1}) = \sum_{i=1}^K \Tilde{w}_{t-1}^{i} \delta(\mathbf{z}_{1:t-1}^{i}-\hat{\mathbf{z}}_{1:t-1}^{i})$ retrieves the definition of $\text{ELBO}_{\text{SMC}}$ in (\ref{eq:ELBO_AESMC}). %the objectives in AESMC/FIVO/VSMC. %(where $\hat{\mathbf{z}}_{1:t}^{1:K}$ are sample trajectories before resampling step in SMC) 
%With this approximation, these 
The objective can be considered as an estimate of MCFOs by SMC and is consistent to MCFOs with the asymptotic bias of $\mathcal{O}(1/K)$; see Appendix \ref{sec:fivo_and_mcfo} for detailed discussion. %Although MCFOs %are commonly intractable and 
%cannot be unbiasedly estimated by estimators like SMC, they are tractable and practical, e.g. using the posterior by normalizing flows \citep{rezende2015variational}. 
MCFOs can freely choose other estimator alternatives such as PIMH \citep{andrieu2010particle} and unbiased MCMC with couplings \citep{jacob2017unbiased} to further improve sampling efficiency of SMC. %and scale to high dimensional problems. Later, we show that MCFOs facilitate the implementation of less biased and variant gradient estimators, which lead to faster and more stable learning of generative and proposal models.

\subsection{Properties of MCFOs}
Except for the general properties inherited from MCOs such as \textit{bound} and \textit{consistency}, the convergence of MCFOs %to $\log p(\mathbf{x}_{1:T})$
is \textit{monotonic} like IW-ELBO, but unique to earlier filtering objectives. Additionally, the asymptotic bias of MCFOs can be shown to relate to the total variances of estimators. %by extending \citep[Theorem~3]{domke2018importance}. %\vspace{-1mm}
\begin{proposition}
\label{prop:properties}
\textnormal{(Properties of MCFOs)}. Let $\mathcal{L}_{\text{MCFO}}^K$ be an MCFO of $\log p(\mathbf{x}_{1:T}) $ by a series of unbiased estimators $R_t$ of $p(\mathbf{x}_t|\mathbf{x}_{1:t-1})$ using $K$ samples. Then, 
\begin{enumerate}[label=\alph*)] % (a), (b), (c), ...
\item \label{prop:properties_1} \textnormal{(Bound)} $\log p(\mathbf{x}_{1:T}) \geq \mathcal{L}_{\text{MCFO}}^K$.
\item \label{prop:properties_2} \textnormal{(Monotonic convergence)} $\mathcal{L}_{\text{MCFO}}^{K+1} \geq \mathcal{L}_{\text{MCFO}}^{K} \geq  \hdots \geq  \mathcal{L}_{\text{MCFO}}^{1}$.
\item \label{prop:properties_3} \textnormal{(Consistency)} If $p(\mathbf{z}_{1},\mathbf{x}_{1}) /q(\mathbf{z}_{1}|\mathbf{x}_{1})$ and ${p(\mathbf{z}_{1:t},\mathbf{x}_{1:t})}/{(p(\mathbf{z}_{1:t-1},\mathbf{x}_{1:t-1})q(\mathbf{z}_{t}|\mathbf{z}_{1:t-1},\mathbf{x}_{1:t}))}$ for all $t\in[2,T]$ are bounded, then $\mathcal{L}_{\text{MCFO}}^K \rightarrow \log p(\mathbf{x}_{1:T})$ as K $\rightarrow \infty$.
\item \label{prop:properties_4} \textnormal{(Asymptotic Bias)} For a large K, the bias of bound is related to the variance of estimator $R_t$, $\mathbb{V}[R_t]$, 
\[\lim_{K\rightarrow \infty} K(\log p(\mathbf{x}_{1:T}) - \mathcal{L}_{\text{MCFO}}^K) = \sum_{t=1}^T \frac{\mathbb{V}[R_t]}{2 p(\mathbf{x}_t|\mathbf{x}_{1:t-1})^2}. \]
\end{enumerate}
\end{proposition}
\proof See Appendix \ref{sec:appendix_lower_bound_property}.

Although increasing the number of samples $K$ leads to a tighter MCFO, a too large $K$ is infeasible in terms of computational and memory.
%As shown by \textit{monotonic convergence}, increasing the number of samples $K$ leads to a tighter MCFO. Although a larger $K$ is favorable for more exact estimates of the log-likelihood, a too large $K$ is infeasible considering computational and memory cost. 
It has also been shown that larger $K$ may deteriorate to learn proposals \citep{rainforth2018tighter}. An appropriate $K$ is a critical hyperparameters that affects learning and inference. On the other hand, \textit{asymptotic bias} suggests another way for a tighter bound, i.e.~using less variant estimators $R_t$, which is something that has been overlooked in recent literature. It explains why the bounds defined by SMC are tighter than IW-ELBO by SIS. Thus, a proposal model that permits less variant $R_t$, either designed or learned, is another key instrument. %for both learning and inference. 

%Therefore, the choice of importance proposals the variance of estimators is determined by . Therefore, a poor proposal leads to poor sample efficiency, a weak approximation of true posterior, and higher variance of estimates $although some negative effects could be mitigated by increasing the number of samples and/or applying additional MCMC moves \citep{doucet2009tutorial} which have a computational cost. 

% SMC would outperform with SIS, as resampling in SMC reduces the variance sum of estimators. %This makes SMC and other Monte Carlo Markov Chain (MCMC) methods preferred over SIS for tighter bounds. %To be noted, the asymptotic bias only applies when K is adequately large so it might not always apply at small K regime \citet{rainforth2018tighter}.

\subsection{Optimal Importance Proposals}
\label{sec:importance_proposal}

Considering proposals $q$ as an argument of MCFOs, we can derive the optimal proposals when they maximize the bound.
\begin{proposition}
\textnormal{(Optimal importance proposals $q^*$ for an MCFO)}. The bound is maximized and exact to $\log p(\mathbf{x}_{1:T})$ when the importance proposals are
\begin{equation*}
    \begin{aligned}
        & q^*(\mathbf{z}_{1}|\mathbf{x}_{1}) = p(\mathbf{z}_{1}|\mathbf{x}_{1}), \\
        & \text{for all } t=2:T \\
        & q^*(\mathbf{z}_{t}|\mathbf{z}_{1:t-1},\mathbf{x}_{1:t}) = \frac{p(\mathbf{z}_{1:t}|\mathbf{x}_{1:t})}{p(\mathbf{z}_{1:t-1}|\mathbf{x}_{1:t-1})} = p(\mathbf{z}_{t}|\mathbf{z}_{1:t-1}, \mathbf{x}_{1:t}).\\
    \end{aligned}
\end{equation*}
\end{proposition}
\proof See Appendix \ref{sec:optimal_importance_proposal}.

The optimal importance proposals always propagate samples from the previous filtering posterior $p(\mathbf{z}_{1:t-1}|\mathbf{x}_{1:t-1})$ to the new target $p(\mathbf{z}_{1:t}|\mathbf{x}_{1:t})$, thus leads MCFOs to be exact. For SSMs that assume Markovian latent variables %, $p(\mathbf{z}_{t}| \mathbf{z}_{1:t-1})=p(\mathbf{z}_{t}|\mathbf{z}_{t-1}), p(\mathbf{z}_{t-1}|\mathbf{z}_{t:T}, \mathbf{x}_{t:T})=p(\mathbf{z}_{t}|\mathbf{z}_{t-1})$, 
and conditional independent observations, %\citep{sarkka2013bayesian}, % $p(\mathbf{x}_{t}|\mathbf{z}_{1:t},\mathbf{x}_{1:t-1})  = p(\mathbf{x}_{t}|\mathbf{z}_{t})$ , 
the optimal importance proposals are further simplified to $p(\mathbf{z}_{t}|\mathbf{z}_{t-1}, \mathbf{x}_{t})$ %,  \text{for all } t=2:T$ 
\citep[Proposition~2]{doucet2000sequential}. %However, commonly 
For common intractable problems, though the filtering posteriors and optimal proposals are not accessible, %we have no access to, %$p(\mathbf{z}_{1}|\mathbf{x}_{1})$ or $p(\mathbf{z}_{1:t}|\mathbf{x}_{1:t})$ or $p(\mathbf{z}_{1:t-1}|\mathbf{x}_{1:t-1})$. We 
we can learn a parametric adaptive importance proposal model jointly with generative models by optimizing MCFOs.

\subsection{Learning Generative and Proposal Models}
\label{sec:gradient_estimate}

%table 1
\begin{table*}[!htb]
  \centering
  \small
    \begin{tabular}{lcc}
    %\cmidrule(r){1-2}
    Method  & $\triangledown_{\boldphi}$ & $\triangledown_{\boldtheta}$\\
    \hline
    MCFO & See (\ref{eq:gradient_phi}) & See (\ref{eq:gradient_theta_rws})\\
    \hline
    AESMC/FIVO/VSMC &  
    \multicolumn{2}{c}{
    \parbox{0.7\textwidth}{
    \begin{equation}
        \triangledown_{\boldtheta, \boldphi} \log \hat{p}(\mathbf{x}_{1:T}) + \sum_{t=2}^T \log \frac{\hat{p}(\mathbf{x}_{1:t})}{\hat{p}(\mathbf{x}_{1:t-1})} \left( \sum_{i=1}^K \triangledown_{\boldtheta, \boldphi} \log ({w_{t-1}^{A_{t-1}^i}}/{\sum_j w_{t-1}^j}) \right) 
        \label{eq:AESMC_gradient}
    \end{equation}
    }}
    \\
    \hline
    IWAE & \multicolumn{2}{c}{
    $
    \triangledown_{\boldtheta, \boldphi} f_{\boldtheta, \boldphi}( \mathbf{x}_{1:T}, \Tilde{g}_{\boldphi}(\mathbf{x}_{1:T}^{1}, \boldsymbol{\epsilon}_{1:T}^{1}), \hdots, \Tilde{g}_{\boldphi}(\mathbf{x}_{1:T}, \boldsymbol{\epsilon}_{1:T}^{K}))
    $, where 
    $\Tilde{g}$ is reparameterized function of $q_{\boldphi}(\mathbf{z}_{1:T}|\mathbf{x}_{1:T})$ 
    } \\
    \hline
    NASMC & $
    \sum_{t=1}^T \sum_{i=1}^K \Tilde{w}_t^{i} \triangledown_{\boldphi} \log q_{\boldphi}(\hat{\mathbf{z}}_t^{i}|\mathbf{x}_{1:t}, \hat{\mathbf{z}}_{1:t-1}^{A_{t-1}^{i}})$ & 
    $ \sum_{t=1}^T \sum_{i=1}^K \Tilde{w}_t^{i} \triangledown_{\boldtheta} \log p_{\boldtheta}(\hat{\mathbf{z}}_t^{i}, \mathbf{x}_t|\mathbf{x}_{1:t-1}, \hat{\mathbf{z}}_{1:t-1}^{A_{t-1}^{i}})$
    \\
    \hline
    %RWS & \[ \sum_{i=1}^K \frac{\prod_{t=1}^T w_{t}^i}{\sum_{l=1}^K (\prod_{t=1}^T w_{t}^l)} \triangledown_{\boldphi} \log q_{\boldphi}(\hat{\mathbf{z}}_{1:T}^{i}|\mathbf{x}_{1:t})\]
    %&  \[ \sum_{i=1}^K \frac{\prod_{t=1}^T w_{t}^i}{\sum_{l=1}^K (\prod_{t=1}^T w_{t}^l)} \triangledown_{\boldtheta} \log p_{\boldtheta}(\hat{\mathbf{z}}_{1:T}^{i}|\mathbf{x}_{1:t})\]
  \end{tabular}
  \caption{Comparison of gradient estimates by MCFO, AESMC/FIVO/VSMC, IWAE, NASMC.}\vspace{-2mm}
  \label{table:gradient}
\end{table*}

To illustrate the learning of a flexible importance proposal and/or a generative model, %The proposed method could be applied to any parametric distribution, where w
we explicitly parameterize generative model $p_{\boldtheta}$ and proposal model $q_{\boldphi}$ by $\boldtheta$ and $\boldphi$ respectively, and optimize them by gradient-based algorithms.% to optimize differentiable parametric distributions. The gradients of an MCFO with respect to $\boldtheta$ and $\boldphi$, are decomposed as $\triangledown_{\boldtheta, \boldphi} \mathcal{L}_{\text{MCFO}}^K = \sum_{t=1}^T \triangledown_{\boldtheta, \boldphi} \mathcal{L}_t^K$. %, and each $\triangledown_{\boldtheta, \boldphi} \mathcal{L}_t^K $ can be further specified by two terms, $\mathbb{E}_{Q_t^K}\left[\log R_t^K \triangledown_{\boldtheta, \boldphi} \log Q_t^K \right]$ and $\mathbb{E}_{Q_t^K}\left[ \triangledown_{\boldtheta, \boldphi} \log R_t^K  \right]$; see Appendix \ref{sec:appendix_gradient_theta} for detailed derivations.

%Even if they allow tight bounds, 
Earlier methods %AESMC/FIVO/VSMC 
suffer from high variance in gradient estimate due to the second term in (\ref{eq:AESMC_gradient}) in Table~\ref{table:gradient}. This is mainly caused by 1) large magnitudes of $\log \hat{p}(\mathbf{x}_{1:T})$, especially at the beginning of training \citep{mnih2014neural}; and 2) high variance in the gradient for non-smooth categorical distribution of discrete ancestral indices $A_{t-1}^i$ in SMC. %\citep{naesseth2018variational, maddison2017filtering} 
FIVO and VSMC propose to ignore the high variance term to stabilize and accelerate convergence. However, it comes at the cost of an induced bias that cannot be eliminated by increasing the number of samples%$K$
 and deteriorates the convergence to optimum \citep{roeder2017sticking}. 

MCFOs circumvent the issue for $\triangledown_{\boldphi} \mathcal{L}_{\text{MCFO}}^K $ by the \textit{reparameterization trick} \citep{kingma2014auto}
without an extra bias. Assuming the proposal distribution $q_{\boldphi}$ is reparameterizable \citep{mohamed2019monte}, $\triangledown_{\boldphi} \mathcal{L}_t^K$ is estimated less variantly by:
\begin{equation}
    \begin{aligned}
        & \triangledown_{\boldphi} \mathcal{L}_t^K = \mathbb{E}_{p_{\boldtheta}(\mathbf{z}_{1:t-1}^{i}|\mathbf{x}_{1:t-1})} \big[\triangledown_{\boldphi} \mathbb{E}_{q_{\boldphi}(\mathbf{z}_{t}^{i}|\mathbf{z}_{1:t-1}^{i},\mathbf{x}_{1:t})} \\
        & \quad \left(\log \frac{1}{K} \sum_{i=1}^K %\frac{p(\mathbf{x}_{t}|\mathbf{z}_{1:t}^{i},\mathbf{x}_{1:t-1}^{i})p(\mathbf{z}_{t}^{i}|\mathbf{z}_{1:t-1}^{i}, \mathbf{x}_{1:t-1}^{i})}{q_{\boldphi}(\mathbf{z}_{t}^{i}|\mathbf{z}_{1:t-1}^{i},\mathbf{x}_{1:t})}]]\\
        \frac{p_{\boldtheta}(\mathbf{z}_{1:t}^{i},\mathbf{x}_{1:t})}{p_{\boldtheta}(\mathbf{z}_{1:t-1}^{i},\mathbf{x}_{1:t-1})q_{\boldphi}(\mathbf{z}_{t}^{i}|\mathbf{z}_{1:t-1}^{i},\mathbf{x}_{1:t})} \right) \big]\\
        \simeq & \triangledown_{\boldphi}f_{\boldtheta,\boldphi}(\mathbf{x}_{1:t},\hat{\mathbf{z}}_{1:t-1}^{1:K},\\
        & \quad \underbrace{ g_{\boldphi}(\hat{\mathbf{z}}_{1:t-1}^{1}, \boldsymbol{\epsilon}^i, \mathbf{x}_{1:t}), \hdots, g_{\boldphi}(\hat{\mathbf{z}}_{1:t-1}^{K},\boldsymbol{\epsilon}^K, \mathbf{x}_{1:t})}_{\text{reparameterization trick, } \hat{\mathbf{z}}_{t}^{i} = g_{\boldphi}(\hat{\mathbf{z}}_{1:t-1}^{i},\boldsymbol{\epsilon}^i, \mathbf{x}_{1:t})}), 
        % \log \frac{1}{K} \sum_{i=1}^K  \frac{p(g_{\phi}^{i}(\mathbf{z}_{t-1}^{i},\epsilon, \mathbf{x}_t)|\mathbf{x}_{t})p(g_{\phi}^{i}(\mathbf{z}_{t-1}^{i},\epsilon, \mathbf{x}_t)|\mathbf{z}_{t-1}^{i})}{q_{\boldphi}(g_{\phi}^{i}(\mathbf{z}_{t-1}^{i},\epsilon, \mathbf{x}_t)|\mathbf{z}_{t-1}^{i},\mathbf{x}_{t})}]
    \end{aligned}
    \label{eq:gradient_phi}
\end{equation}
where $\{\hat{\mathbf{z}}_{1:t}^{i}\}_{i=1:K}$ are K sample trajectories, e.g. from SMC, and can be specified with ancestral indices $A_{t-1}^i$ when resampling applies, %the function 
$f_{\boldtheta,\boldphi}(\cdot)$ is the logarithm average function $\log \frac{1}{K} \sum_{i=1}^K \left( \frac{\cdot}{\cdot} \right)$, and $\boldsymbol{\epsilon}^i$ is a sample from a base distribution $p(\boldsymbol{\epsilon})$. %To be noted, the only resource of biasness of this gradient is from posterior approximation $\hat{p}_{\boldtheta}(\mathbf{z}_{1:t-1}|\mathbf{x}_{1:t-1})$ to true posterior $p_{\boldtheta}(\mathbf{z}_{1:t-1}|\mathbf{x}_{1:t-1})$. 
% To be noted, the inner expectation 
%if $f_{\boldtheta,\boldphi}(\cdot)$ is assumed differentiable, while the estimates by score function are biased \citep{mohamed2019monte}. Thus when generative parameters also converge and samples from $p_{\boldtheta}(\mathbf{z}_{1:t-1}|\mathbf{x}_{1:t-1})$ are distributed as $p(\mathbf{z}_{1:t-1}|\mathbf{x}_{1:t-1})$, the gradient $\triangledown_{\boldphi} \mathcal{L}_t^K$ would become unbiased and have 0 mean at its optimum.
%Different from the variance explosion in the original AESMC, VSMC and FIVOs, decomposing the global estimator into multiple local estimators is a way of Rao-Blackwellization to reduce variance in the gradient. Although samples $\mathbf{z}_{1:t-1}^{1:K}$ are indirectly dependent on $\boldphi$ through proposals, the posterior $p_{\boldtheta}(\mathbf{z}_{1:t-1}|\mathbf{x}_{1:t-1})$ is only conditioned on generative parameters $\boldtheta$. Therefore, we can decouple the first expectation over $p_{\boldtheta}(\mathbf{z}_{1:t-1}|\mathbf{x}_{1:t-1})$ and the gradient computation, and use the reparameterization trick to further lower the variance for second expectation. Furthermore, the second part of the gradient is unbiased %by reparameterization trick 
%\noindent
The same trick cannot directly apply to $\triangledown_{\boldtheta} \mathcal{L}_t^K$, because of the existence of $\boldtheta$ in the expectation. %as $\triangledown_{\boldphi} \mathcal{L}_t^K$ 
%unless we have the true posterior $p(\mathbf{z}_{1:t-1}|\mathbf{x}_{1:t-1})$ or the ability to draw samples from it. % There are a few ways to mitigate the high variance of gradient estimation in general \citep{mnih2014neural, mohamed2019monte, caterini2018hamiltonian, tucker2018doubly}. 
Instead, we use the score function of $p(\mathbf{x}_{t}|\mathbf{x}_{1:t-1})$ to estimate $\triangledown_{\boldtheta} \mathcal{L}_t^K$:
\begin{equation}
    \begin{aligned}
        \triangledown_{\boldtheta} \mathcal{L}_t^K \simeq \sum_{i=1}^K \Tilde{w}^i_t \triangledown_{\boldtheta} \log p_{\boldtheta}(\mathbf{x}_{t}, \hat{\mathbf{z}}_{t}^{i}|\mathbf{x}_{1:t-1}, \hat{\mathbf{z}}_{1:t-1}^{i}),
    \end{aligned}
    \label{eq:gradient_theta_rws}
\end{equation}
\begin{comment}
\begin{equation}
    \begin{aligned}
        \triangledown_{\boldtheta} \log p_{\boldtheta}(\mathbf{x}_{1:T}) & = \int p_{\mathbf{\theta}}(\mathbf{z}_{1:T}|\mathbf{x}_{1:T}) \triangledown_{\mathbf{\theta}} \log p_{\mathbf{\theta}}(\mathbf{x}_{1:T}, \mathbf{z}_{1:T})  d \mathbf{z}_{1:T}\\
        & \simeq \sum_{t=1}^T \sum_{i=1}^K \Tilde{w}^i_t (\triangledown_{\boldtheta} \log p_{\boldtheta}(\mathbf{x}_{t}| \mathbf{z}_{t}^{i})+\triangledown_{\boldtheta} \log p_{\boldtheta}(\mathbf{z}_{t}^{i}| \mathbf{z}_{1:t-1}^{i})),
    \end{aligned}
    \label{eq:gradient_theta_rws}
\end{equation}
\end{comment}
\noindent
where $\Tilde{w}^i_t$ are the normalized importance weights of $\hat{\mathbf{z}}_{1:t}^{i}$; see Appendix \ref{sec:appendix_gradient_theta} and \ref{sec:appendix_gradient_theta_rws} for detailed derivation of (\ref{eq:gradient_phi}) and (\ref{eq:gradient_theta_rws}). Essentially, the score function estimate is equivalent as dropping the high variance term in (\ref{eq:AESMC_gradient}) for $\triangledown_{\boldtheta}$. 

NASMC \citep{gu2015neural} is closely related to MCFOs with SMC implementation in terms of $\triangledown_{\boldtheta}$, and RWS \citep{bornschein2014reweighted} is a special case of NASMC that replaces SMC by SIS. These two methods, however, construct a different surrogate objectives for optimizing $\boldphi$. While NASMC and RWS minimize the approximated inclusive KL-divergence, $\text{KL}(\hat{p}_{\boldtheta}(\mathbf{z}_{1:t}|\mathbf{x}_{1:t})||q_{\boldphi}(\mathbf{z}_{1:t}|\mathbf{x}_{1:t}))$, MCFOs minimize the 
%MCFOs are to minimize 
disclusive KL-divergence, $\text{KL}(Q_t^K(\mathbf{z}_{1:t}^{1:K}|\mathbf{x}_{1:t})||p_{\boldtheta}(\mathbf{z}_{1:t}^{1:K}|\mathbf{x}_{1:t}))$, on the extended latent space as the dual problem of maximizing surrogate objectives. When the family of proposals is adequately flexible to include simple true posteriors%$p_{\boldtheta}(\mathbf{z}_{1:t}|\mathbf{x}_{1:t})$
, both NASMC and MCFOs converge to the same optimum. To fit a potentially complex multi-modal posterior, the simple proposal learned by NASMC and RWS tends to have undesired low density everywhere in order to cover all modalities, thus impairs 
%of distributions for proposal ends up with undesired low density everywhere to impair
the sample efficiency of estimators and restricts the learning of generative models. For MCFOs, $Q_t^K(\mathbf{z}_{1:t}^{1:K}|\mathbf{x}_{1:t})$ is naturally a mixture of simple proposal distributions $q_{\phi}$ with importance weights $\Tilde{w}_{t-1}^i$, which remains flexible to fit multi-modal posteriors, while sustaining sample efficiency.

%\noindent
Furthermore, when the latent and observation variables of generative models are assumed to be finite-order Markovian, %that they are only dependent on finite components of history, 
both gradients of MCFOs can be updated in an incremental manner which makes them well suited for arbitrarily long sequences and data streams.

\section{Experiments}
We seek to evaluate our method in experiments by answering: 1) what is the side-effect of ignoring the high variance term as proposed in earlier methods; %do the proposed gradient estimators reduce bias and variance; 
2) do the gradient estimates of MCFOs reduce the variance without the cost of additional bias; %bias and variance reductions facilitate generative and proposal model learning; 
3) how does the number of samples affect the learning of generative models and how sample efficient are the learned proposal models? %selecting lower variant estimators ? 
We evaluate two instances of MCFOs, MCFO-SMC and MCFO-PIMH, using SMC and PIMH %(see Appendix A.2) 
respectively %as the sampling methods 
to learn generative and proposal models on LGSSM, non-Gaussian, nonlinear, high dimensional SSMs of video sequences, and non-Markovian polyphonic music sequences\footnote{The implementation of our algorithms and experiments are available at \url{https://github.com/ssajj1212/MCFO}.}. We restrict the form of posteriors to $q_{\boldphi}(\mathbf{z}_{t}|\mathbf{z}_{t-1}, \mathbf{x}_{t})$ for SSM cases to encode history into low dimensional representations, while using VRNN \citep{chung2015recurrent} to accommodate long temporal dependencies for non-Markovian data. To be noted, all models are amortized % by the same MLP 
over all time instances. 

\begin{comment}
\begin{itemize}
    \item empirical evidence on bias and variance reduction in gradient estimate by our method
    \item proposed objectives and gradient estimates to facilitate generative and proposal parameters learning
    \item other sampler than SMC to boost performance?
\end{itemize}
\end{comment}

%figure 1
\begin{figure*}[!htb]
    \centering
    \begin{subfigure}[b]{0.707\textwidth}
        \centering
        \includegraphics[width=\textwidth]{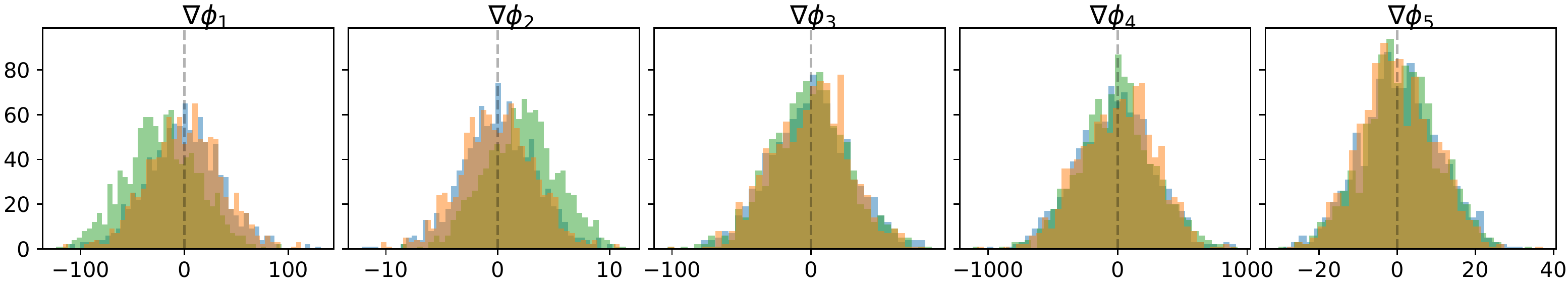}
    \end{subfigure}%
    \hfill
    \begin{subfigure}[b]{0.293\textwidth}
        \centering
        \includegraphics[width=\textwidth]{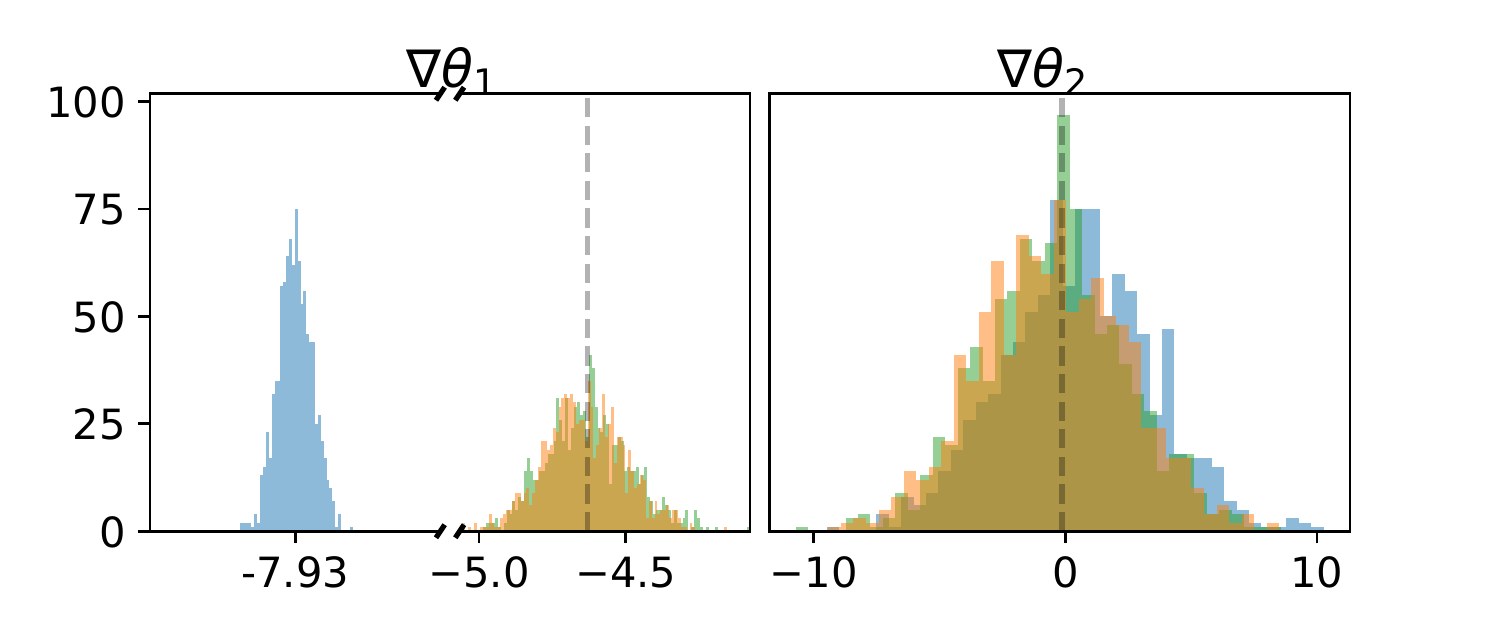}
    \end{subfigure}%
    \vfill
    \begin{subfigure}[b]{0.707\textwidth}
        \centering
        \includegraphics[width=\textwidth]{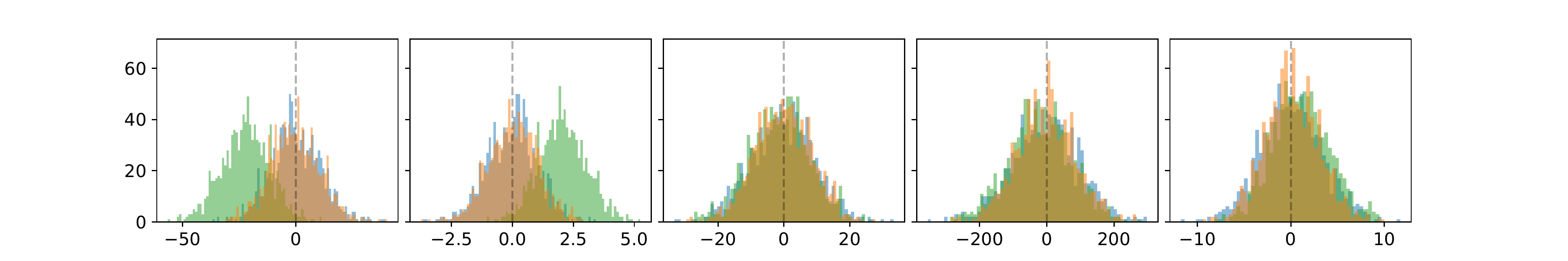}
    \end{subfigure}%
    \hfill
    \begin{subfigure}[b]{0.293\textwidth}
        \centering
        \includegraphics[width=\textwidth]{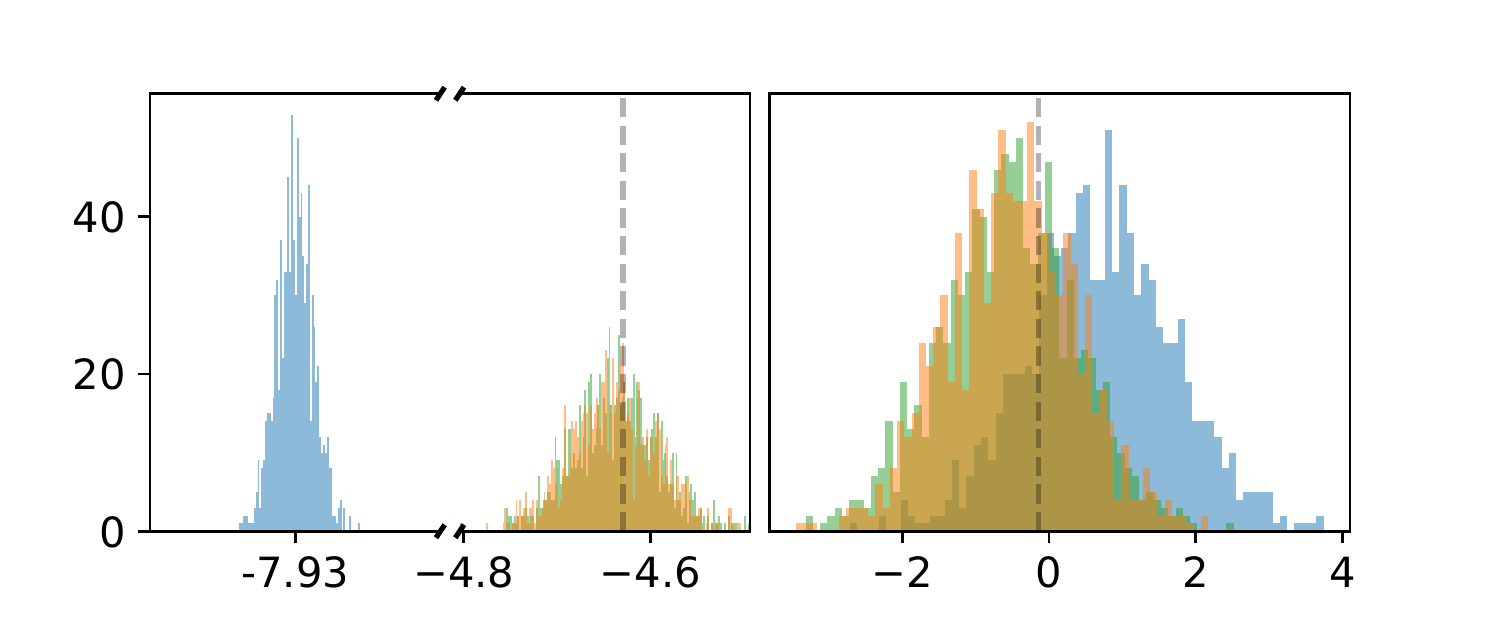}
    \end{subfigure}%
    \vfill
    \begin{subfigure}[b]{0.707\textwidth}
        \centering
        \includegraphics[width=\textwidth]{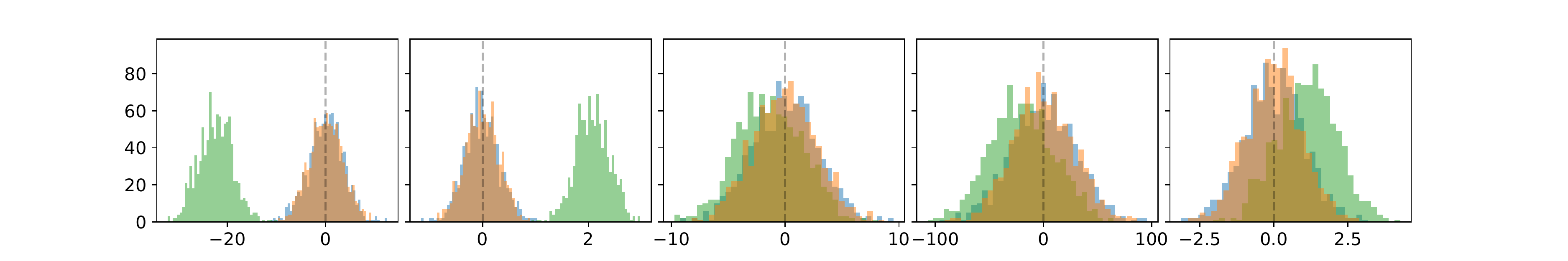}
    \end{subfigure}%
    \hfill
    \begin{subfigure}[b]{0.293\textwidth}
        \centering
        \includegraphics[width=\textwidth]{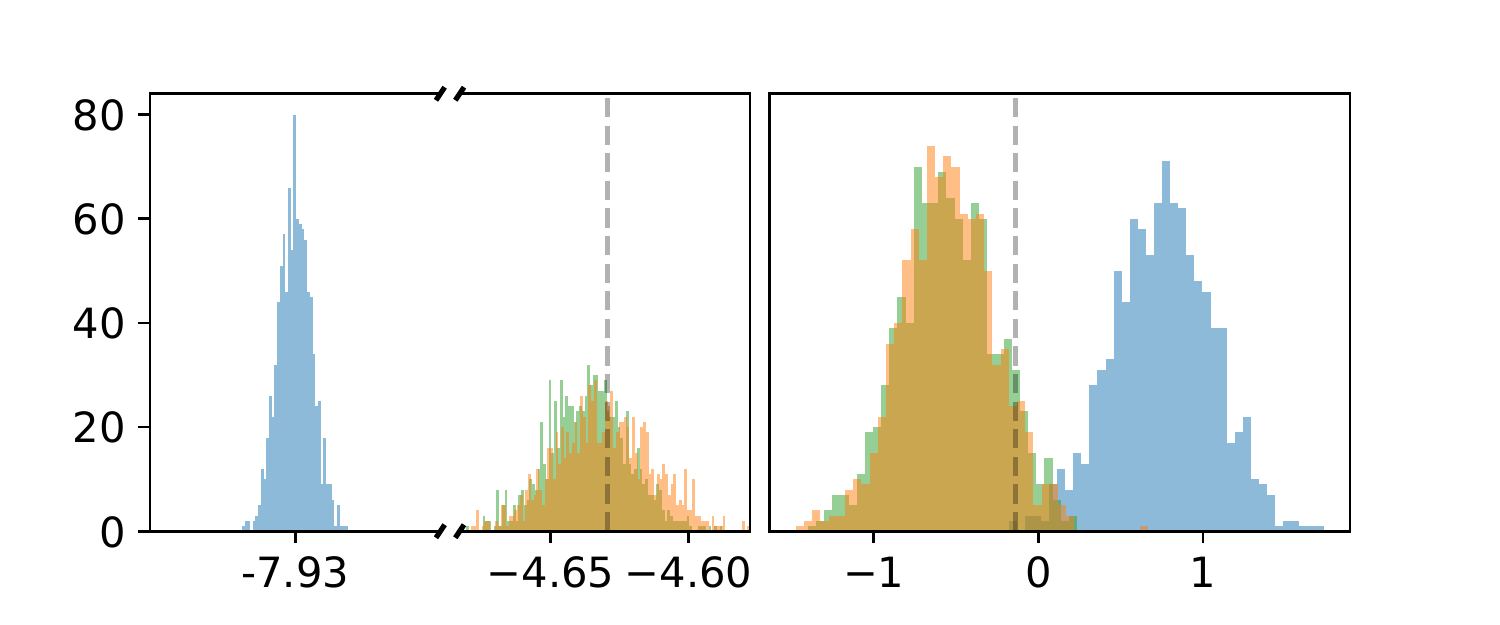}
    \end{subfigure}
    \vfill
    \begin{subfigure}[b]{0.305\textwidth}
        \centering
        \includegraphics[width=\textwidth]{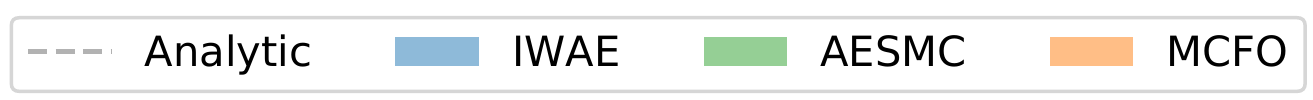}
    \end{subfigure} \vspace{-2mm}
    \caption{Gradient estimates of IWAE, AESMC, MCFO with respect to generative and proposal parameters at their optima with different numbers of samples $K$; \textit{Top:} $K = 10$, \textit{Middle:} $K = 100$, \textit{Bottom:} $K = 1000$.}\vspace{-4mm}
    \label{fig:gradient_phi_optimum}
\end{figure*}

% figure 2
\iftrue
\begin{figure*}[!htb]
    %\vspace{-4mm}
    \centering
    \begin{subfigure}[b]{0.33\textwidth}
        \centering
        \includegraphics[width=\textwidth]{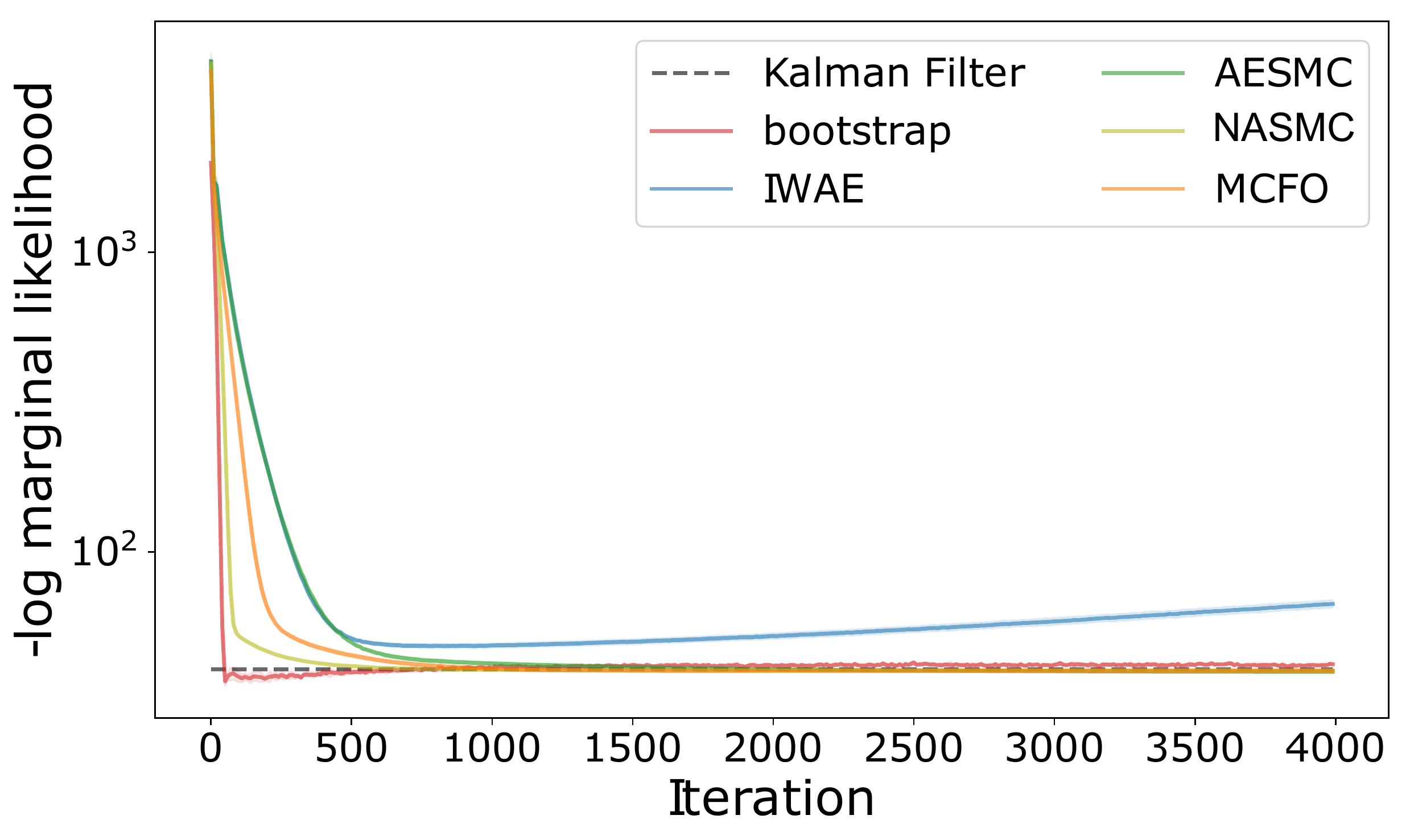}
    \end{subfigure}%
    \hfill
    \begin{subfigure}[b]{0.33\textwidth}
        \centering
        \includegraphics[width=\textwidth]{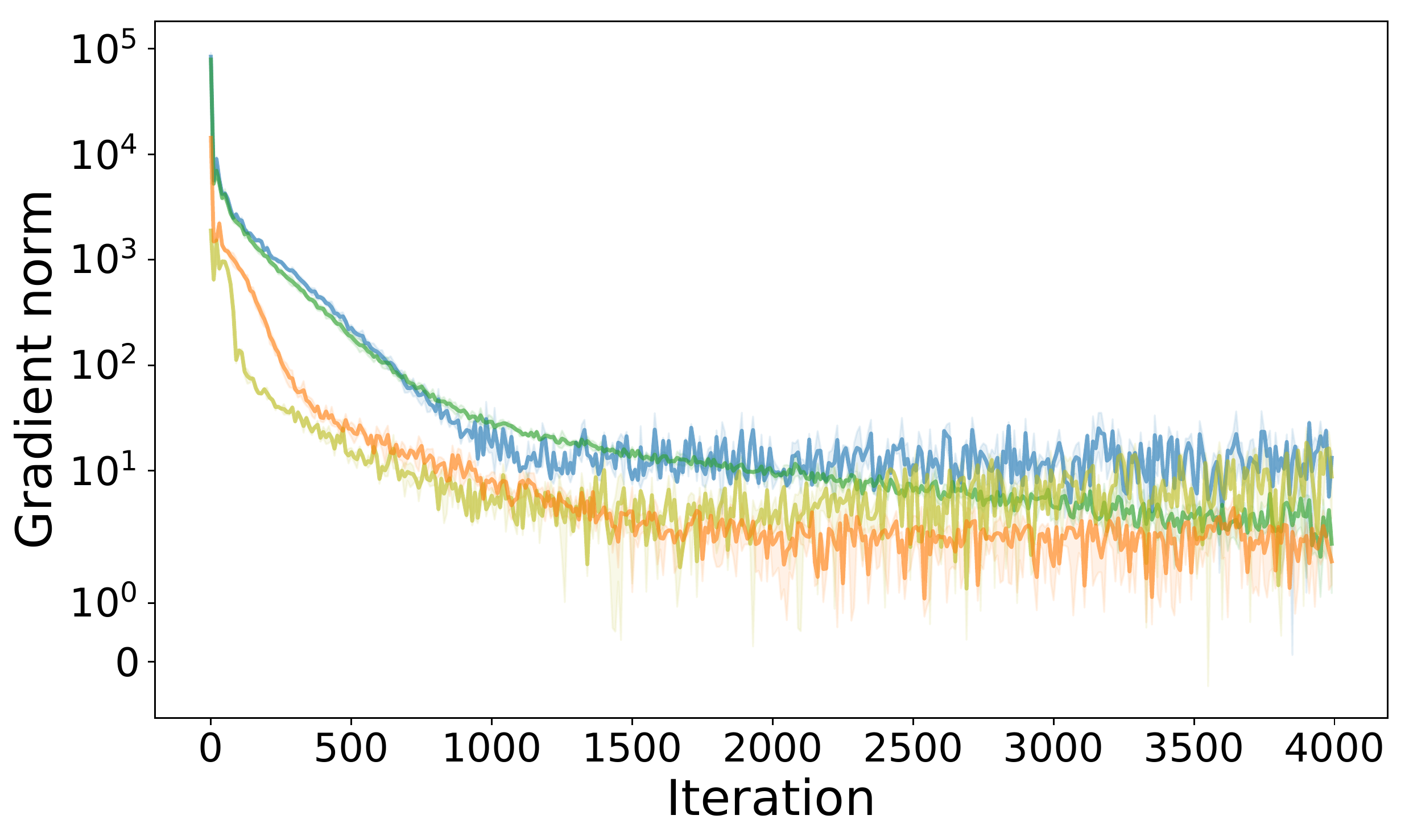}
    \end{subfigure}%
    \hfill
    \begin{subfigure}[b]{0.33\textwidth}
        \centering
        \includegraphics[width=\textwidth]{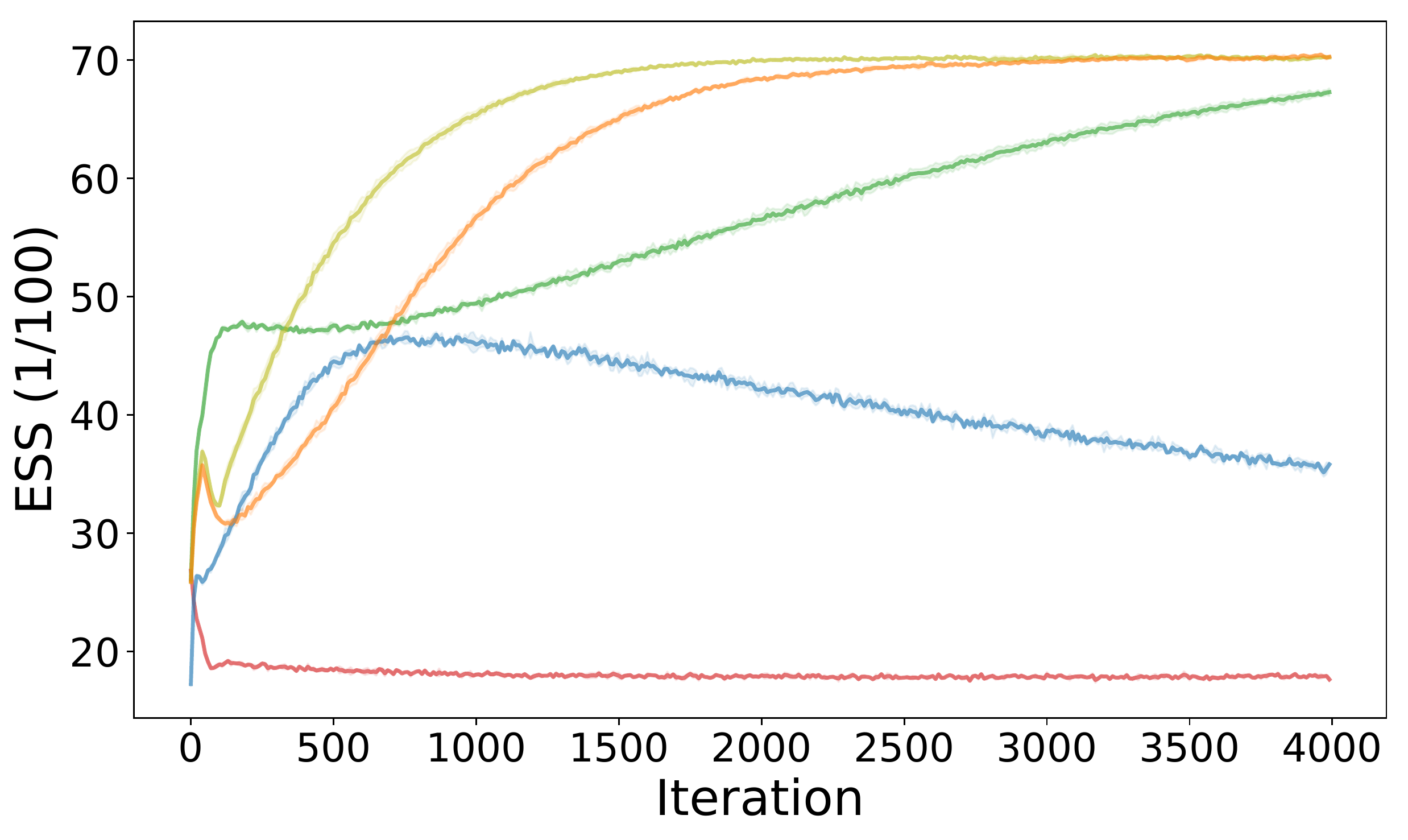}
    \end{subfigure} \vspace{-3mm}
    \caption{\textit{Left}: Negative marginal log-likelihoods (NLLs) on test data. \textit{Middle}: Gradient norms of parameters. \textit{Right}: Effective sample size (ESS). Lines indicate the average of 3 random seed trainings and shaded areas for standard deviation. See Figure \ref{fig:LGSSM_proposal_bias} for the convergence of generative and proposal parameters, $\boldtheta$ and $\boldphi$ in Appendix \ref{sec:appendix_learning_ssm}.}\vspace{-3mm}
    \label{fig:learning_ssm_K100}
\end{figure*}
\fi

\subsection{Gradient Estimation}

Following \citep{rainforth2018tighter, le2017auto}, we carry out experiments to examine gradient estimators on a tractable LGSSM, defined %, for which inference and marginal likelihoods are tractable. We define a LGSSM 
by $\theta_1$ and $\theta_2$, and importance proposal $q_{\boldphi}(\mathbf{z}_{t}|\mathbf{z}_{t-1}, \mathbf{x}_{t})$, parameterized by $\boldphi$: %$\{ \phi_i\}_{i=1:5}$: %three weights and two biases for the mean, and two variances:
\begin{equation}
    \begin{aligned}
        & p(z_1) = \mathcal{N}(z_1; \mu_0, \sigma_0^2), p_{\boldtheta}(z_t|z_{t-1}) = \mathcal{N}(z_t; \theta_1 z_{t-1}, \Sigma_Q),\\
        & p_{\boldtheta}(x_t|z_t) = \mathcal{N}(x_t; \theta_2 z_t, \Sigma_R),\\
        & q_{\boldphi}(z_1|x_{1}) = \mathcal{N}(z_1; \phi_1 x_1 + \phi_2, \Sigma_{q,1}), \\
        & q_{\boldphi}(z_t|z_{t-1},x_{t}) = \mathcal{N}(z_t; \phi_3 z_{t-1} + \phi_4 x_t + \phi_5, \Sigma_{q,t}).
    \end{aligned}
    \label{eq:lgssm}
\end{equation}
\begin{comment}
\begin{equation}
    \begin{aligned}
        & p(z_1) = \mathcal{N}(z_1; \mu_0, \sigma_0^2)\\
        & p(z_t|z_{t-1}) = \mathcal{N}(z_t; \theta_1 z_{t-1}, \sigma_Q^2)\\
        & p(x_t|z_t) = \mathcal{N}(x_t; \theta_2 z_t, \sigma_R^2)
    \end{aligned}
\end{equation}
\end{comment}
See Appendix \ref{sec:appendix_optimal_proposal_lgssm} for detailed derivations. %of optimum $\boldphi^*$ and variances $\Sigma_{q,1}$, $\Sigma_{q,t}$, and gradient $\triangledown_{\boldtheta} \log p(\mathbf{x}_{1:T})$. %with respect to $\boldtheta$ on sequences. %$\{\mathbf{x}^n_{1:T} \}^{n=1:N}$. 

%Although mini-batches are used when training the generative and importance proposal models and gradient variance is reduced by increasing batch size, to better compare the gradient variance, We study gradient on one single sequence in this section. 
%\noindent

The gradient estimates are computed by backwards automatic differentiation \citep{baydin2017automatic} on the objectives defined by IWAE, AESMC and MCFO-SMC w.r.t. $\boldtheta$ and $\boldphi$ using sequences generated by the LGSSM. AESMC is implemented ignoring the high variance term in (\ref{eq:AESMC_gradient}) as suggested. Figure~\ref{fig:gradient_phi_optimum} shows 1000 gradient samples by all three methods under different numbers of samples, $K$, when both $\boldtheta$ and $\boldphi$ are at their optima. %$\boldtheta$ is the same as was used for generation and $\boldphi$ gives the optimal proposals. 

For $\triangledown \boldphi$, the induced bias of AESMC is distinct and does not disappear with increasing $K$, which makes parameters unable to converge to the exact optimum. Although increasing $K$ decreases the variance for all methods, it is specially detrimental to AESMC for which gradient estimates are barely close to true gradients \citep{rainforth2018tighter}, but beneficial to IWAE and MCFO. For $\triangledown \boldtheta$, MCFO and AESMC have similar estimates close to the analytical gradients, while IWAE estimates are substantially deviated due to the high variance of SIS estimators. Given the empirical gain of an alternating strategy with optimizing IWAE objective for $\boldphi$ and AESMC for $\boldtheta$ \citep{le2017auto}, training by MCFOs is expected to have a similar performance but with less computations. 
%Therefore, learning by MCFOs is expected advantageous as the empirical gain by alternating objectives that uses IWAE for $\boldphi$ and AESMC for $\boldtheta$ shows \citep{le2017auto}, but with less computations. 
See Appendix \ref{sec:appendix_gradient_estimate_ex} for more gradient estimates at other locations. 

\subsection{Learning and Inference on LGSSMs}

%With less biased and variant gradient estimators, MCFO is expected to converge faster than AESMC and IWAE. To evaluate their effectiveness %of proposed gradient estimators 

To examine the learning of generative and proposal parameters, we generate 5000 trajectories by LGSSM in (\ref{eq:lgssm}) with $\theta_1 = 0.9$, $\theta_2 = 1.2$, of which 4000 are for training and rest for testing. Figure~\ref{fig:learning_ssm_K100} illustrates 5 benchmarking methods including bootstrap filtering \citep{sarkka2013bayesian}, IWAE, AESMC, NASMC and MCFO-SMC, using the same initialization and optimizer; see Appendix \ref{sec:appendix_learning_ssm} %\citep[Appendix~D.2]{chen2021} %~\ref{sec:appendix_learning_ssm} 
for experiment setups. Note that bootstrap uses prior as proposal, thus no proposal parameter needs to learn. To evaluate the performance of learned proposal models for sample efficiency and tightness of lower bound, we report the variance of estimators by $\text{ESS} = (\sum_{i} (\Tilde{w}_t^{i})^2)^{-1}$, and average over test sequences. 

MCFO-SMC and NASMC learn more sample efficient proposal models than bootstrap, AESMC and IWAE, and converge to the exact analytic optimum. Although AESMC does not differ significantly in terms of NLLs from MCFO and NASMC, the bias in gradient estimates shown previously, causes it slow to converge and cannot converge to the exact optimum. MCFO learns both generative and proposal models faster than AESMC. %, thus NASMC and MCFOs are the better alternatives. 
For this simple case, NASMC converges faster than MCFOs, since fitting a Gaussian proposal with NASMC to the uni-modal Gaussian posterior of LGSSM is easier than fitting a mixture of Gaussian proposals with MCFO. However, NASMC may fail to learn multi-modal posteriors for general intractable problems, as shown in the next section. Furthermore, increasing the number of samples and replacing SMC by PIMH with different number of sweeps only slightly improve the learning by MCFOs, see more results in Appendix \ref{sec:appendix_learning_ssm}. %\ref{sec:appendix_learning_ssm}.

\iftrue
% figure 3
\begin{figure*}[!htb]
    %\vspace{-4mm}
    \centering
    \begin{subfigure}[b]{0.71\textwidth}
        \centering
        \includegraphics[width=0.49\textwidth]{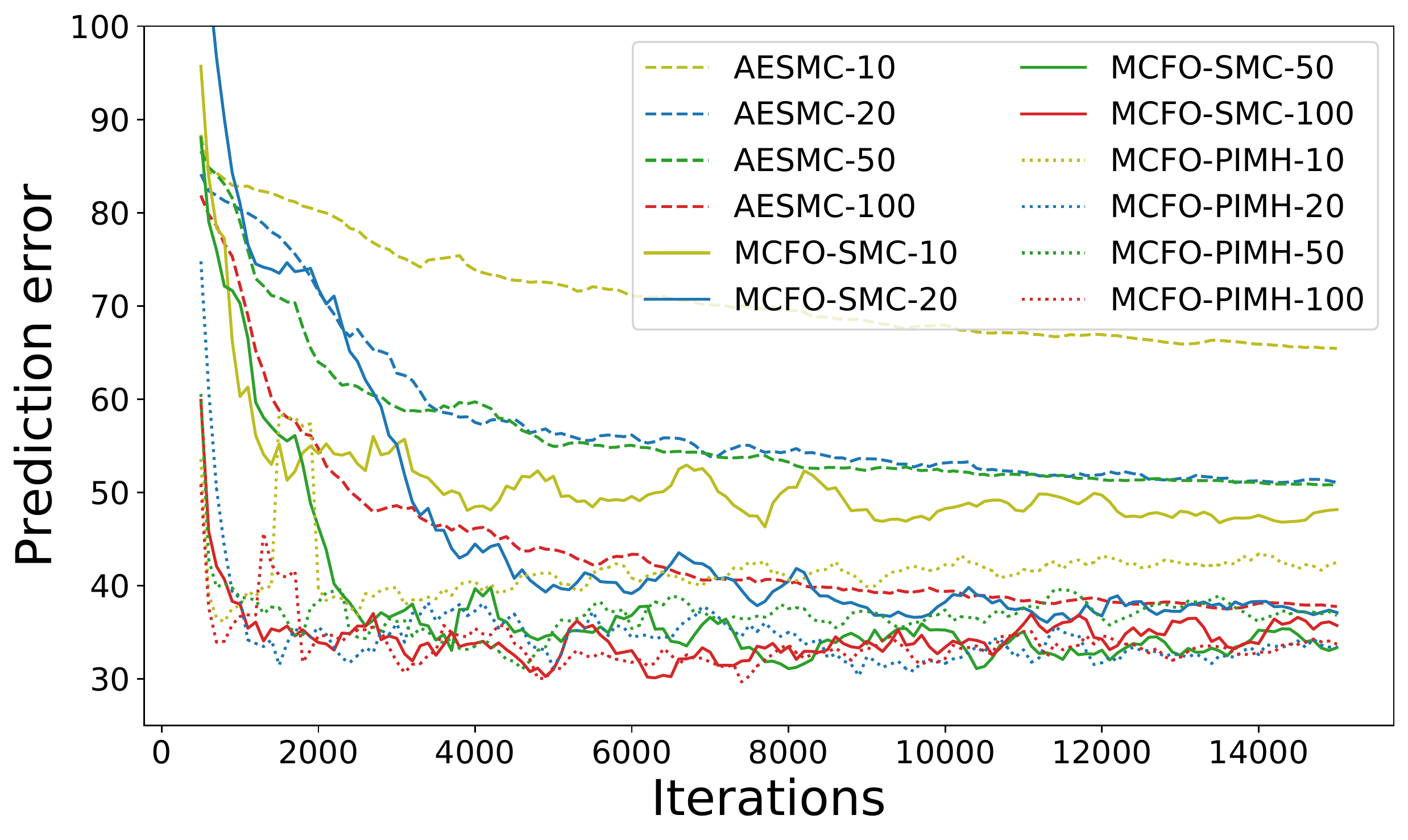}
        \includegraphics[width=0.49\textwidth]{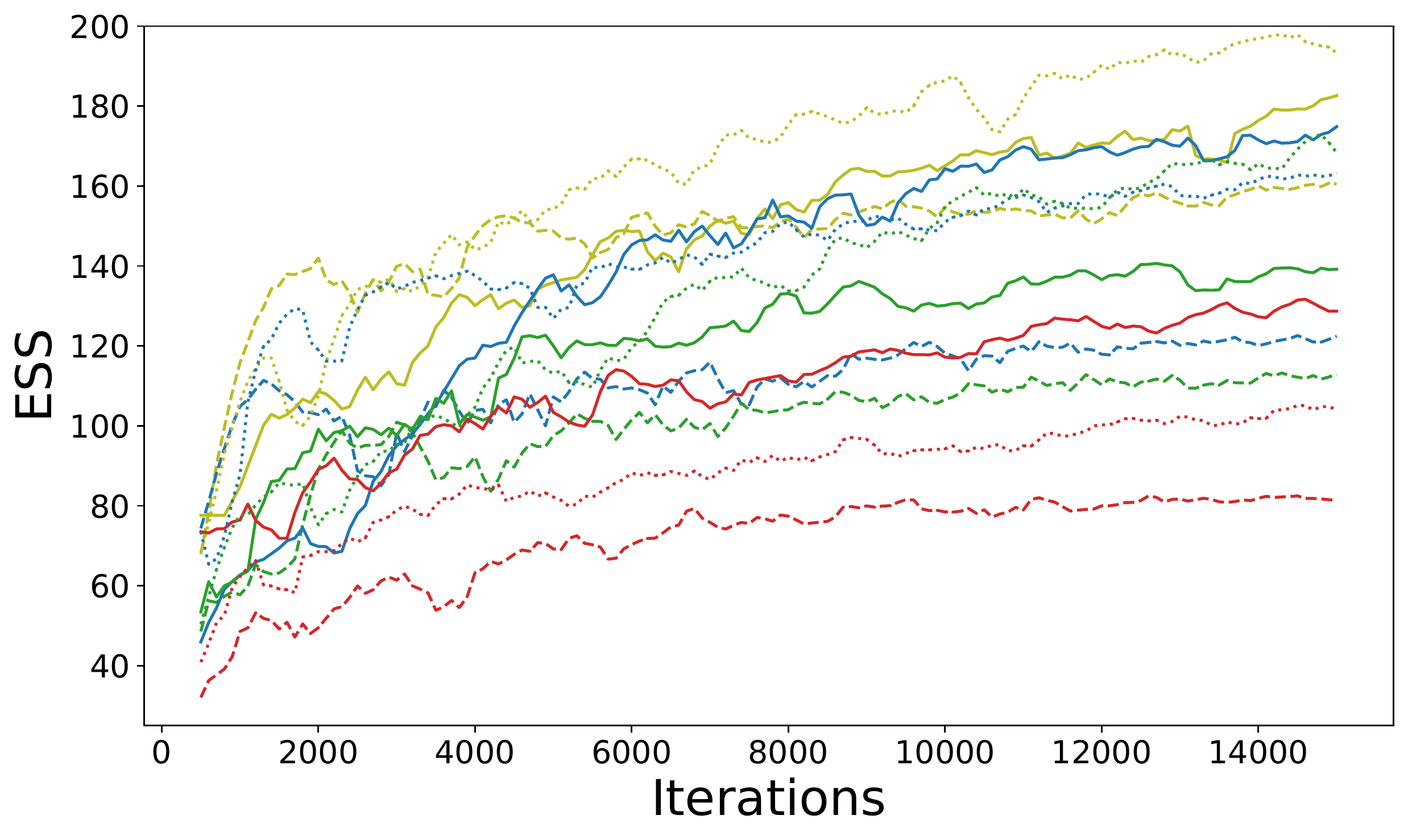}
    \end{subfigure}%
    \begin{subfigure}[b]{0.27\textwidth}
        \includegraphics[width=\textwidth]{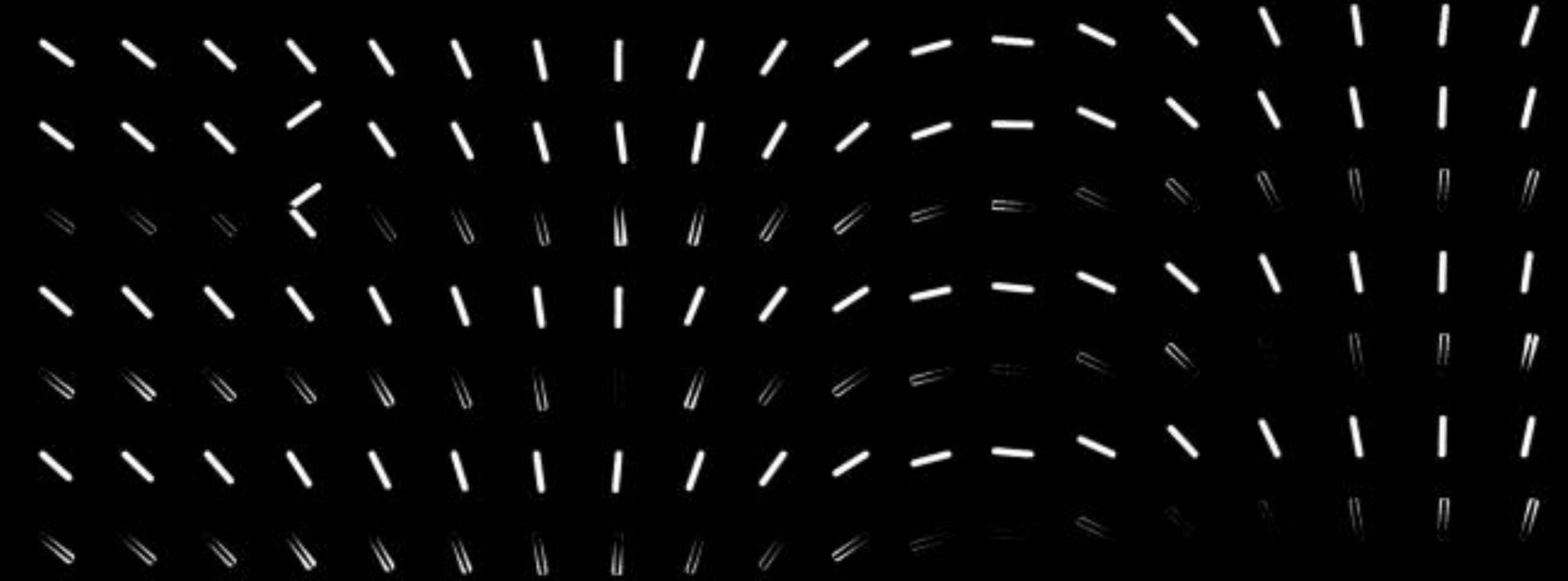}
        \includegraphics[width=\textwidth]{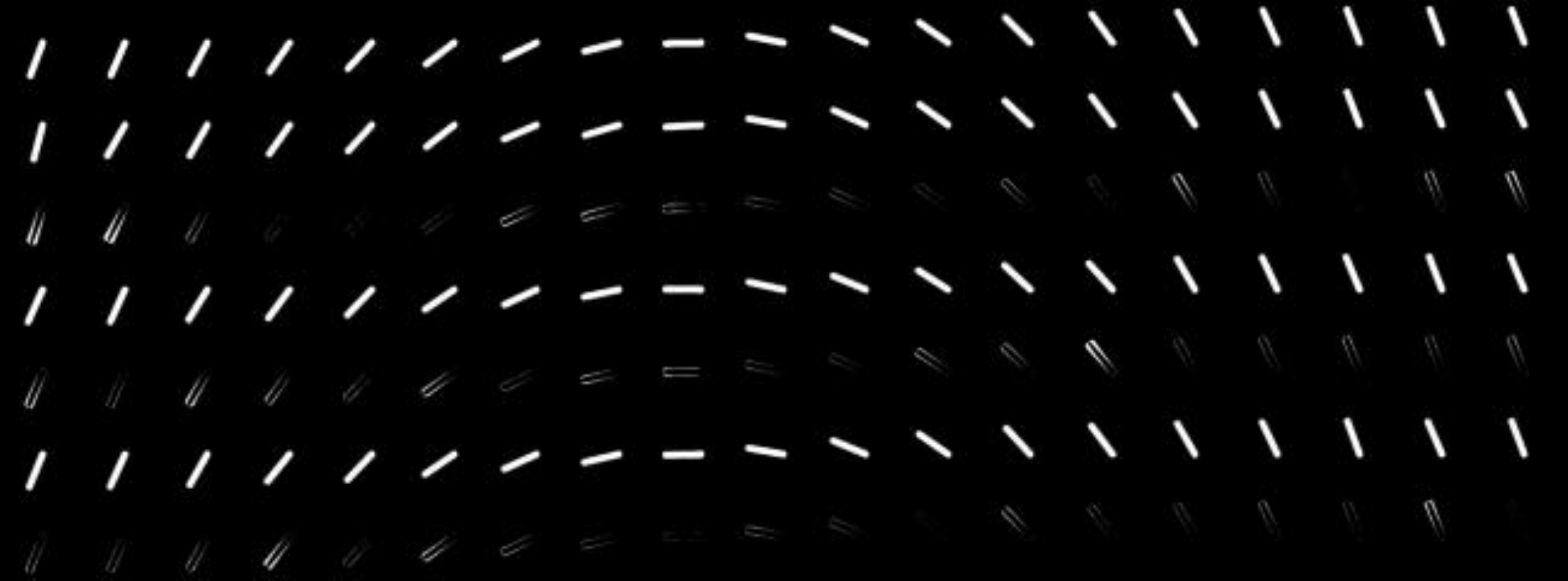}
    \end{subfigure}
    %\vspace{-6mm}
    \vspace{-2mm}
    \caption{\textit{Left, Middle}: One-step prediction errors and ESS on the test sets of generative and proposal models learned by AESMC, MCFO-SMC, MCFO-PIMH with $K=10,20,50,100$, evaluated by SMC with 1000 samples and moving average over 3 evaluation runs. \textit{Right:} Two sequences with one-step predictions by AESMC, MCFO-SMC and MCFO-PIMH with $K=100$. Each row is Bernoulli mean of observations, the one-step predictions and the absolute differences between predictions and observations by AESMC, MCFO-SMC and MCFO-PIMH.}\vspace{-2mm}
    \label{fig:evaluation_Pendulum}
\end{figure*}
% Table 1

\begin{table*}[!htb]
  \centering
  \small
    \begin{tabular}{l|llll||lllll}
    %\cmidrule(r){1-2}
      & $K$    & AESMC  & MCFO-SMC & MCFO-PIMH   & $K$    & AESMC  & MCFO-SMC & MCFO-PIMH  \\
    \hline
    Prediction &\multirow{2}{*}{10}&  65.53 $\pm$ 0.18 & 47.54 $\pm$ 0.94 & \textbf{42.14 $\pm$ 1.49} & \multirow{2}{*}{50} &  50.84 $\pm$ 0.16 & \textbf{34.06 $\pm$ 1.46} & 37.01 $\pm$ 1.47\\
    ESS &  & 160.04 $\pm$ 1.97 &   180.97 $\pm$ 3.16 &   \textbf{195.62 $\pm$ 3.05} &   & 112.53 $\pm$ 2.42 &  139.24 $\pm$ 2.46 & \textbf{168.78$\pm$ 4.51}\\
    \hline
    Prediction & \multirow{2}{*}{20} & 51.13 $\pm$ 0.37  & 37.18 $\pm$ 0.72 & \textbf{33.82 $\pm$ 2.06} & \multirow{2}{*}{100} & 37.87 $\pm$ 0.21 & 36.17 $\pm$ 1.82  & \textbf{33.71 $\pm$ 0.98} \\
    ESS &  & 122.51 $\pm$ 1.89 & \textbf{172.99  $\pm$ 3.46} &  163.01 $\pm$ 2.10 &   &  81.93 $\pm$ 0.82 & \textbf{130.09 $\pm$ 2.12} & 104.53 $\pm$ 1.94
  \end{tabular}
  \vspace{-2mm}
  \caption{One-step prediction errors and ESS on the test set of generative and proposal models learned by AESMC, MCFO-SMC, MCFO-PIMH with $K=10,20,50,100$, evaluated by SMC with 1000 samples averaged over last 1000 iterations.}\vspace{-2mm}
  \label{table:pendulum_evaluation}
\end{table*}
\fi

\begin{table*}[!htb]
  \centering
    \small
    \begin{tabular}{l c c c c c}
    %\cmidrule(r){1-2}
    Methods & Nottingham & JSB chorales & MuseData & Piano-midi.de\\
    \hline
    MCFO-SMC-10 & 2.23 $\pm$0.16 & 3.87 $\pm$0.09 & 3.79 $\pm$0.10 & 6.24$\pm$0.14 \\
    MCFO-SMC-20 & 2.14 $\pm$0.13 & 3.69 $\pm$0.12 & 3.65 $\pm$0.11 & 6.11$\pm$0.15 \\
    MCFO-PIMH-10 & 2.12 $\pm$0.10 & 3.63 $\pm$0.07 & 3.59 $\pm$0.08 & 6.08$\pm$0.09 \\
    MCFO-PIMH-20 & \textbf{2.06 $\pm$0.08} & \textbf{3.54 $\pm$0.08} & \textbf{3.48 $\pm$0.10} & \textbf{6.03 $\pm$0.12} \\
    \hline
    FIVO & 2.58 $\dagger$ (2.60 $\pm$ 0.18)& 4.08 $\dagger$ (3.90 $\pm$0.14) & 5.80 $\dagger$ (5.85$\pm$0.15) & 6.41 $\dagger$ (6.37$\pm$0.19) \\
    IWAE & 2.52 $\dagger$ (2.50 $\pm$ 0.25)& 5.77 $\dagger$ (5.43$\pm$0.20) & 6.54 $\dagger$ (6.28$\pm$0.23) & 6.74 $\dagger$ (6.54$\pm$0.21)\\
    NASMC \citep{gu2015neural} $\dagger$ & 2.72 & 3.99 & 6.89 & 7.61 \\
    SRNN \citep{fraccaro2016sequential} $\dagger$ & 2.94 & 4.74 & 6.28 & 8.20 \\
    STONE \citep{bayer2014learning} $\dagger$ & 2.85 & 6.91 & 6.16 & 7.13 \\
    %RNN-NADE \citep{boulanger2012modeling} $\dagger$ & 2.91 & 5.83 & 6.74 & 7.48 \\
    %RNN $\dagger$ & 4.46 & 8.71 & 8.31 & 8.37
  \end{tabular}
  \vspace{-1mm}
  \caption{Estimated NLL per time on polyphonic test sets by SMC with 500 particles. MCFOs, FIVO and IWAE are evaluated by 10 runs, and both FIVO and IWAE, trained the same as MCFO-SMC-10, are reported in parenthesis. $\dagger$ is originally reported in \citep{maddison2017filtering}. % RNN-NADE and RNN from \citep{boulanger2012modeling}. 
  }\vspace{-3mm}
  \label{table:music_evaluation}
\end{table*}

\subsection{Video Sequences}
\label{sec:ex_pendulum}

To assess MCFOs in more general cases%like nonlinear, non-Gaussian and higher dimensional SSMs
, we simulate 1000 video sequences of a single pendulum system in \texttt{gym} \citep{1606.01540}, out of which 500 are used for testing. Each sequence contains 20 $32\times32$ pixel grayscale images representing factorized Bernoulli distributions of high-dimensional observations; see examples in Figure \ref{fig:evaluation_Pendulum}. The transition and proposal distributions, $p_{\boldtheta}(\mathbf{z}_t|\mathbf{z}_{t-1})$ and $q_{\boldphi}(\mathbf{z}_t|\mathbf{x}_t, \mathbf{z}_{t-1})$, are parameteric Gaussian MLPs, while observation models, $p_{\boldtheta}(\mathbf{x}_t|\mathbf{z}_{t})$, are parameteric Bernoulli MLPs. The latent dimension is set to 3, and optimizers and model definitions are the same for all methods; see Appendix \ref{sec:appendix_pendulum} for a detailed description of experiment setups.

%(move to Appendix)Although log-likelihood is used to maximize over the observation samples from Bernoulli distributions during training and is a common performance metric, evaluating it on the test set does not reflect how well the uncertainty is predicted on the observations. % (considering a pixel sample is drawn as 1 from a Bernoulli with mean as 0.6 and two models that outputs 0.7 and 0.9, cross-entropy in the log-likelihood favors the later model, however, first model predicts more correctly). 
%To address it%absence in evaluation
%, we also implement the widely used one-step-ahead prediction error, defined as the sum of L2 norms between the ground truth and one-step ahead predictions, $\sum_{t=2}^T \Vert\tilde{\mathbf{x}}_t^{pred}-\mathbf{x}_t^{GT}\Vert_2$, and the prediction in observation space, $\tilde{\mathbf{x}}_t^{pred}$, defined by the Bernoulli mean of the predicted mean in latent space, $\Tilde{\mathbf{z}}^{pred}_t \approx \mathbb{E}_{p(\mathbf{z}_t|\mathbf{z}_{t-1})p(\mathbf{z}_{t-1}|\mathbf{x}_{1:t-1})}[\mathbf{z}_t]$, estimated by the transition model and the inference by SMC.

Figure \ref{fig:evaluation_Pendulum} shows the commonly used one-step prediction errors in observations and ESSs on the test set, evaluated by SMC with 1000 particles on the models trained by AESMC, MCFO-SMC and MCFO-PIMH with $K=10, 20, 50, 100$. Additionally, Table \ref{table:pendulum_evaluation} reports both metrics averaged over the last 1000 iterations of trainings. Note that \textit{NASMC fails to converge in this task regardless of $K$}. Compared to AESMC, both MCFO-SMC and MCFO-PIMH, show quicker convergences, lower prediction errors and higher ESSs that indicate more sample efficient proposal models% and tighter bounds
, especially at smaller $K$. Furthermore, Figure \ref{fig:latent_filtering_prediction} in Appendix \ref{sec:appendix_pendulum} shows that MCFOs implicitly regularize for simpler generative models. Although MCFO-PIMH converges faster than MCFO-SMC and AESMC because of better Monte Carlo approximations, the improvement at convergence is marginally small considering that it requires more computations for each sweep in PIMH. Increasing $K$ does improve generative model learning, but slightly impairs the sample efficiency of proposal models. No statistically significant gain is observed to increase $K$ over 200. %; see Appendix \ref{sec:appendix_pendulum}. 
Therefore, the sweet spot of $K$ needs to balance the performance of generation and inference. %according to use cases \citep{rainforth2018tighter}. 

%In terms of prediction errors, they converge faster and outperform AESMC with significant margins at the end of training, 

%As demonstrated by the predicted examples, MCFO-PIMH has the least deviation between predictions to true Bernoulli observations, while AESMC performs the worst. % and especially fails on one instance of the first observation sequence.  %Prediction errors decreases as increasing $K$ in general, but with two exceptions on MCFO-SMC-100 and MCFO-PIHM-50. 

%In terms of ESS, MCFO-SMC and MCFO-PIMH learn the proposal models with better sample efficiency during training and at the end, compared to AESMC. 

%Compared to MCFO-SMC, MCFO-PIMH has higher ESS for $K=10, 50$, and is comparable for $K=20$, which shows that more robust sampling methods are beneficial in terms of sample efficiency. Furthermore, when .
%\noindent
%Larger $K$ facilitate lower estimated NLL and prediction errors as expected. It might be somehow counter-intuitive that increasing $K$ worsens ESS. It is explained by that larger $K$ generates more samples to be fitted by $Q_t^K$ which tends to distribute density everywhere.  

\subsection{Polyphonic Music}
To demonstrate the performance of MCFOs for non-Markovian high dimensional data with complex temporal dependencies, we train VRNN models with MCFO-SMC and MCFO-PIMH on four polyphonic music datasets \citep{boulanger2012modeling}. We preprocess all musical notes to 88-dimensional binary sequences and configure generative and proposal models as in \citep{maddison2017filtering}; see Appendix \ref{sec:appendix_polyphonic_music} for experiment details. Table \ref{table:music_evaluation} reports the estimated NLLs using 500 samples, as with the other benchmarked methods, %using SMC 
on the models trained by MCFO-SMC and MCFO-PIMH with 10, 20 samples. As can be seen for all four datasets, MCFO-SMC and MCFO-PIMH are either superior or comparable to the other state-of-the-art algorithms. % that include long temporal dependences. %Therefore, MCFOs improve learning in modeling long temporal dependencies.

\section{Conclusion}
We introduce Monte Carlo filtering objectives (MCFOs), a new family of variational filtering objectives for learning generative and amortized importance proposal models of time series. MCFOs extend to a wider choice of estimators and accommodate important theoretical properties to achieve tighter objectives. We show empirically that MCFOs and the proposed gradient estimators facilitate more stable and efficient learning on parametric generative and importance proposal models, compared to a number of state-of-the-art methods in various tasks. In future works, we would like to extend MCFOs to smoothing problems and explore tractable MCFOs by flow-based methods.

\section*{Acknowledgements}
This work is supported by the Wallenberg AI, Autonomous Systems and Software Program.

\bibliographystyle{named}
\bibliography{references}

\onecolumn

\newpage

\clearpage
%\section*{Appendix}

\setcounter{secnumdepth}{3}

\newpage
\newpage

\pagestyle{empty}

\newpage

\newpage

\setcounter{subsection}{0}
\renewcommand{\thesubsection}{\Alph{subsection}}
%\section*{Supplementary Materials for Monte Carlo Filtering Objectives: A New Family of Variational Objectives to Learn Generative Model and Neural Adaptive Proposal for Time Series}
\section*{Appendix}
\subsection{Sampling Methods of Sequences}
\subsubsection{Sequential Monte Carlo}
\label{sec:appendix_smc}
\begin{algorithm}[!h]
\caption{Sequential Monte Carlo}
\label{algo:SMC}
\begin{algorithmic}[1]
 \REQUIRE observation sequence $\mathbf{x}_{1:T}$, the number of particles $K$, proposal distributions $\{ q(\mathbf{z}_t|\mathbf{z}_{1:t-1}, \mathbf{x}_{1:t}) \}_{t=1:T}$, joint likelihoods $\{ p(\mathbf{x}_{1:t}, \mathbf{z}_{1:t}) \}_{t=1:T}$ 

  %\ENSURE $y = x^n$
  \STATE Sample $\hat{\mathbf{z}}_{1}^{i}$ from $q_1(\mathbf{z}_1)$, for $i=1:K$
  \STATE Compute weights $w_1^{i}=\frac{p(\mathbf{x}_1, \hat{\mathbf{z}}_1^i)}{q(\hat{\mathbf{z}}_1^{i}|\mathbf{x}_1)}$ \\and normalize weights $\Tilde{w}_1^{i}=\frac{w_1^{i}}{\sum_{j=1}^K w_1^{j}}$
  \STATE Resample $\bar{\mathbf{z}}_{1}^{i}$ from $\{ \Tilde{w}_1^{i}, \hat{\mathbf{z}}_{1}^{i} \}$
  \FOR{$t=2:T$}
  \STATE Sample $\hat{\mathbf{z}}_{t}^{i}$ from $q(\mathbf{z}_t|\bar{\mathbf{z}}_{1:t-1}^{i}, \mathbf{x}_{1:t})$
  \STATE Set $\hat{\mathbf{z}}^{i}_{1:t} = (\bar{\mathbf{z}}_{1:t-1}^{i}, \hat{\mathbf{z}}_t^{i})$
  \STATE Compute weights\\
  \hspace{0.3cm} $w_t^{i} = 
  \frac{p(\mathbf{x}_{1:t}, \hat{\mathbf{z}}_{1:t}^{i})}{p(\bar{\mathbf{z}}^{i}_{1:t-1},\mathbf{x}_{1:t-1})q(\hat{\mathbf{z}}^{i}_t|\bar{\mathbf{z}}^{i}_{1:t-1},\mathbf{x}_{1:t})}$, \\and normalize weights $\Tilde{w}_t^{i} = \frac{w_t^{i}}{\sum_{j=1}^Kw_t^{j}}$
  \STATE Resample $\bar{\mathbf{z}}_{1:t}^{i} = \hat{\mathbf{z}}_{1:t}^{A_{t-1}^i}$ from $\{ \Tilde{w}_t^{i}, \hat{\mathbf{z}}_{1:t}^{i} \}$, where $A_{t-1}^i$ is the anstral index
  \ENDFOR
  
  \RETURN weighted trajectories $\{w_{1:T}^{i}, \hat{\mathbf{z}}^{i}_{1:T}\}_{i=1:K}$, marginal likelihood estimates $\hat{p}(\mathbf{x}_{1:T}) = \prod_{t=1}^T \frac{1}{K} \sum_{i=1}^K w_t^{i}$
\end{algorithmic}
\end{algorithm}

AESMC, VSMC and FIVO, define three closely related MCOs, exploit the estimators by SMC as in (\ref{eq:smc_estimator}). The main differences between these objectives lie in their implementations of SMCs: AESMC implements the classic SMC as \citep[Section~3.5]{doucet2009tutorial}; VSMC assumes Markovian latent variables and conditional independence of observation so that proposal models are simplified to the form of $q(\mathbf{z}_t|\mathbf{z}_{t-1}, \mathbf{x}_t)$ and observation conditionals to $p(\mathbf{x}_t|\mathbf{z}_t)$ in the generative models; FIVO defines objectives using the adaptive resampling SMC \citep[Section~3.5]{doucet2009tutorial}, but it uses a fixed resampling schedule as AESMC during training.

\subsubsection{Particle Independence Metropolis-Hastings}
\label{sec:pimh}

Although resampling in SMC substantially decreases the variance in estimates of marginal likelihoods compared to Sequential Importance Sampling (SIS), the approximations of SMC deteriorate when the sample components are not rejuvenated at subsequent time steps \citep{andrieu2010particle}. To give a more reliable approximation of the posterior, particle independence Metropolis-Hastings (PIMH) \citep{andrieu2010particle} uses SMC approximations as proposal distributions in Metropolis-Hastings (MH) iterations, and compute the acceptance ratio using the estimates of marginal distributions by SMC. Algorithm \ref{algo:PIMH} provides pseudo-code for sampling with PIMH. 

To be noted, the PIMH sampler itself might not be a serious competitor to standard SMC and the potential improvement comes along with increased computational cost. Combining with other MCMC transitions e.g. PIMH implemented in particle marginal Metropolis-Hastings is found to better target to the true posterior of latent variables and model parameters. We use PIMH in our experiments as a complement to standard SMC to investigate whether better approximations of posteriors lead to better learning in generative and proposal models. To be noted, the number of SMC sweeps per iteration, $M$, needs to be specified by the user, thus PIMH requires more computations compared to standard SMC. 

%does not require SMC to provide an exact approximation of $p(\mathbf{z}_{1:T}|\mathbf{x}_{1:T})$ but samples distributed according to the target. Therefore, it is more robust. 

%However, 

%See more detailed description of PIMH algorithm in \ref{sec:pimh}.

%\textit{degeneracy} problems typically occur for long sequences almost immediately with a finite number of samples \citep{doucet2009tutorial}. %, if only finite samples are used. %so that joint posterior $p(\mathbf{z}_{1:T}|\mathbf{x}_{1:T})$ deteriorates as samples drawn at any time $t<T$ are rejuvenated. Shortening dependencies in proposals to only conditioning on the most recent latent variables or regenerating previous paths $\mathbf{z}_{t-L+1:t}^{i}$ with a fixed lag $L$ could mitigate the problem. 

%Particle Independence Metropolis-Hastings (PIMH) \citep{andrieu2010particle} aims to mitigate the degeneracy problem by 

\begin{algorithm}
\caption{Particle independent Metropolis-Hastings}
\label{algo:PIMH}
\begin{algorithmic}[1]
\STATE Run SMC to obtain approximate $\hat{p}(\mathbf{z}_{1:T}|\mathbf{x}_{1:T})$, and sample $\mathbf{z}_{1:T}(0)$ from $\hat{p}(\mathbf{z}_{1:T}|\mathbf{x}_{1:T})$ and compute marginal likelihood estimates $\hat{p}(\mathbf{x}_{1:T})(0)$.
\FOR{$m=1:M$}
\STATE run SMC to sample $\mathbf{z}^*_{1:T}$ from $\hat{p}(\mathbf{z}_{1:T}|\mathbf{x}_{1:T})$ and compute $\hat{p}(\mathbf{x}_{1:T})^*$ 
\STATE with probability $\min (1, \frac{\hat{p}(\mathbf{x}_{1:T})^*}{\hat{p}(\mathbf{x}_{1:T})(m-1)})$,
%\STATE 
update $\mathbf{z}_{1:T}(m) = \mathbf{z}^*_{1:T}$ and $\hat{p}(\mathbf{x}_{1:T})(m)=\hat{p}(\mathbf{x}_{1:T})^*$
\STATE otherwise $\mathbf{z}_{1:T}(m) = \mathbf{z}_{1:T}(m-1)$ and $\hat{p}(\mathbf{x}_{1:T})(m)=\hat{p}(\mathbf{x}_{1:T})(m-1)$
\ENDFOR 
\end{algorithmic}
\end{algorithm}

\subsection{Monte Carlo Filtering Objectives (MCFOs)} 

\subsubsection{Derivation of MCFOs}
\label{sec:appendix_lower_bound}
Using Jensen's inequality, we could derive an MCO, $\mathcal{L}_t^K$, for each conditional marginal log-likelihood $\log p(\mathbf{x}_t|\mathbf{x}_{1:t-1})$:
\begin{equation}
\begin{aligned}
    \log p(\mathbf{x}_t|\mathbf{x}_{1:t-1}) & = \log \int p(\mathbf{z}_{1:t-1}|\mathbf{x}_{1:t-1}) q(\mathbf{z}_{t}|\mathbf{z}_{1:t-1},\mathbf{x}_{1:t})  \frac{p(\mathbf{z}_{1:t},\mathbf{x}_{1:t})}{p(\mathbf{z}_{1:t-1},\mathbf{x}_{1:t-1})q(\mathbf{z}_{t}|\mathbf{z}_{1:t-1},\mathbf{x}_{1:t})}   \mathrm{d} \mathbf{z}_{1:t}\\
    & \geq \int p(\mathbf{z}_{1:t-1}^{1:K}|\mathbf{x}_{1:t-1}) \prod_{i=1}^K q(\mathbf{z}_{t}^i|\mathbf{z}_{1:t-1}^i,\mathbf{x}_{1:t}) \\
    & \quad \underbrace{\quad \quad \quad \quad \quad \log \left( \frac{1}{K} \sum_{i=1}^K     
    \frac{p(\mathbf{z}_{1:t}^i,\mathbf{x}_{1:t})}{p(\mathbf{z}_{1:t-1}^i,\mathbf{x}_{1:t-1})q(\mathbf{z}_{t}^i|\mathbf{z}_{1:t-1}^i,\mathbf{x}_{1:t})} \right)  \mathrm{d} \mathbf{z}_{1:t}^{1:K} }_{\mathcal{L}_t^K}
    \\
    % & = \mathbb{E}_{p(\mathbf{z}_{1:t-1}|\mathbf{x}_{1:t-1})}[\mathbb{E}_{q(\mathbf{z}_{t}|\mathbf{z}_{1:t-1},\mathbf{x}_{1:t})}[\log \frac{p(\mathbf{z}_{1:t},\mathbf{x}_{1:t})}{p(\mathbf{z}_{1:t-1},\mathbf{x}_{1:t-1})q(\mathbf{z}_{t}|\mathbf{z}_{1:t-1},\mathbf{x}_{1:t})} ]]\\
    & = \mathbb{E}_{Q_t^K(\mathbf{z}_{1:t}^{1:K} \mid \mathbf{x}_{1:t})} \left[ \log  R_t^K(\mathbf{z}_{1:t}^{1:K}, \mathbf{x}_{1:t}) \right],\\
    &  \quad \quad \quad \quad Q_t^K(\mathbf{z}_{1:t}^{1:K} \mid \mathbf{x}_{1:t}) = p(\mathbf{z}_{1:t-1}^{1:K}|\mathbf{x}_{1:t-1})\prod_{i=1}^K q(\mathbf{z}_{t}^i|\mathbf{z}_{1:t-1}^i,\mathbf{x}_{1:t}),\\
    &  \quad \quad \quad \quad R_t^K(\mathbf{z}_{1:t}^{1:K}, \mathbf{x}_{1:t}) = \frac{1}{K} \sum_{i=1}^K \frac{p(\mathbf{z}_{1:t}^i,\mathbf{x}_{1:t})}{p(\mathbf{z}_{1:t-1}^i,\mathbf{x}_{1:t-1})q(\mathbf{z}_{t}^i|\mathbf{z}^i_{1:t-1},\mathbf{x}_{1:t})}.
\end{aligned}
\label{eq:mco_t}
\end{equation}
Specifically for $\log p(\mathbf{x}_1)$:
\begin{equation}
\begin{aligned}
    \log p(\mathbf{x}_1) & = \log \int \frac{p(\mathbf{z}_{1},\mathbf{x}_{1})}{q(\mathbf{z}_{1}|\mathbf{x}_{1})} q(\mathbf{z}_{1}|\mathbf{x}_{1}) \mathrm{d} \mathbf{z}_{1} \\
    & \geq \underbrace{\int \prod_{i=1}^K q(\mathbf{z}_{1}^i|\mathbf{x}_{1}) \log \left( \frac{1}{K} \sum_{i=1}^K \frac{p(\mathbf{z}_{1}^i,\mathbf{x}_{1})}{q(\mathbf{z}_{1}^i|\mathbf{x}_{1})} \right) \mathrm{d} \mathbf{z}_{1}^{1:K}}_{\mathcal{L}_1^K}\\
    & = \mathbb{E}_{Q_1^K(\mathbf{z}_{1}^{1:K} \mid \mathbf{x}_{1})} \left[ \log R_1^K(\mathbf{z}_{1}^{1:K}, \mathbf{x}_{1}) \right],\\
    & \quad \quad \quad \quad Q_1^K(\mathbf{z}_{1}^{1:K} \mid \mathbf{x}_{1}) = \prod_{i=1}^K q(\mathbf{z}_{1}^{i}|\mathbf{x}_{1}),\\
    & \quad \quad \quad \quad  R_1^K(\mathbf{z}_{1}^{1:K}, \mathbf{x}_{1}) = \frac{1}{K} \sum_{i=1}^K \frac{p(\mathbf{z}_{1}^i,\mathbf{x}_{1})}{q(\mathbf{z}_{1}^i|\mathbf{x}_{1})}.
\end{aligned}
\label{eq:mco_1}
\end{equation}
We define $\mathcal{L}_{\text{MCFO}}^K$ as the summation of $\mathcal{L}_t^K$:
\begin{equation*}
    \begin{aligned}
        \log p(\mathbf{x}_{1:T}) & = \log p(\mathbf{x}_1) + \sum_{t=2}^T \log p(\mathbf{x}_t|\mathbf{x}_{1:t-1})\\
        & \geq \sum_{t=1}^T \mathcal{L}_t^K = \mathcal{L}_{\text{MCFO}}^K.
    \end{aligned}
\label{eq:mco_bound}
\end{equation*}

\subsubsection{Relation between MCFOs and FIVOs}
\label{sec:fivo_and_mcfo}

Considering the filtering problem that assumes that future observations have no impact on the posterior, FIVO defined by the SMC estimator in (\ref{eq:smc_estimator}) can be decomposed as:
\begin{equation*}
    \begin{aligned}
        & \mathbb{E}_{\hat{p}(\mathbf{z}_{1:T}|\mathbf{x}_{1:T})} \left[ \log \hat{p}(\mathbf{x}_{1:T}) \right] \\
        = & \mathbb{E}_{\hat{p}(\mathbf{z}_{1:T}|\mathbf{x}_{1:T})} \left[ \log \left( \prod_{t=1}^T \left( \frac{1}{K} \sum_{i=1}^K w_t^i \right) \right) \right] \\
        = & \sum_{t=1}^T \mathbb{E}_{\hat{p}(\mathbf{z}_{1:t-1}^{1:K}|\mathbf{x}_{1:t-1}) \cdot \prod_{i=1}^K  q(\mathbf{z}_{t}^{i}|\mathbf{z}_{1:t-1}^{i},\mathbf{x}_{1:t})} \left[ \log \left( \frac{1}{K} \sum_{i=1}^K w_t^i \right) \right]\\
        = &\sum_{t=1}^T \mathbb{E}_{\hat{p}(\mathbf{z}_{1:t-1}^{1:K}|\mathbf{x}_{1:t-1}) \cdot \prod_{i=1}^K  q(\mathbf{z}_{t}^{i}|\mathbf{z}_{1:t-1}^{i},\mathbf{x}_{1:t})} \left[ \log \left( \frac{1}{K} \sum_{i=1}^K \frac{p(\mathbf{z}_{1:t}^{i},\mathbf{x}_{1:t})}{p(\mathbf{z}_{1:t-1}^{i},\mathbf{x}_{1:t-1})q(\mathbf{z}_{t}^{i}|\mathbf{z}_{1:t-1}^{i},\mathbf{x}_{1:t})} \right)   \right].
    \end{aligned}
\end{equation*}
If replacing $p(\mathbf{z}_{1:t-1}^{1:K}|\mathbf{x}_{1:t-1})$ in $Q_t^K(\mathbf{z}_{1:t}^{1:K}|\mathbf{x}_{1:t})$ of MCFO by K sample approximation $\hat{p}(\mathbf{z}_{1:t-1}^{1:K}|\mathbf{x}_{1:t-1})$, the bound leads to the definition of FIVOs. %(where $\hat{\mathbf{z}}_{1:t}^{1:K}$ are sample trajectories before resampling step in SMC) 

By defining $\psi(\mathbf{z}_{1:t-1}^{1:K})$ as a test function
\begin{equation*}
    \psi(\mathbf{z}_{1:t-1}^{1:K}) =  \int \prod_{i=1}^K  q(\mathbf{z}_{t}^{i}|\mathbf{z}_{1:t-1}^{i},\mathbf{x}_{1:t}) \left[ \log \left( \frac{1}{K} \sum_{i=1}^K \frac{p(\mathbf{z}_{1:t}^{i},\mathbf{x}_{1:t})}{p(\mathbf{z}_{1:t-1}^{i},\mathbf{x}_{1:t-1})q(\mathbf{z}_{t}^{i}|\mathbf{z}_{1:t-1}^{i},\mathbf{x}_{1:t})} \right)   \right] \mathrm{d} \mathbf{z}_{t}^{1:K} , 
\end{equation*}
each term in MCFO is the expectation of the test function over $p(\mathbf{z}_{1:t-1}^{1:K}|\mathbf{x}_{1:t-1})$:
\begin{equation*}
    \mathcal{L}_t^K := \int p(\mathbf{z}_{1:t-1}^{1:K}|\mathbf{x}_{1:t-1}) \psi(\mathbf{z}_{1:t-1}^{1:K}) \mathrm{d} \mathbf{z}_{1:t-1}^{1:K},
\end{equation*}
while each term in FIVO is the expectation over $\hat{p}(\mathbf{z}_{1:t-1}^{1:K}|\mathbf{x}_{1:t-1})$
\begin{equation*}
    \hat{\mathcal{L}}_t^K := \int \hat{p}(\mathbf{z}_{1:t-1}^{1:K}|\mathbf{x}_{1:t-1}) \psi(\mathbf{z}_{1:t-1}^{1:K}) \mathrm{d} \mathbf{z}_{1:t-1}^{1:K} = \sum_{i=1}^K \Tilde{w}_{t-1}^i \psi(\hat{\mathbf{z}}_{1:t-1}^{1:K}).
\end{equation*}
Therefore, each $\hat{\mathcal{L}}_t^K$ of FIVOs could be considered as an estimate on $\mathcal{L}_t^K$ of MCFOs using approximations $\hat{p}(\mathbf{z}_{1:t-1}^{1:K}|\mathbf{x}_{1:t-1}) = \sum_{i=1}^K \Tilde{w}_{t-1}^{i} \delta(\mathbf{z}_{1:t-1}^{1:K}-\hat{\mathbf{z}}_{1:t-1}^{1:K})$. Although this estimate is biased for finite number of samples $K$, it is consistent and its asymptotic bias is given by
\begin{equation*}
    \begin{aligned}
        \lim_{K \rightarrow \infty} K(\mathcal{L}_t^K - \hat{\mathcal{L}}_t^K) & = \int \frac{p^2(\mathbf{z}_{1:t-1}^{1:K}|\mathbf{x}_{1:t-1})}{q(\mathbf{z}_{1:t-1}^{1:K})} (\psi(\mathbf{z}_{1:t-1}^{1:K}) - \mathcal{L}_t^K) \mathrm{d} \mathbf{z}_{1:t-1}^{1:K}, \\
    \end{aligned}
\end{equation*}
where $q(\mathbf{z}_{1:t-1}^{1:K})$ is proposal distribution of $p(\mathbf{z}_{1:t-1}^{1:K}|\mathbf{x}_{1:t-1})$, marginalized over resampling ancestral indexes. It also shows that the asymptotic bias is at $\mathcal{O}(1/K)$.

\begin{comment}

\begin{equation*}
    \begin{aligned}
        & \log p(\mathbf{x}_{1:t}|\mathbf{x}_{1:t-1}) = \mathbb{E}_{Q_t^K}[\log R_t^] + \mathrm{KL}[ \underbrace{\prod_{i=1}^K( p(\mathbf{z}_{1:t-1}^{i}|\mathbf{x}_{1:t-1}) q(\mathbf{z}_{t}^{i}|\mathbf{z}_{1:t-1}^{i},\mathbf{x}_{1:t}))}_{Q_t(\mathbf{z}_{0:t}^{1:K})} \parallel p(\mathbf{z}_{1:t}^{1:K}|\mathbf{x}_{1:t})]\\
        
        \int p(\mathbf{z}_{1:t-1}^{1:K}|\mathbf{x}_{1:t-1}) \prod_{i=1}^K q(\mathbf{z}_{t}^i|\mathbf{z}_{1:t-1}^i,\mathbf{x}_{1:t}) \\
    & \quad \underbrace{\quad \quad \quad \quad \quad \log \left( \frac{1}{K} \sum_{i=1}^K     
    \frac{p(\mathbf{z}_{1:t}^i,\mathbf{x}_{1:t})}{p(\mathbf{z}_{1:t-1}^i,\mathbf{x}_{1:t-1})q(\mathbf{z}_{t}^i|\mathbf{z}_{1:t-1}^i,\mathbf{x}_{1:t})} \right)  d \mathbf{z}_{1:t}^{1:K} }_{\mathcal{L}_t^K}

        & \log Z_1 = \mathbb{E}_{Q_1^K} \log \frac{1}{K} \sum_{i=1}^K \frac{p(\mathbf{x}_1, \mathbf{z}_1^{i}|\mathbf{z}_0^{i})}{q(\mathbf{z}_1^{i}|\mathbf{x}_1, \mathbf{z}_0^{i})} + \mathrm{KL}[\underbrace{\prod_{i=1}^K (p(\mathbf{z}_0^{i}) q(\mathbf{z}_1^{i}|\mathbf{x}_1, \mathbf{z}_0^{i}))}_{Q_1(\mathbf{z}_{0:1}^{1:K})} \parallel p(\mathbf{z}_1^{1:K}, \mathbf{z}_0^{1:K}|\mathbf{x}_1)]
    \end{aligned}
\end{equation*}

\end{comment}

\subsubsection{Properties of MCFOs}

\label{sec:appendix_lower_bound_property}
\begin{proposition*}
\label{properties}
(\textnormal{Properties of MCFOs}). Let $\mathcal{L}_{\text{MCFO}}^K$ be an MCFO of $\log p(\mathbf{x}_{1:T}) $ by a series of unbiased estimators $R_t$ of $p(\mathbf{x}_t|\mathbf{x}_{1:t-1})$ using $K$ samples. Then, 
\begin{enumerate}[label=\alph*)] % (a), (b), (c), ...
\item \label{prop:properties_1} \textnormal{(Bound)} $\log p(\mathbf{x}_{1:T}) \geq \mathcal{L}_{\text{MCFO}}^K$.
\item \label{prop:properties_2} \textnormal{(Monotonic convergence)} $\mathcal{L}_{\text{MCFO}}^{K+1} \geq \mathcal{L}_{\text{MCFO}}^{K} \geq  \hdots \geq  \mathcal{L}_{\text{MCFO}}^{1}$.
\item \label{prop:properties_3} \textnormal{(Consistency)} If $p(\mathbf{z}_{1},\mathbf{x}_{1}) /q(\mathbf{z}_{1}|\mathbf{x}_{1})$ and ${p(\mathbf{z}_{1:t},\mathbf{x}_{1:t})}/{(p(\mathbf{z}_{1:t-1},\mathbf{x}_{1:t-1})q(\mathbf{z}_{t}|\mathbf{z}_{1:t-1},\mathbf{x}_{1:t}))}$ for all $t\in[2,T]$ are bounded, then $\mathcal{L}_{\text{MCFO}}^K \rightarrow \log p(\mathbf{x}_{1:T})$ as K $\rightarrow \infty$.
\item \label{prop:properties_4} \textnormal{(Asymptotic Bias)} For a large K, the bias of bound is related to the variance of estimator $R_t$, $\mathbb{V}[R_t]$, 
\[\lim_{K\rightarrow \infty} K(\log p(\mathbf{x}_{1:T}) - \mathcal{L}_{\text{MCFO}}^K) = \sum_{t=1}^T \frac{\mathbb{V}[R_t]}{2 p(\mathbf{x}_t|\mathbf{x}_{1:t-1})^2}. \]
\end{enumerate}
\end{proposition*}

\begin{proof}

\begin{enumerate}[label=\alph*)]

\item Following Jensen's inequality for each $\log p(\mathbf{x}_t|\mathbf{x}_{1:t-1})$ and $\log p(\mathbf{x}_1)$ in (\ref{eq:mco_t}) and (\ref{eq:mco_1}) that
\begin{equation*}
    \begin{aligned}
        \log p(\mathbf{x}_{1:T}) & = \log p(\mathbf{x}_1) + \sum_{t=2}^T \log p(\mathbf{x}_t|\mathbf{x}_{1:t-1})\\
        & \geq \sum_{t=1}^T \mathcal{L}_t^K = \mathcal{L}_{\text{MCFO}}^K.
    \end{aligned}
\end{equation*}
\item
Following the similar logic as \citep[Theorem~1]{burda2015importance}, let $l$ and $k$ be positive integers and $k>l$, set $I \subset \{ 1, \hdots, k  \}$ with $|I| = l$ are uniformly distributed subset of unique indices of $\{1, \hdots, k\}$. Using
\[ \mathbb{E}_{I=\{i_1, \hdots, i_l \}} \left[ \frac{1}{l} \sum_{j=1}^l  a_{i_j} \right] = \frac{1}{k} \sum_{i=1}^k a_i,  \]
and Jensen's inequality, we obtain
\begin{equation*}
    \begin{aligned}
        \mathcal{L}_{t}^k & = \int p(\mathbf{z}_{1:t-1}^{1:k}|\mathbf{x}_{1:t-1})\prod_{i=1}^k q(\mathbf{z}_{t}^i|\mathbf{z}_{1:t-1}^i,\mathbf{x}_{1:t}) \\
        & \quad \quad \quad \quad \left[ \log \left( \frac{1}{k} \sum_{i=1}^k \frac{p(\mathbf{z}_{1:t}^i,\mathbf{x}_{1:t})}{p(\mathbf{z}_{1:t-1}^i,\mathbf{x}_{1:t-1})q(\mathbf{z}_{t}^i|\mathbf{z}^i_{1:t-1},\mathbf{x}_{1:t})} \right) \right] \mathrm{d} \mathbf{z}_{1:t}^{1:k}\\
        & = \mathbb{E}_{\mathbf{z}_{1:t}^{1}, \hdots, \mathbf{z}_{1:t}^{k}} \left[ \log  \left( \frac{1}{k} \sum_{i=1}^k \frac{p(\mathbf{z}_{1:t}^i,\mathbf{x}_{1:t})}{p(\mathbf{z}_{1:t-1}^i,\mathbf{x}_{1:t-1})q(\mathbf{z}_{t}^i|\mathbf{z}^i_{1:t-1},\mathbf{x}_{1:t})} \right) \right]\\
        & = \mathbb{E}_{\mathbf{z}_{1:t}^{1}, \hdots, \mathbf{z}_{1:t}^{k}} \left[ \log \mathbb{E}_{I=\{ i_1, \hdots, i_l \}} \left[  \frac{1}{l} \sum_{j=1}^l \frac{p(\mathbf{z}_{1:t}^{i_j},\mathbf{x}_{1:t})}{p(\mathbf{z}_{1:t-1}^{i_j},\mathbf{x}_{1:t-1})q(\mathbf{z}_{t}^{i_j}|\mathbf{z}^{i_j}_{1:t-1},\mathbf{x}_{1:t})} \right] \right]\\
        & \geq \mathbb{E}_{\mathbf{z}_{1:t}^{1}, \hdots, \mathbf{z}_{1:t}^{k}} \left[ \mathbb{E}_{I=\{ i_1, \hdots, i_l \}} \left[ \log \left( \frac{1}{l} \sum_{j=1}^l \frac{p(\mathbf{z}_{1:t}^{i_j},\mathbf{x}_{1:t})}{p(\mathbf{z}_{1:t-1}^{i_j},\mathbf{x}_{1:t-1})q(\mathbf{z}_{t}^{i_j}|\mathbf{z}^{i_j}_{1:t-1},\mathbf{x}_{1:t})} \right) \right] \right]\\
        & \geq \mathbb{E}_{\mathbf{z}_{1:t}^{1}, \hdots, \mathbf{z}_{1:t}^{l}} \left[ \log \left( \frac{1}{l} \sum_{i=1}^l \frac{p(\mathbf{z}_{1:t}^{i},\mathbf{x}_{1:t})}{p(\mathbf{z}_{1:t-1}^{i},\mathbf{x}_{1:t-1})q(\mathbf{z}_{t}^{i}|\mathbf{z}^{i}_{1:t-1},\mathbf{x}_{1:t})} \right) \right] = \mathcal{L}_{t}^l.
    \end{aligned}
\end{equation*}
Specifically, 
\begin{equation*}
    \begin{aligned}
        \mathcal{L}_1^k & = \int \prod_{i=1}^k q(\mathbf{z}_{1}^i|\mathbf{x}_{1}) \log \left( \frac{1}{k} \sum_{i=1}^k \frac{p(\mathbf{z}_{1}^i,\mathbf{x}_{1})}{q(\mathbf{z}_{1}^i|\mathbf{x}_{1})} \right) \mathrm{d} \mathbf{z}_{1}^{1:k}\\
        & = \mathbb{E}_{\mathbf{z}_{1}^{1}, \hdots, \mathbf{z}_{1}^{k}} \left[ \log \mathbb{E}_{I=\{ i_1, \hdots, i_l \}} \left[  \frac{1}{l} \sum_{j=1}^l \frac{p(\mathbf{z}_{1}^{i_j},\mathbf{x}_{1})}{q(\mathbf{z}_{1}^{i_j}|\mathbf{x}_{1})} \right] \right]\\
        & \geq \mathbb{E}_{\mathbf{z}_{1}^{1}, \hdots, \mathbf{z}_{1}^{l}} \left[ \log  \frac{1}{l} \sum_{i=1}^l \frac{p(\mathbf{z}_{1}^{i},\mathbf{x}_{1})}{q(\mathbf{z}_{1}^{i}|\mathbf{x}_{1})} \right] = \mathcal{L}_1^l.
    \end{aligned}
\end{equation*}
\vspace*{-\abovedisplayskip}
Therefore,
\begin{equation*}
    \begin{aligned}
        \mathcal{L}_{\text{MCFO}}^k & = \sum_{t=1}^T \mathcal{L}_t^k \geq \sum_{t=1}^T \mathcal{L}_t^l = \mathcal{L}_{\text{MCFO}}^l.
    \end{aligned}
\end{equation*}
\item
For $t=1$, consider estimator $R_1^K$ as a random variable. If $p(\mathbf{z}_{1},\mathbf{x}_{1}) /q(\mathbf{z}_{1}|\mathbf{x}_{1})$ is bounded, following strong law of large number that $R_1^K$ converges to $p(\mathbf{x}_1)$ almost surely. Since the logarithmic function is continuous, $\log R_1^K$ converges to $\log p(\mathbf{x}_1)$ almost surely. Similarly, if $p(\mathbf{z}_{1:t},\mathbf{x}_{1:t}) /p(\mathbf{z}_{1:t-1},\mathbf{x}_{1:t-1})q(\mathbf{z}_{t}|\mathbf{z}_{1:t-1},\mathbf{x}_{1:t})$ for all $t=2:T$, $\log R_t^K$ converges to $\log p(\mathbf{x}_t|\mathbf{x}_{1:t-1})$ almost surely. And $R_t^K$ for all $t$ is uniformly integrable. By Vitali's convergence theorem, $\mathcal{L}_{\text{MCFO}}^K = \sum_{t=1}^T \mathbb{E}[\log R_t^K] \rightarrow \log p(\mathbf{x}_{1:T})$ as $K \rightarrow \infty$.
\vspace*{-\abovedisplayskip}
\item
\citep[Theorem~3]{domke2018importance} proves that asymptotic bias for an MCO for unbiased estimator $R_t^K$ relates the variance of $R_t$ defined by single sample, $\mathbb{V}[R_t]$,
\begin{equation*}
    \lim_{K \rightarrow \infty} K(\log p(\mathbf{x}_{t}|\mathbf{x}_{1:t-1}) - \mathbb{E}_{Q_t^K} \left[ \log R_t^K \right]) = \frac{\mathbb{V}[R_t]}{2 \log p(\mathbf{x}_{t}|\mathbf{x}_{1:t-1})^2 },
\end{equation*}
where $R_t = \dfrac{p(\mathbf{z}_{1:t},\mathbf{x}_{1:t})}{p(\mathbf{z}_{1:t-1},\mathbf{x}_{1:t-1})q(\mathbf{z}_{t}|\mathbf{z}_{1:t-1},\mathbf{x}_{1:t})}$, $R_1 = \dfrac{p(\mathbf{z}_{1},\mathbf{x}_{1})}{q(\mathbf{z}_{1}|\mathbf{x}_{1})}$.
\vspace*{-\abovedisplayskip}
When $T$ is finite, 
\begin{equation*}
    \begin{aligned}
        & \lim_{K \rightarrow \infty} K(\log p(\mathbf{x}_{1:T}) - \mathcal{L}_{\text{MCFO}}^K) \\
        = & \lim_{K \rightarrow \infty} K\left( \sum_{t=1}^T \log p(\mathbf{x}_{t}|\mathbf{x}_{1:t-1}) - \sum_{t=1}^T \mathbb{E}_{Q_t^K}\left[\log R_t^K \right]\right) \\
        = &  \sum_{t=1}^T\lim_{K \rightarrow \infty}  K  \left(  \log p(\mathbf{x}_{t}|\mathbf{x}_{1:t-1}) - \mathbb{E}_{Q_t^K} \left[\log R_t^K \right] \right) \\
        = & \sum_{t=1}^T \frac{\mathbb{V}[R_t]}{2 \log p(\mathbf{x}_{t}|\mathbf{x}_{1:t-1})^2 }.
    \end{aligned}
\end{equation*}

Considering $T \rightarrow \infty$, using (\textbf{Tannery's Theorem}) that if $b_t = \dfrac{\mathbb{V}[R_t]}{2 \log p(\mathbf{x}_{t}|\mathbf{x}_{1:t-1})^2}$ remains bounded that $b_t \leq M_t$ for all $t \geq 1$ and $\sum_{t=1}^{\infty} M_t < \infty$, then
\begin{equation*}
    \begin{aligned}
        & \lim_{K \rightarrow \infty} K(\log p(\mathbf{x}_{1:T}) - \mathcal{L}_{\text{MCFO}}^K) \\
        = & \lim_{K \rightarrow \infty}  \sum_{t=1}^{\infty}  K \left(\log p(\mathbf{x}_{t}|\mathbf{x}_{1:t-1}) - \mathbb{E}_{Q_t^K} \left[ \log R_t^K \right]\right) \\
        = & \sum_{t=1}^{\infty} \lim_{K \rightarrow \infty}    K \left(\log p(\mathbf{x}_{t}|\mathbf{x}_{1:t-1}) - \mathbb{E}_{Q_t^K} \left[ \log R_t^K \right] \right) \\
        = & \sum_{t=1}^{\infty} \frac{\mathbb{V}[R_t]}{2 \log p(\mathbf{x}_{t}|\mathbf{x}_{1:t-1})^2 }.
    \end{aligned}
\end{equation*}
Therefore, asymptotic bias is valid for sequences of any length.

\end{enumerate}
\end{proof}

\subsubsection{Optimal Importance Proposals}
\label{sec:optimal_importance_proposal}
\begin{proposition*}
\textnormal{Optimal importance proposal $q^*$ for an MCFO}. The bound is maximized and exact to $\log p(\mathbf{x}_{1:T})$ when the optimal importance proposals are
\begin{equation*}
    \begin{aligned}
        & q^*(\mathbf{z}_{1}|\mathbf{x}_{1}) = p(\mathbf{z}_{1}|\mathbf{x}_{1}), \\
        & q^*(\mathbf{z}_{t}|\mathbf{z}_{1:t-1},\mathbf{x}_{1:t}) = p(\mathbf{z}_{t}|\mathbf{z}_{1:t-1}, \mathbf{x}_{1:t}) = \frac{p(\mathbf{z}_{1:t}|\mathbf{x}_{1:t})}{p(\mathbf{z}_{1:t-1}|\mathbf{x}_{1:t-1})},  & \text{for all } t=2:T. 
    \end{aligned}
\end{equation*}
\end{proposition*}

\begin{proof}
$\mathcal{L}_{\text{MCFO}}^K$ becomes exact \textit{if and only if} $\mathcal{L}_t^K$ is exact for all $t$. $\log p(\mathbf{x}_{t}|\mathbf{x}_{1:t-1})$ is composed of lower bound $\mathcal{L}_t^K$ and tightness gap:
\begin{equation*}
    \begin{aligned}
        \log p(\mathbf{x}_{t}|\mathbf{x}_{1:t-1}) & =  \int p(\mathbf{z}_{1:t-1}^{1:K}|\mathbf{x}_{1:t-1}) \prod_{i=1}^K q(\mathbf{z}_{t}^i|\mathbf{z}_{1:t-1}^i,\mathbf{x}_{1:t}) \\
        & \quad \underbrace{\quad \quad \quad \quad \quad \log \left( \frac{1}{K} \sum_{i=1}^K     
        \frac{p(\mathbf{z}_{1:t}^i,\mathbf{x}_{1:t})}{p(\mathbf{z}_{1:t-1}^i,\mathbf{x}_{1:t-1})q(\mathbf{z}_{t}^i|\mathbf{z}_{1:t-1}^i,\mathbf{x}_{1:t})} \right)  \mathrm{d} \mathbf{z}_{1:t}^{1:K} }_{\mathcal{L}_t^K}\\
        & + \int p(\mathbf{z}_{1:t-1}^{1:K}|\mathbf{x}_{1:t-1}) \prod_{i=1}^K q(\mathbf{z}_{t}^i|\mathbf{z}_{1:t-1}^i,\mathbf{x}_{1:t})\\
        & \quad \quad  \quad \quad   \log  \dfrac{p(\mathbf{x}_t|\mathbf{x}_{1:t-1})}{\frac{1}{K} \sum_{i=1}^K     
        \frac{p(\mathbf{z}_{1:t}^i,\mathbf{x}_{1:t})}{p(\mathbf{z}_{1:t-1}^i,\mathbf{x}_{1:t-1})q(\mathbf{z}_{t}^i|\mathbf{z}_{1:t-1}^i,\mathbf{x}_{1:t})}} \mathrm{d} \mathbf{z}_{1:t}^{1:K},
        %\mathrm{KL}[ p(\mathbf{z}_{1:t-1}^{1:K}|\mathbf{x}_{1:t-1}) \prod_{i=1}^K q(\mathbf{z}_{t}^i|\mathbf{z}_{1:t-1}^i,\mathbf{x}_{1:t}) \parallel p(\mathbf{z}_{1:t}^{1:K}|\mathbf{x}_{1:t})]\\
    \end{aligned}
\end{equation*}
where
\begin{equation*}
    \begin{aligned}
        & \int p(\mathbf{z}_{1:t-1}^{1:K}|\mathbf{x}_{1:t-1}) \prod_{i=1}^K q(\mathbf{z}_{t}^i|\mathbf{z}_{1:t-1}^i,\mathbf{x}_{1:t}) \log  \dfrac{p(\mathbf{x}_t|\mathbf{x}_{1:t-1})}{\frac{1}{K} \sum_{i=1}^K     
        \frac{p(\mathbf{z}_{1:t}^i,\mathbf{x}_{1:t})}{p(\mathbf{z}_{1:t-1}^i,\mathbf{x}_{1:t-1})q(\mathbf{z}_{t}^i|\mathbf{z}_{1:t-1}^i,\mathbf{x}_{1:t})}} \mathrm{d} \mathbf{z}_{1:t}^{1:K}\\
        = & -\int p(\mathbf{z}_{1:t-1}^{1:K}|\mathbf{x}_{1:t-1}) \prod_{i=1}^K q(\mathbf{z}_{t}^i|\mathbf{z}_{1:t-1}^i,\mathbf{x}_{1:t}) \log 
        \frac{1}{K} \sum_{i=1}^K     
        \frac{p(\mathbf{z}_{1:t}^i|\mathbf{x}_{1:t})}{p(\mathbf{z}_{1:t-1}^i|\mathbf{x}_{1:t-1}) q(\mathbf{z}_{t}^i|\mathbf{z}_{1:t-1}^i,\mathbf{x}_{1:t})}
        \mathrm{d} \mathbf{z}_{1:t}^{1:K}.
    \end{aligned}
\end{equation*}
For any $K$ and $\mathbf{z}_{1:t}^i$ $\forall i$, the lower bound $\mathcal{L}_t^K$ is exact when the gap becomes zero that
\begin{equation*}
    q^{*}(\mathbf{z}_{t}|\mathbf{z}_{1:t-1},\mathbf{x}_{1:t}) = \frac{p(\mathbf{z}_{1:t}|\mathbf{x}_{1:t})}{p(\mathbf{z}_{1:t-1}|\mathbf{x}_{1:t-1})}.
\end{equation*}
Specifically, $\log p(\mathbf{x}_1)$ is decomposed to lower bound $\mathcal{L}_1^K$ and the gap:
\begin{equation*}
    \begin{aligned}
        \log p(\mathbf{x}_1) & = \underbrace{\int \prod_{i=1}^K q(\mathbf{z}_{1}^i|\mathbf{x}_{1}) \log \left( \frac{1}{K} \sum_{i=1}^K \frac{p(\mathbf{z}_{1}^i,\mathbf{x}_{1})}{q(\mathbf{z}_{1}^i|\mathbf{x}_{1})} \right) \mathrm{d} \mathbf{z}_{1}^{1:K}}_{\mathcal{L}_1^K}\\
        & \quad \quad + \int \prod_{i=1}^K q(\mathbf{z}_{1}^i|\mathbf{x}_{1}) \log \dfrac{p(\mathbf{x}_1)}{\frac{1}{K} \sum_{i=1}^K \frac{p(\mathbf{z}_{1}^i,\mathbf{x}_{1})}{q(\mathbf{z}_{1}^i|\mathbf{x}_{1})}} \mathrm{d} \mathbf{z}_{1}^{1:K},
    \end{aligned}
\end{equation*}
where
\begin{equation*}
    \int \prod_{i=1}^K q(\mathbf{z}_{1}^i|\mathbf{x}_{1}) \log \dfrac{p(\mathbf{x}_1)}{\frac{1}{K} \sum_{i=1}^K \frac{p(\mathbf{z}_{1}^i,\mathbf{x}_{1})}{q(\mathbf{z}_{1}^i|\mathbf{x}_{1})}} \mathrm{d} \mathbf{z}_{1}^{1:K} = -\int \prod_{i=1}^K q(\mathbf{z}_{1}^i|\mathbf{x}_{1}) \log \frac{1}{K} \sum_{i=1}^K \frac{p(\mathbf{z}_{1}^i|\mathbf{x}_{1})}{q(\mathbf{z}_{1}^i|\mathbf{x}_{1})} \mathrm{d} \mathbf{z}_{1}^{1:K}.
\end{equation*}
For any $K$ and $\mathbf{z}_{1}^i$ $\forall i$, the lower bound $\mathcal{L}_1^K$ is exact when the gap becomes zero that
\begin{equation*}
    q^{*}(\mathbf{z}_{1}|\mathbf{x}_{1}) = p(\mathbf{z}_{1}|\mathbf{x}_{1}).
\end{equation*}
\end{proof}

%\subsection{Gradient derivation}
\subsubsection{Gradient $\triangledown_{{\boldtheta, \boldphi}} \mathbb{E}_{Q_t^K}[\log R_t^K]$}
\label{sec:appendix_gradient_theta}

\vspace*{-\abovedisplayskip}

\begin{equation*}
    \begin{aligned}
        \triangledown_{\boldtheta, \boldphi} \mathcal{L}_t^K = \triangledown_{{\boldtheta, \boldphi}} \mathbb{E}_{Q_t^K}[\log R_t^K] & =  \underbrace{\int \triangledown_{\boldtheta, \boldphi} \left( p_{\boldtheta}(\mathbf{z}_{1:t-1}^{1:K}|\mathbf{x}_{1:t-1}) \prod_{i=1}^K q_{\boldphi}(\mathbf{z}_{t}^{i}|\mathbf{z}_{1:t-1}^{i},\mathbf{x}_{1:t}) \right) \log R_t^K \mathrm{d} \mathbf{z}_{1:t}^{1:K}}_{\circled{1}}\\
        & \quad \quad  + \underbrace{\int p_{\boldtheta}(\mathbf{z}_{1:t-1}^{1:K}|\mathbf{x}_{1:t-1}) \prod_{i=1}^K q_{\boldphi}(\mathbf{z}_{t}^{i}|\mathbf{z}_{1:t-1}^{i},\mathbf{x}_{1:t}) \triangledown_{\boldtheta, \boldphi} \log R_t^K \mathrm{d} \mathbf{z}_{1:t}^{1:K} }_{\circled{2}},
    \end{aligned}
\end{equation*}
\vspace*{-\abovedisplayskip}
where
\begin{equation*}
    \begin{aligned}
        \triangledown_{\boldtheta} \log R_t^K & = \dfrac{ \dfrac{1}{K} \displaystyle \sum_{i=1}^K \triangledown_{\boldtheta} \left( \dfrac{p_{\boldtheta}(\mathbf{z}_{1:t}^{i},\mathbf{x}_{1:t})}{p_{\boldtheta}(\mathbf{z}_{1:t-1}^{i},\mathbf{x}_{1:t-1})q_{\boldphi}(\mathbf{z}_{t}^{i}|\mathbf{z}_{1:t-1}^{i},\mathbf{x}_{1:t})}\right)}{\dfrac{1}{K} \displaystyle \sum_{i=1}^K \dfrac{p_{\boldtheta}(\mathbf{z}_{1:t}^{i},\mathbf{x}_{1:t})}{p_{\boldtheta}(\mathbf{z}_{1:t-1}^{i},\mathbf{x}_{1:t-1})q_{\boldphi}(\mathbf{z}_{t}^{i}|\mathbf{z}_{1:t-1}^{i},\mathbf{x}_{1:t})}} \\
        & = \dfrac{\displaystyle \sum_{i=1}^K \triangledown_{\boldtheta}\log p_{\boldtheta}(\mathbf{z}_{t}^{i},\mathbf{x}_{t}|\mathbf{z}_{1:t-1}^{i},\mathbf{x}_{1:t-1}) \cdot \dfrac{p_{\boldtheta}(\mathbf{z}_{t}^{i},\mathbf{x}_{t}|\mathbf{z}_{1:t-1}^{i},\mathbf{x}_{1:t-1})}{q_{\boldphi}(\mathbf{z}_{t}^{i}|\mathbf{z}_{1:t-1}^{i},\mathbf{x}_{1:t})}}{\displaystyle \sum_{i=1}^K \dfrac{p_{\boldtheta}(\mathbf{z}_{t}^{i},\mathbf{x}_{t}|\mathbf{z}_{1:t-1}^{i},\mathbf{x}_{1:t-1})}{q_{\boldphi}(\mathbf{z}_{t}^{i}|\mathbf{z}_{1:t-1}^{i},\mathbf{x}_{1:t})}}\\
        & = \sum_{i=1}^K \Tilde{w}_{t, \boldtheta, \boldphi}^i \triangledown_{\boldtheta}\log p_{\boldtheta}(\mathbf{z}_{t}^{i},\mathbf{x}_{t}|\mathbf{z}_{1:t-1}^{i},\mathbf{x}_{1:t-1}),
    \end{aligned}
    %\label{eq:gradient_log_R}
\end{equation*}
\vspace*{-\abovedisplayskip}
\begin{equation*}
    \begin{aligned}
        \triangledown_{\boldphi} \log R_t^K & = \dfrac{ \dfrac{1}{K} \displaystyle \sum_{i=1}^K \triangledown_{\boldphi} \left( \dfrac{p_{\boldtheta}(\mathbf{z}_{1:t}^{i},\mathbf{x}_{1:t})}{p_{\boldtheta}(\mathbf{z}_{1:t-1}^{i},\mathbf{x}_{1:t-1})q_{\boldphi}(\mathbf{z}_{t}^{i}|\mathbf{z}_{1:t-1}^{i},\mathbf{x}_{1:t})}\right)}{\dfrac{1}{K} \displaystyle \sum_{i=1}^K \dfrac{p_{\boldtheta}(\mathbf{z}_{1:t}^{i},\mathbf{x}_{1:t})}{p_{\boldtheta}(\mathbf{z}_{1:t-1}^{i},\mathbf{x}_{1:t-1})q_{\boldphi}(\mathbf{z}_{t}^{i}|\mathbf{z}_{1:t-1}^{i},\mathbf{x}_{1:t})}} \\
        & = - \sum_{i=1}^K \Tilde{w}_{t, \boldtheta, \boldphi}^i \triangledown_{\boldtheta}\log q_{\boldphi}(\mathbf{z}_{t}^{i}|\mathbf{z}_{1:t-1}^{i},\mathbf{x}_{1:t}),\\
        & \triangledown_{\boldtheta} \left( p_{\boldtheta}(\mathbf{z}_{1:t-1}^{1:K}|\mathbf{x}_{1:t-1}) \prod_{i=1}^K q_{\boldphi}(\mathbf{z}_{t}^{i}|\mathbf{z}_{1:t-1}^{i},\mathbf{x}_{1:t}) \right) \\
        = & \triangledown_{\boldtheta} \log p_{\boldtheta}(\mathbf{z}_{1:t-1}^{1:K}|\mathbf{x}_{1:t-1}) \left( p_{\boldtheta}(\mathbf{z}_{1:t-1}^{1:K}|\mathbf{x}_{1:t-1}) \prod_{i=1}^K q_{\boldphi}(\mathbf{z}_{t}^{i}|\mathbf{z}_{1:t-1}^{i},\mathbf{x}_{1:t}) \right), \\
        & \triangledown_{\boldphi} \left( p_{\boldtheta}(\mathbf{z}_{1:t-1}^{1:K}|\mathbf{x}_{1:t-1}) \prod_{i=1}^K q_{\boldphi}(\mathbf{z}_{t}^{i}|\mathbf{z}_{1:t-1}^{i},\mathbf{x}_{1:t}) \right)\\
        = & \sum_{i=1}^K \triangledown_{\boldphi} \log q_{\boldphi}(\mathbf{z}_{t}^{i}|\mathbf{z}_{1:t-1}^{i},\mathbf{x}_{1:t})  \left( p_{\boldtheta}(\mathbf{z}_{1:t-1}^{1:K}|\mathbf{x}_{1:t-1}) \prod_{i=1}^K q_{\boldphi}(\mathbf{z}_{t}^{i}|\mathbf{z}_{1:t-1}^{i},\mathbf{x}_{1:t}) \right).
    \end{aligned}
\end{equation*}
\iffalse
\vspace*{-\abovedisplayskip}
\begin{equation*}
    \begin{aligned}
        & \triangledown_{\boldtheta} \left( p_{\boldtheta}(\mathbf{z}_{1:t-1}^{1:K}|\mathbf{x}_{1:t-1}) \prod_{i=1}^K q_{\boldphi}(\mathbf{z}_{t}^{i}|\mathbf{z}_{1:t-1}^{i},\mathbf{x}_{1:t}) \right) \\
        = & \triangledown_{\boldtheta} \log p_{\boldtheta}(\mathbf{z}_{1:t-1}^{1:K}|\mathbf{x}_{1:t-1}) \left( p_{\boldtheta}(\mathbf{z}_{1:t-1}^{1:K}|\mathbf{x}_{1:t-1}) \prod_{i=1}^K q_{\boldphi}(\mathbf{z}_{t}^{i}|\mathbf{z}_{1:t-1}^{i},\mathbf{x}_{1:t}) \right), \\
        & \triangledown_{\boldphi} \left( p_{\boldtheta}(\mathbf{z}_{1:t-1}^{1:K}|\mathbf{x}_{1:t-1}) \prod_{i=1}^K q_{\boldphi}(\mathbf{z}_{t}^{i}|\mathbf{z}_{1:t-1}^{i},\mathbf{x}_{1:t}) \right)\\
        = & \sum_{i=1}^K \triangledown_{\boldphi} \log q_{\boldphi}(\mathbf{z}_{t}^{i}|\mathbf{z}_{1:t-1}^{i},\mathbf{x}_{1:t})  \left( p_{\boldtheta}(\mathbf{z}_{1:t-1}^{1:K}|\mathbf{x}_{1:t-1}) \prod_{i=1}^K q_{\boldphi}(\mathbf{z}_{t}^{i}|\mathbf{z}_{1:t-1}^{i},\mathbf{x}_{1:t}) \right).
    \end{aligned}
\end{equation*}
\fi

Therefore, the gradient with respect to $\boldtheta$ is estimated by:
\vspace*{-\abovedisplayskip}
\begin{equation*}
    \begin{aligned}
        \circled{1} & \simeq \log R_t^K(\hat{\mathbf{z}}_{1:t}^{1:K},\mathbf{x}_{1:t}) \triangledown_{\boldtheta} \log p_{\boldtheta}(\hat{\mathbf{z}}_{1:t-1}^{1:K}|\mathbf{x}_{1:t-1}),   \\
        \circled{2} & \simeq \sum_{i=1}^K \Tilde{w}_{t, \boldtheta, \boldphi}^i \triangledown_{\boldtheta}\log p_{\boldtheta}(\hat{\mathbf{z}}_{t}^{i},\mathbf{x}_{t}|\hat{\mathbf{z}}_{1:t-1}^{i},\mathbf{x}_{1:t-1}),\\
        & \quad \quad \quad \hat{\mathbf{z}}_{1:t-1}^{i} \sim p_{\boldtheta}(\mathbf{z}_{1:t-1}|\mathbf{x}_{1:t-1}), \quad \quad \hat{\mathbf{z}}_{t}^{i} \sim q_{\boldphi}(\mathbf{z}_{t}|\hat{\mathbf{z}}_{1:t-1}^{i},\mathbf{x}_{1:t}),
    \end{aligned}
\end{equation*}
and the gradient with respect to $\boldphi$ is estimated by:
\vspace*{-\abovedisplayskip}
\begin{equation*}
    \begin{aligned}
        \circled{1} & \simeq \log R_t^K(\hat{\mathbf{z}}_{1:t}^{1:K},\mathbf{x}_{1:t}) \sum_{i=1}^K  \triangledown_{\boldphi} \log q_{\boldphi}(\hat{\mathbf{z}}_{t}^{i}|\hat{\mathbf{z}}_{1:t-1}^{i},\mathbf{x}_{1:t}), \\
        \circled{2} & \simeq - \sum_{i=1}^K \Tilde{w}_{t, \boldtheta, \boldphi}^i \triangledown_{\boldtheta}\log q_{\boldphi}(\hat{\mathbf{z}}_{t}^{i}|\hat{\mathbf{z}}_{1:t-1}^{i},\mathbf{x}_{1:t}),\\
        & \quad \quad \quad \hat{\mathbf{z}}_{1:t-1}^{i} \sim p_{\boldtheta}(\mathbf{z}_{1:t-1}|\mathbf{x}_{1:t-1}), \quad \quad \hat{\mathbf{z}}_{t}^{i} \sim q_{\boldphi}(\mathbf{z}_{t}|\hat{\mathbf{z}}_{1:t-1}^{i},\mathbf{x}_{1:t}).
    \end{aligned}
\end{equation*}

\subsubsection{Score Function of $ p_{\boldtheta}(\mathbf{x}_{t}|\mathbf{x}_{1:t-1})$}
\label{sec:appendix_gradient_theta_rws}
Fisher's identity gives the score function of $ p_{\boldtheta}(\mathbf{x}_{t}|\mathbf{x}_{1:t-1})$ related to the distribution and derivative with respect to parameters: % the gradient of marginal log-likelihood with the gradient of joint log-likelihood with hidden variables:
\iftrue
\begin{equation*}
    \begin{aligned}
        \triangledown_{\boldtheta} \log p_{\boldtheta}(\mathbf{x}_{1:t}) 
        & = \dfrac{\triangledown_{\boldtheta} p_{\boldtheta}(\mathbf{x}_{1:t})}{p_{\boldtheta}(\mathbf{x}_{1:t})}\\
        & = \int \dfrac{\triangledown_{\boldtheta} p_{\boldtheta}(\mathbf{x}_{1:t}, \mathbf{z}_{1:t})}{p_{\boldtheta}(\mathbf{x}_{1:t})} d \mathbf{z}_{1:t}\\
        & = \int \dfrac{\triangledown_{\boldtheta} \log p_{\boldtheta}(\mathbf{x}_{1:t}, \mathbf{z}_{1:t}) p_{\boldtheta}(\mathbf{x}_{1:t}, \mathbf{z}_{1:t})}{p_{\boldtheta}(\mathbf{x}_{1:t})} d \mathbf{z}_{1:t}\\
        & = \int p_{\boldtheta}(\mathbf{z}_{1:t}|\mathbf{x}_{1:t}) \triangledown_{\boldtheta} \log p_{\boldtheta}(\mathbf{x}_{1:t}, \mathbf{z}_{1:t})  d \mathbf{z}_{1:t}.\\
    \end{aligned}
    \label{eq:gradient_theta_rws_fisher}
\end{equation*}
\fi

\iffalse
\begin{equation}
    \begin{aligned}
        \triangledown_{\boldtheta} \log p_{\boldtheta}(\mathbf{x}_{t}|\mathbf{x}_{1:t-1}) 
        & = \dfrac{\triangledown_{\boldtheta} p_{\boldtheta}(\mathbf{x}_{t}|\mathbf{x}_{1:t-1})}{p_{\boldtheta}(\mathbf{x}_{t}|\mathbf{x}_{1:t-1})}\\
        & = \int \dfrac{\triangledown_{\boldtheta} p_{\boldtheta}(\mathbf{x}_{1:T}, \mathbf{z}_{1:T})}{p_{\boldtheta}(\mathbf{x}_{1:T})} d \mathbf{z}_{1:T}\\
        & = \int \dfrac{\triangledown_{\boldtheta} \log p_{\boldtheta}(\mathbf{x}_{1:T}, \mathbf{z}_{1:T}) p_{\boldtheta}(\mathbf{x}_{1:T}, \mathbf{z}_{1:T})}{p_{\boldtheta}(\mathbf{x}_{1:T})} d \mathbf{z}_{1:T}\\
        & = \int p_{\boldtheta}(\mathbf{z}_{1:T}|\mathbf{x}_{1:t}) \triangledown_{\boldtheta} \log p_{\boldtheta}(\mathbf{x}_{1:t}, \mathbf{z}_{1:t})  d \mathbf{z}_{1:t}.\\
    \end{aligned}
    \label{eq:gradient_theta_rws_fisher}
\end{equation}
\fi

Assuming the filtering problem that the latent variable does not condition on future observations, the score function of $p_{\boldtheta}(\mathbf{x}_{t}|\mathbf{x}_{1:t-1})$ using SMC estimates becomes:
%Assuming latent variables having Markov property and conditional independence on observations, (\ref{eq:gradient_theta_rws_fisher}) is further simplified and estimated by Sequential Monte Carlo:
\begin{equation*}
    \begin{aligned}
        \triangledown_{\boldtheta} \log p_{\boldtheta}(\mathbf{x}_{t}|\mathbf{x}_{1:t-1}) 
        & = \triangledown_{\boldtheta} \log p_{\boldtheta}(\mathbf{x}_{1:t}) - \triangledown_{\boldtheta} \log p_{\boldtheta}(\mathbf{x}_{1:t-1}) \\
        & = \int p_{\boldtheta}(\mathbf{z}_{1:t}|\mathbf{x}_{1:t}) \triangledown_{\boldtheta} \log p_{\boldtheta}(\mathbf{x}_{1:t}, \mathbf{z}_{1:t})  d \mathbf{z}_{1:t} \\
        & \quad \quad \quad \quad  - \int p_{\boldtheta}(\mathbf{z}_{1:t-1}|\mathbf{x}_{1:t-1}) \triangledown_{\boldtheta} \log p_{\boldtheta}(\mathbf{x}_{1:t-1}, \mathbf{z}_{1:t-1})  d \mathbf{z}_{1:t-1}\\
        & = \int p_{\boldtheta}(\mathbf{z}_{1:t}|\mathbf{x}_{1:t}) \triangledown_{\boldtheta} \log p_{\boldtheta}(\mathbf{x}_{t}, \mathbf{z}_{t}|\mathbf{x}_{1:t-1}, \mathbf{z}_{1:t-1})  d \mathbf{z}_{1:t} \\
        % & \sum_{t=1}^T \int \triangledown_{\boldtheta} \log p_{\boldtheta}(\mathbf{x}_{t},\mathbf{z}_{t}|\mathbf{x}_{1:t-1}, \mathbf{z}_{1:t-1}) p_{\boldtheta}(\mathbf{z}_{1:t}|\mathbf{x}_{1:t}) \mathrm{d} \mathbf{z}_{1:t}\\
        %& = \sum_{t=1}^T \int \triangledown_{\boldtheta} \left( \log p_{\boldtheta}(\mathbf{x}_{t}|\mathbf{z}_{t}) + \log p_{\boldtheta}(\mathbf{z}_{t}|\mathbf{z}_{1:t}) \right) p_{\boldtheta}(\mathbf{z}_{1:t}|\mathbf{x}_{1:t}) \mathrm{d} \mathbf{z}_{1:t}\\
        & \simeq \sum_{i=1}^K \Tilde{w}_{t, \boldtheta, \boldphi}^i \triangledown_{\boldtheta} p_{\boldtheta}(\mathbf{x}_{t},\hat{\mathbf{z}}_{t}^{i}|\mathbf{x}_{1:t-1}, \hat{\mathbf{z}}_{1:t-1}^{i}),\\
        % & \simeq \sum_{t=1}^T \sum_{i=1}^K \Tilde{w}_{t, \boldtheta, \boldphi}^i (\triangledown_{\boldtheta} \log p_{\boldtheta}(\mathbf{x}_{t}| \mathbf{z}_{t}^{i})+\triangledown_{\boldtheta} \log p_{\boldtheta}(\mathbf{z}_{t}^{i}| \mathbf{z}_{1:t-1}^{i})),\\
        %& \quad \quad \Tilde{w}^i_t = \frac{w^i_t}{\sum_{i=1}^K w^i_t}, \quad \quad w^i_t = \frac{p_{\boldtheta}(\mathbf{x}_{t}| \mathbf{z}_{t}^{i}) p_{\boldtheta}(\mathbf{z}_{t}^{i}| \mathbf{z}_{1:t-1}^{i})}{q(\mathbf{z}_{t}^{i}|\mathbf{z}_{t-1}^{i}, \mathbf{x}_{t})} \\
        %& \quad \quad \mathbf{z}_{1:t-1}^{i} \sim p_{\boldtheta}(\mathbf{z}_{1:t-1}|\mathbf{x}_{1:t-1}), \quad \quad \mathbf{z}_{t}^{i} \sim q_{\boldphi}(\mathbf{z}_{t}|\mathbf{z}_{1:t-1}^{i}, \mathbf{x}_{t})
        & \quad \quad \quad \Tilde{w}_{t, \boldtheta, \boldphi}^i = \dfrac{w_{t, \boldtheta, \boldphi}^i}{\sum_{j=1}^K w_{t, \boldtheta, \boldphi}^j}, w_{t, \boldtheta, \boldphi}^i = \dfrac{p_{\boldtheta}(\mathbf{z}_{t}^{i},\mathbf{x}_{t}|\mathbf{z}_{1:t-1}^{i},\mathbf{x}_{1:t-1})}{q_{\boldphi}(\mathbf{z}_{t}^{i}|\mathbf{z}_{1:t-1}^{i},\mathbf{x}_{1:t})},\\
        & \quad \quad \quad \hat{\mathbf{z}}_{1:t-1}^{i} \sim p_{\boldtheta}(\mathbf{z}_{1:t-1}|\mathbf{x}_{1:t-1}), \quad \quad \hat{\mathbf{z}}_{t}^{i} \sim q_{\boldphi}(\mathbf{z}_{t}|\hat{\mathbf{z}}_{1:t-1}^{i},\mathbf{x}_{1:t}).
    \end{aligned}
    \label{eq:gradient_theta_rws_derive}
\end{equation*}

\subsection{Linear Gaussian State Space Model (LGSSM)}
\subsubsection{Optimal Proposals of LGSSM}
\label{sec:appendix_optimal_proposal_lgssm}
The optimal proposal $q(z_t|x_{1:t}, z_{1:t-1})$ for LGSSM defined in (\ref{eq:lgssm}) to have zero variance of important weights is:
\begin{equation*}
    \begin{aligned}
        & q^*(z_1|x_{1}) = p(z_1|x_{1}), \\
        & q^*(z_t|z_{1:t-1},x_{1:t}) = p(z_t|x_t, z_{t-1}),
    \end{aligned}
\end{equation*}
where 
\begin{equation*}
    \begin{aligned}
        & p(z_1|x_{1}) = \dfrac{p(z_1,x_1)}{p(x_1)} = \dfrac{p(x_1|z_1) p(z_1)}{\int p(x_1|z_1) p(z_1) \mathrm{d} z_1} = \dfrac{\mathcal{N}(x_1; \theta_2 z_1, \Sigma_R)\mathcal{N}(z_1; \mu_0, \sigma_0^2)}{\int \mathcal{N}(x_1; \theta_2 z_1, \Sigma_R)\mathcal{N}(z_1; \mu_0, \sigma_0^2) \mathrm{d} z_1}\\
        = & \mathcal{N}(z_1; \frac{\Sigma_R}{\Sigma_R+\sigma_0^2 \theta_2^2}\mu_0 + \frac{\sigma_0^2 \theta_2}{\Sigma_R +\sigma_0^2 \theta_2^2}x_1, \frac{\sigma_0^2 \Sigma_R}{\Sigma_R+\sigma_0^2 \theta_2^2}),\\
        & p(z_t|x_t, z_{t-1}) =  \dfrac{p(x_t|z_t) p(z_t|z_{t-1})}{p(x_t|z_{t-1})} = \dfrac{\mathcal{N}(x_t; \theta_2 z_t, \Sigma_R)\mathcal{N}(z_t; \theta_1 z_{t-1}, \Sigma_Q)}{\int \mathcal{N}(x_t; \theta_2 z_t, \Sigma_R)\mathcal{N}(z_t; \theta_1 z_{t-1}, \Sigma_Q) \mathrm{d} z_t}\\
        = & \mathcal{N}(z_t; \frac{\Sigma_R \theta_1}{\Sigma_R+\Sigma_Q \theta_2^2}z_{t-1} + \frac{\Sigma_Q \theta_2}{\Sigma_R+\Sigma_Q \theta_2^2} x_t, \frac{\Sigma_Q \Sigma_R}{\Sigma_R+\Sigma_Q \theta_2^2}).
    \end{aligned}
\end{equation*}

\subsubsection{Gradient $\triangledown_{\boldtheta} \log p(x_{1:T})$}
\label{sec:appendix_lgssm_gradient}
$\log p(x_{1:T})$ for LGSSM defined in (\ref{eq:lgssm}) is tractable:
\begin{equation*}
    \begin{aligned}
        & \log p(x_{1:T}) = \log p(x_1) + \sum_{t=2}^T \log p(x_t|x_{1:t-1}),\\
        & \log p(x_1) \sim \mathcal{N}(x_1; \theta_2 m_1^{-}(\boldtheta), S_1(\boldtheta)),\\
        & \log p(x_t|x_{1:t-1}) \sim \mathcal{N}(x_t;\theta_2 m_t^{-}(\boldtheta), S_t(\boldtheta)),
    \end{aligned}
\end{equation*}   
where
\begin{equation*}
    \begin{aligned}
        & m_t^{-}(\boldtheta)=\theta_1 m_{t-1}(\boldtheta), \quad \quad S_t(\boldtheta) = \theta_2^2 P_t^{-}(\boldtheta) + \Sigma_R, \quad \quad P_t^{-}(\boldtheta) =  \theta_1^2 P_{t-1}(\boldtheta) + \Sigma_Q, \\
        & m_1^{-}(\boldtheta) = \mu_0, \quad \quad S_1(\boldtheta) = \theta_2^2 \sigma_0^2 + \Sigma_R,\\
        & v_t(\boldtheta) = x_t - \theta_2 m_t^{-}(\boldtheta), \quad \quad 
        m_t(\boldtheta) = m_t^{-}(\boldtheta) + K_t(\boldtheta) v_t(\boldtheta), \\
        & K_t(\boldtheta) = \dfrac{\theta_2 P_t^{-}(\boldtheta)}{S_t(\boldtheta)}, \quad \quad 
        P_t(\boldtheta) = P_t^{-}(\boldtheta) - K_t(\boldtheta)^2 S_t(\boldtheta).
    \end{aligned}
\end{equation*}
Gradient of joint marginal log-likelihood w.r.t. $\boldtheta$, $\triangledown_{\boldtheta} \log p(x_{1:T})$:
\begin{equation*}
    \allowdisplaybreaks
    \begin{aligned}
        & \triangledown_{\boldtheta} \log p(x_{1:T}) = \triangledown_{\boldtheta} \log p(x_1) + \sum_{t=2}^T \triangledown_{\boldtheta} \log p(x_t|x_{1:t-1}),\\
        & \triangledown_{\boldtheta} \log p(x_t|x_{1:t-1}) = -\frac{1}{2} \triangledown_{\boldtheta} \log S_t(\boldtheta) + \frac{1}{2} \dfrac{\triangledown_{\boldtheta} S_t(\boldtheta) v_t(\boldtheta)^2}{S_t(\boldtheta)^2}  - \dfrac{v_t(\boldtheta)}{S_t(\boldtheta)},\\
        & \triangledown_{\boldtheta} \log p(x_1) = -\frac{1}{2} \triangledown_{\boldtheta} \log S_1(\boldtheta) + \frac{1}{2} \dfrac{\triangledown_{\boldtheta} S_1(\boldtheta) v_1(\boldtheta)^2}{S_1(\boldtheta)^2}  - \dfrac{v_1(\boldtheta)}{S_1(\boldtheta)},
    \end{aligned}
\end{equation*}
where
\begin{equation*}
    \allowdisplaybreaks
    \begin{aligned}
        & \triangledown_{\theta_1}S_t(\boldtheta) = \triangledown_{\theta_1}P_t^{-}(\boldtheta) \theta_2^2,
        \quad \quad \triangledown_{\theta_2}S_t(\boldtheta) = \triangledown_{\theta_2}P_t^{-}(\boldtheta) \theta_2^2 + 2 P_t^{-}(\boldtheta) \theta_2,\\     
        & \triangledown_{\theta_1}P_t^{-}(\boldtheta) = 2\theta_1 P_{t-1}(\boldtheta) + \theta_1^2 \triangledown_{\theta_1} P_{t-1}(\boldtheta),
        \quad \quad \triangledown_{\theta_2}P_t^{-}(\boldtheta) = \theta_1^2 \triangledown_{\theta_2} P_{t-1}(\boldtheta),\\
        & \triangledown_{\boldtheta}P_t(\boldtheta) = 2\triangledown_{\boldtheta} P_t^{-}(\boldtheta) - 2 K_t(\boldtheta) S_t(\boldtheta) \triangledown_{\boldtheta} K_t(\boldtheta) - K_t(\boldtheta)^2 \triangledown_{\boldtheta} S_t(\boldtheta),\\
        & \triangledown_{\theta_1} K_t(\boldtheta) =  \frac{\theta_2 \triangledown_{\theta_1} P_t^{-}(\boldtheta)}{S_t(\boldtheta)} - \frac{\theta_2 P_t^{-}(\boldtheta) \triangledown_{\theta_1} S_t(\boldtheta)}{S_t(\boldtheta)^2},\\
        & \triangledown_{\theta_2} K_t(\boldtheta) =  \dfrac{\theta_2 \triangledown_{\theta_2} P_t^{-}(\boldtheta) + P_t^{-}(\boldtheta)}{S_t(\boldtheta)} - \dfrac{\theta_2 P_t^{-}(\boldtheta) \triangledown_{\theta_2} S_t(\boldtheta)}{S_t(\boldtheta)^2},\\
        & \triangledown_{\theta_1}v_t(\boldtheta) = -\theta_2 \triangledown_{\theta_1} m_t^{-}(\boldtheta),
        \quad \quad \triangledown_{\theta_2}v_t(\boldtheta) = - m_t^{-}(\boldtheta) - \theta_2 \triangledown_{\theta_2} m_t^{-}(\boldtheta),\\
        & \triangledown_{\theta_1} m_t^{-}(\boldtheta) = m_{t-1}(\boldtheta) + \theta_1 \triangledown_{\theta_1} m_{t-1}(\boldtheta),
        \quad \quad \triangledown_{\theta_2} m_t^{-}(\boldtheta) = \theta_1 \triangledown_{\theta_2} m_{t-1}(\boldtheta)\\
        & \triangledown_{\boldtheta} m_{t}(\boldtheta) = \triangledown_{\boldtheta} m_t^{-}(\boldtheta) + v_t(\boldtheta) \triangledown_{\boldtheta} K_t(\boldtheta)  + K_t(\boldtheta) \triangledown_{\boldtheta} v_t(\boldtheta), \\
        & \triangledown_{\boldtheta} P_1^{-}(\boldtheta) = 0, 
        \quad \quad \triangledown_{\boldtheta} m_1^{-}(\boldtheta) = 0.
    \end{aligned}
\end{equation*}

\begin{figure}[!htb]
    \centering
    \begin{subfigure}[b]{0.305\textwidth}
        \centering
        \includegraphics[width=\textwidth]{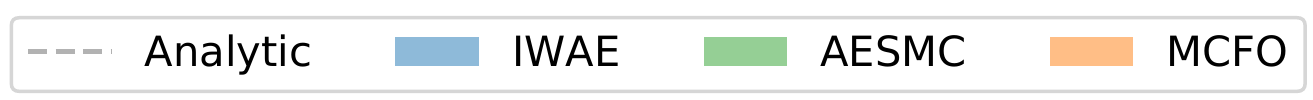}
    \end{subfigure}%
    \vfill
    \begin{subfigure}[b]{0.705\textwidth}
        \centering
        \includegraphics[width=\textwidth]{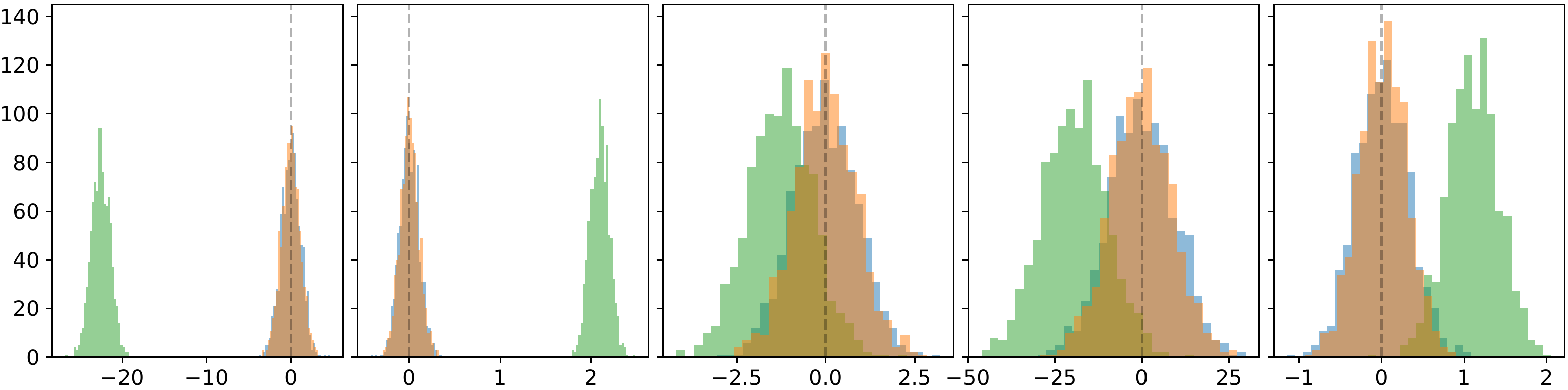}
    \end{subfigure}%
    \hfill
    \begin{subfigure}[b]{0.295\textwidth}
        \centering
        \includegraphics[width=\textwidth]{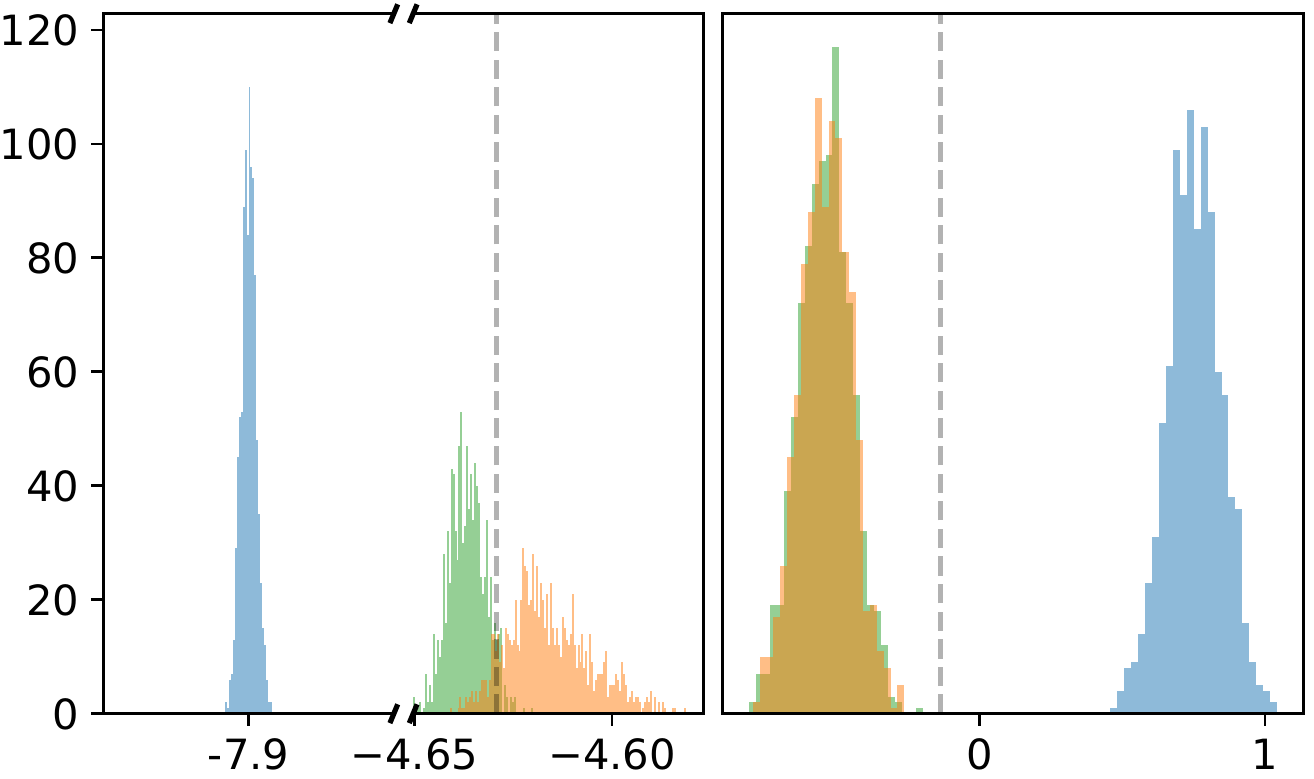}
    \end{subfigure}%
    \vfill
    \begin{subfigure}[b]{0.707\textwidth}
        \centering
        \includegraphics[width=\textwidth]{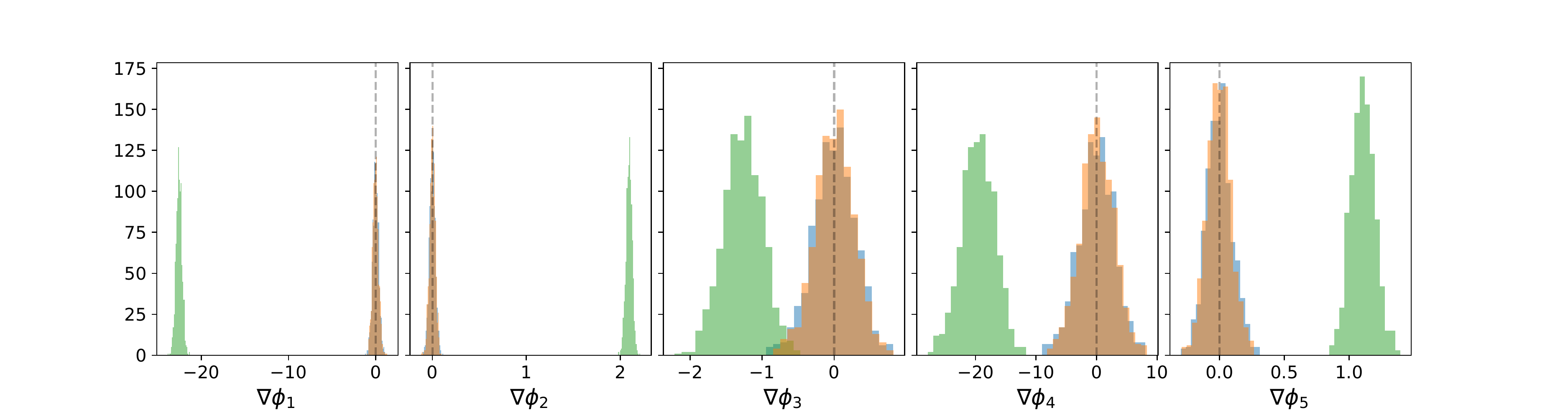}
    \end{subfigure}%
    \hfill
    \begin{subfigure}[b]{0.293\textwidth}
        \centering
        \includegraphics[width=\textwidth]{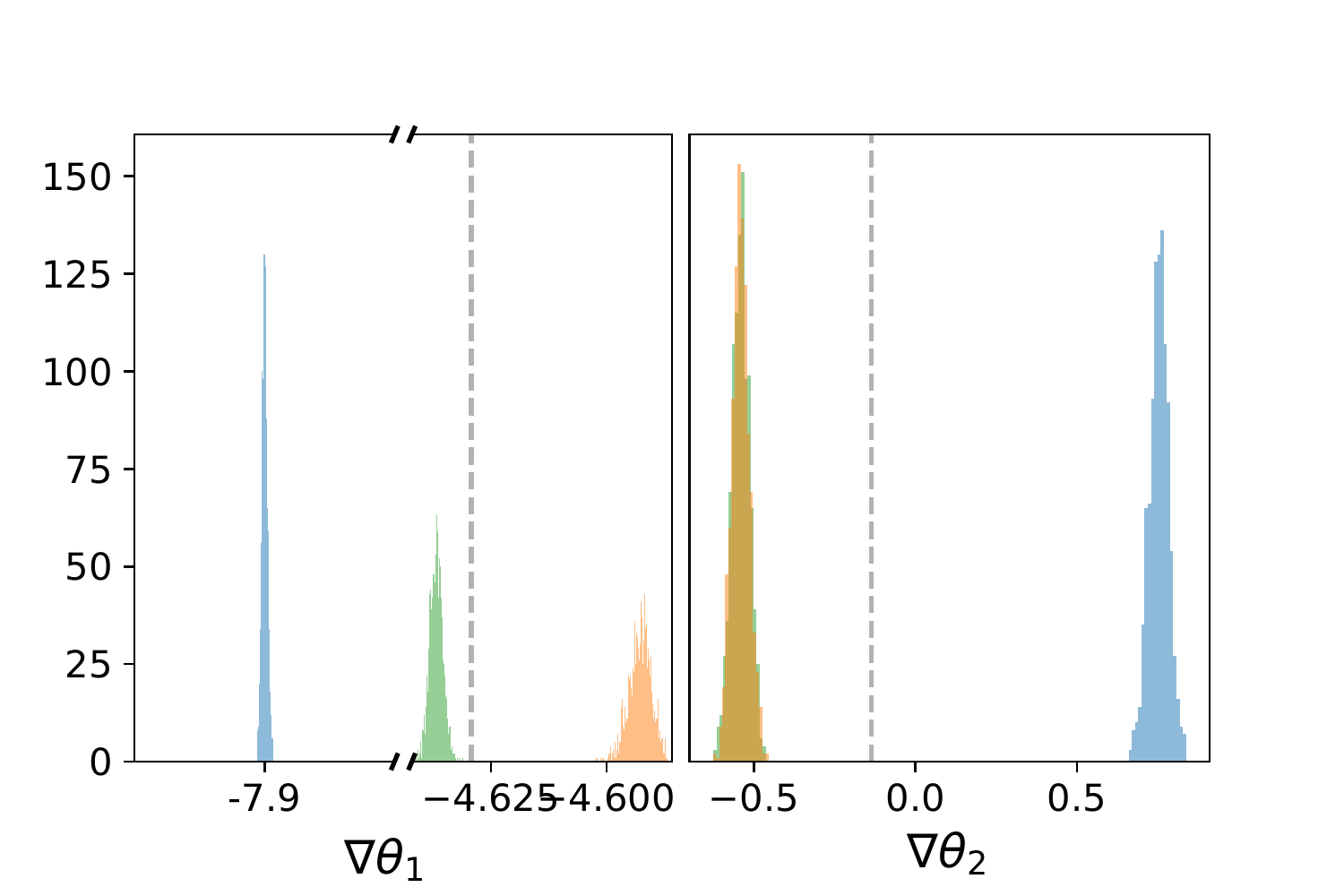}
    \end{subfigure}%
    \vspace{-1mm}
    \caption{Gradient estimates of IWAE, AESMC, MCFO by Automatic Differentiation and analytic gradient of marginal log-likelihood w.r.t. generative and proposal parameters at the same optima as Figure 1; \textit{Top:}K = 10000, \textit{Low:} K = 100000. \vspace{-2mm}}
    \label{fig:gradient_phi_optimum_more}
\end{figure}

\begin{table}[!htb]
    \centering
    \footnotesize
    \begin{tabular}{rc|ccccc}
        Methods & K  & $\triangledown \phi_1$ & $\triangledown \phi_2$ & $\triangledown \phi_3$ &
        $\triangledown \phi_4$ & $\triangledown \phi_5$ \\
        \hline
        MCFO & \multirow{3}{*}{1}  & 0.45 $\pm$ 104.50 & -0.05 $\pm$ 10.01 & -0.97 $\pm$ 82.39 & 2.19 $\pm$ 870.35 & 0.07 $\pm$ 27.33\\
        AESMC &  & -29.34 $\pm$ 108.92 & 2.33 $\pm$ 10.11 & 1.75 $\pm$ 90.75 & -27.07 $\pm$ 904.10 & 1.35 $\pm$ 31.45 \\
        IWAE &  & 1.24 $\pm$ 106.79 & 0.09 $\pm$ 3.29 & -0.35 $\pm$ 89.99 & -12.45 $\pm$ 995.37 & 0.89 $\pm$ 30.07 \\
        \hline
        MCFO & \multirow{3}{*}{10}  & 0.29 $\pm$ 34.84 & -0.03 $\pm$ 3.24 & -0.57 $\pm$ 27.40 & 1.99 $\pm$ 284.15 & -0.05 $\pm$ 9.33\\
        AESMC &  & -22.81 $\pm$ 35.04 & 2.12 $\pm$ 3.25 & -0.92 $\pm$ 27.70 & -22.49 $\pm$ 294.92& 0.88 $\pm$ 9.49 \\
        IWAE &  & -0.96 $\pm$ 35.46 & 0.09 $\pm$ 3.29 & -0.35 $\pm$ 28.85 & -9.80 $\pm$ 284.94 & 0.65 $\pm$ 10.00 \\
        \hline
        MCFO & \multirow{3}{*}{100} & -0.37 $\pm$ 10.77 & 0.04 $\pm$ 1.00 & -0.25$\pm$ 8.52 & -3.21 $\pm$ 86.13 & 0.10 $\pm$ 2.87 \\
        AESMC & & -22.13 $\pm$ 10.59 & 2.05 $\pm$ 0.98 & -1.38 $\pm$ 8.99 &-19.99 $\pm$ 91.75 & 1.10$\pm$ 3.04  \\
        IWAE & & 0.22 $\pm$ 11.11 & -0.02 $\pm$ 1.03 & -0.27 $\pm$ 9.61 & -3.93 $\pm$ 98.78 & 0.08 $\pm$ 3.35\\
        \hline
        MCFO & \multirow{3}{*}{1000} & 0.11 $\pm$ 3.39 & 0.00 $\pm$ 0.31 & -0.04 $\pm$ 2.75& -0.19 $\pm$ 27.93 & 0.02 $\pm$ 0.93\\
        AESMC & & -22.39 $\pm$ 3.42 & 2.08 $\pm$ 0.32 & -1.38 $\pm$ 2.87 & -20.17$\pm$ 29.54 & 1.12 $\pm$ 0.99\\
        IWAE & & 0.13 $\pm$ 3.58 & -0.01 $\pm$ 0.33 & -0.01 $\pm$ 2.89 & -0.13$\pm$ 30.07 & 0.01 $\pm$ 1.01\\
        \hline
        MCFO & \multirow{3}{*}{10000} & -0.04 $\pm$ 1.10 & 0.00 $\pm$ 0.10 & -0.02 $\pm$ 0.85  & -0.31 $\pm$ 8.76 & 0.01 $\pm$ 0.29\\
        AESMC & & -22.55 $\pm$ 1.12 & 2.09 $\pm$ 0.10 & -1.27 $\pm$ 0.91 & -19.48 $\pm$ 9.18 & 1.10$\pm$ 0.31\\
        IWAE & & -0.01 $\pm$ 1.15 & 0.00 $\pm$ 0.11 & -0.03 $\pm$ 0.91 & -0.09 $\pm$ 9.49 & 0.01 $\pm$ 0.32\\
        \hline
        MCFO & \multirow{3}{*}{100000} & -0.01 $\pm$ 0.33 & 0.00 $\pm$ 0.031 & 0.01 $\pm$ 0.28 & 0.09 $\pm$ 2.79 & 0.00 $\pm$ 0.092\\
        AESMC & & -22.56$\pm$ 0.34& 2.09 $\pm$ 0.032 & -1.27 $\pm$ 0.27 & -19.57 $\pm$ 2.73 & 1.11 $\pm$ 0.092\\
        IWAE & & 0.00 $\pm$ 0.34 & 0.00 $\pm$ 0.032 & 0.00 $\pm$ 0.29 & -0.04 $\pm$ 2.97 & 0.00 $\pm$ 0.099
    \end{tabular}
    \begin{tabular}{rc|cc}
        Methods & K & $\triangledown \theta_1$ & $\triangledown \theta_2$\\
        \hline
        MCFO & \multirow{3}{*}{10}
        & -4.64 $\pm$ 0.14 & -0.52 $\pm$ 2.94  \\
        AESMC & & -4.63 $\pm$ 0.14 & -0.45 $\pm$ 2.81\\
        IWAE & & -7.93 $\pm$ 0.059 & 0.51 $\pm$ 2.94\\
        \hline
        MCFO & \multirow{3}{*}{100} & -4.64$\pm$0.046 &-0.51 $\pm$0.88\\
        AESMC & & -4.64 $\pm$ 0.046 & -0.54 $\pm$ 0.92 \\
        IWAE & & -7.93 $\pm$ 0.020& 0.71 $\pm$ 0.99 \\
        \hline
        MCFO & \multirow{3}{*}{1000} &-4.63 $\pm$0.018 & -0.55$\pm$0.28\\
        AESMC & &-4.64 $\pm$ 0.015&-0.55 $\pm$ 0.29\\
        IWAE & & -7.93$\pm$0.0060 &0.75 $\pm$0.30 \\
        \hline
        MCFO & \multirow{3}{*}{10000} & -4.62$\pm$0.0056 &-0.54 $\pm$0.089\\
        AESMC & & -4.64 $\pm$ 0.0046 & -0.54$\pm$0.092 \\
        IWAE & & -7.93$\pm$0.00189 & 0.75 $\pm$0.095 \\
        \hline
        MCFO & \multirow{3}{*}{100000} & -4.59 $\pm$ 0.0024 & -0.54 $\pm$ 0.026\\
        AESMC & & -4.64 $\pm$ 0.0015 & -0.54$\pm$ 0.027 \\
        IWAE & & -7.93$\pm$ 0.00059 & 0.75 $\pm$ 0.030\\
    \end{tabular}
    \caption{The mean and standard deviation of gradient estimates with respect to parameters $\boldphi$ and $\boldtheta$ reported in Figure 1 and \ref{fig:gradient_phi_optimum_more}.}\vspace{-4mm}
    \label{tab:gradient_stats_phi}
\end{table}

\begin{figure}[!htb]
    %\vspace{-4mm}
    \centering
    \begin{subfigure}[b]{0.704\textwidth}
        \centering
        \includegraphics[width=\textwidth]{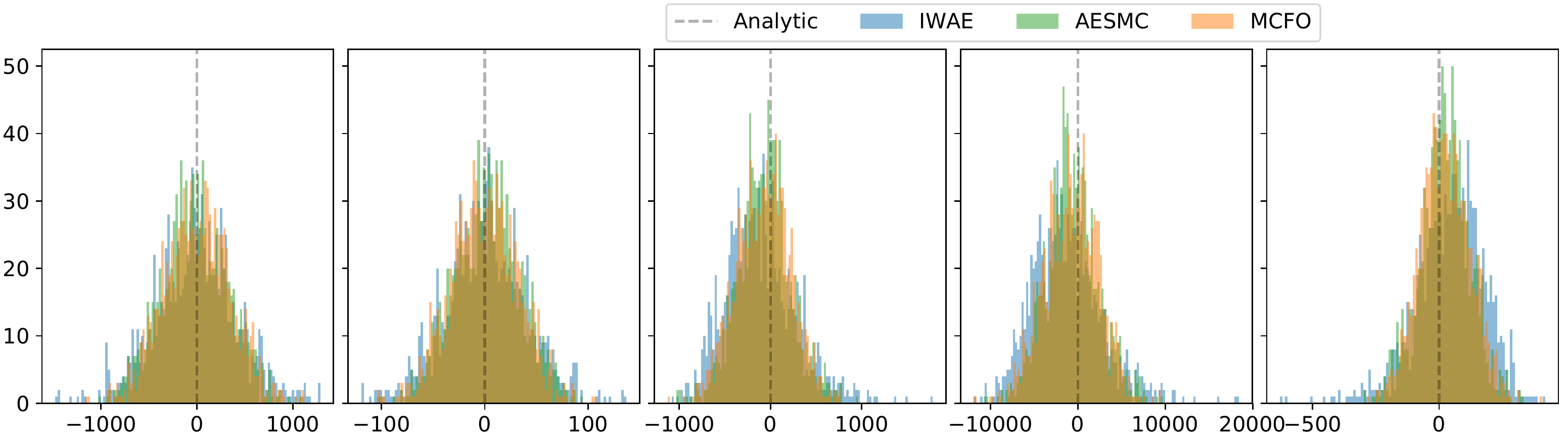}
    \end{subfigure}%
    \hfill
    \begin{subfigure}[b]{0.296\textwidth}
        \centering
        \includegraphics[width=\textwidth]{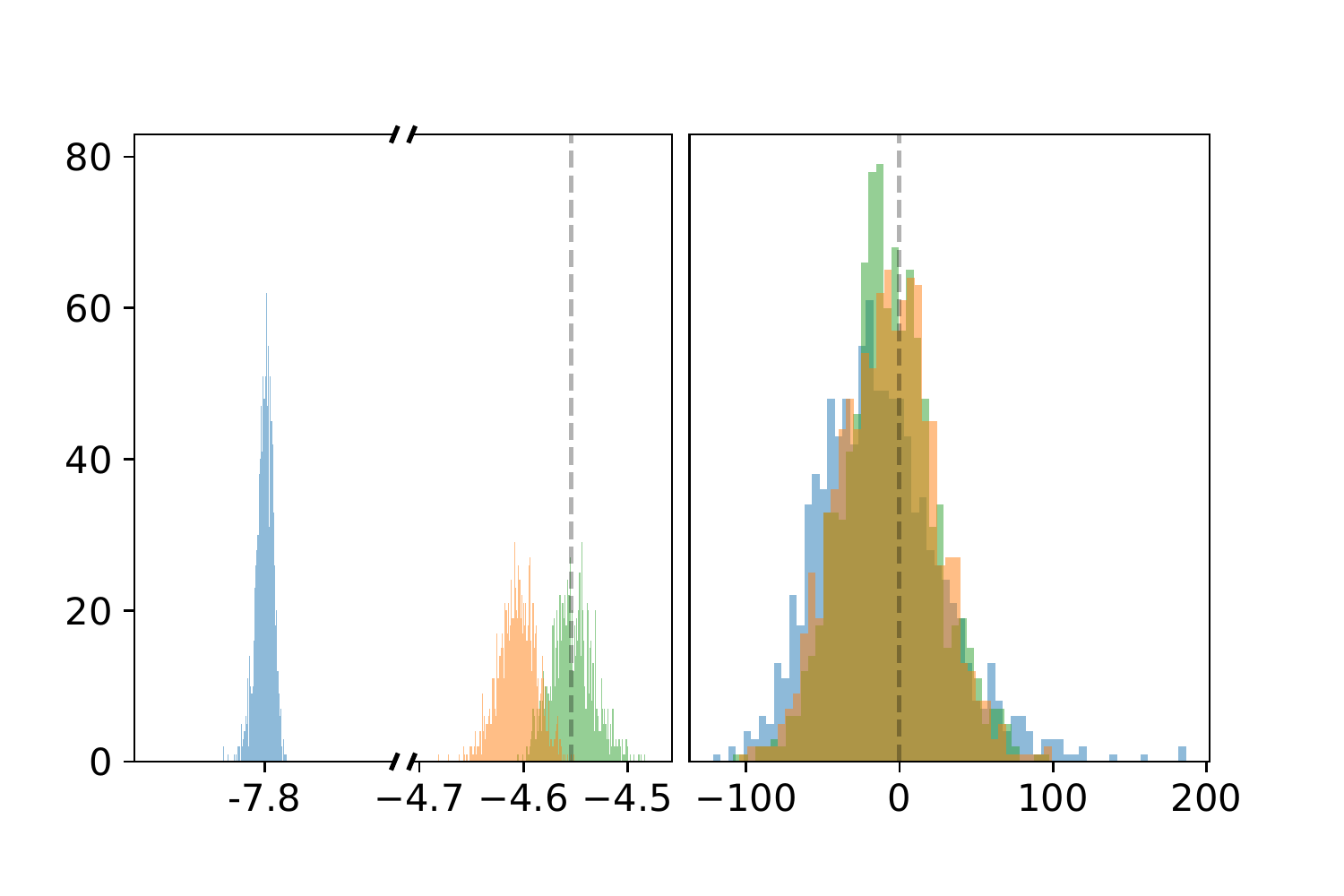}
    \end{subfigure}%
    \vfill
    \begin{subfigure}[b]{0.707\textwidth}
        \centering
        \includegraphics[width=\textwidth]{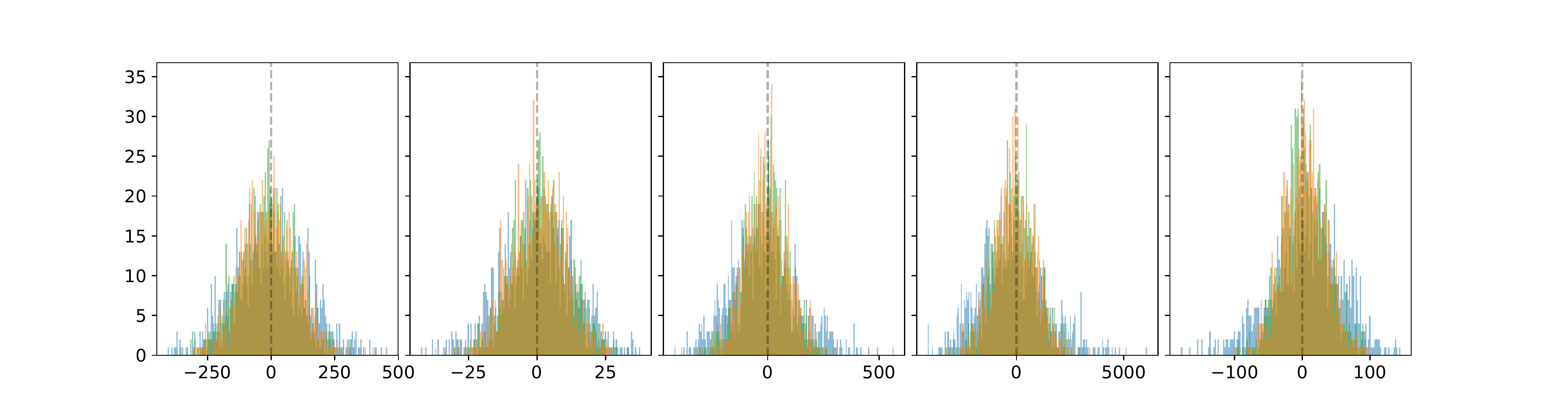}
    \end{subfigure}%
    \hfill
    \begin{subfigure}[b]{0.293\textwidth}
        \centering
        \includegraphics[width=\textwidth]{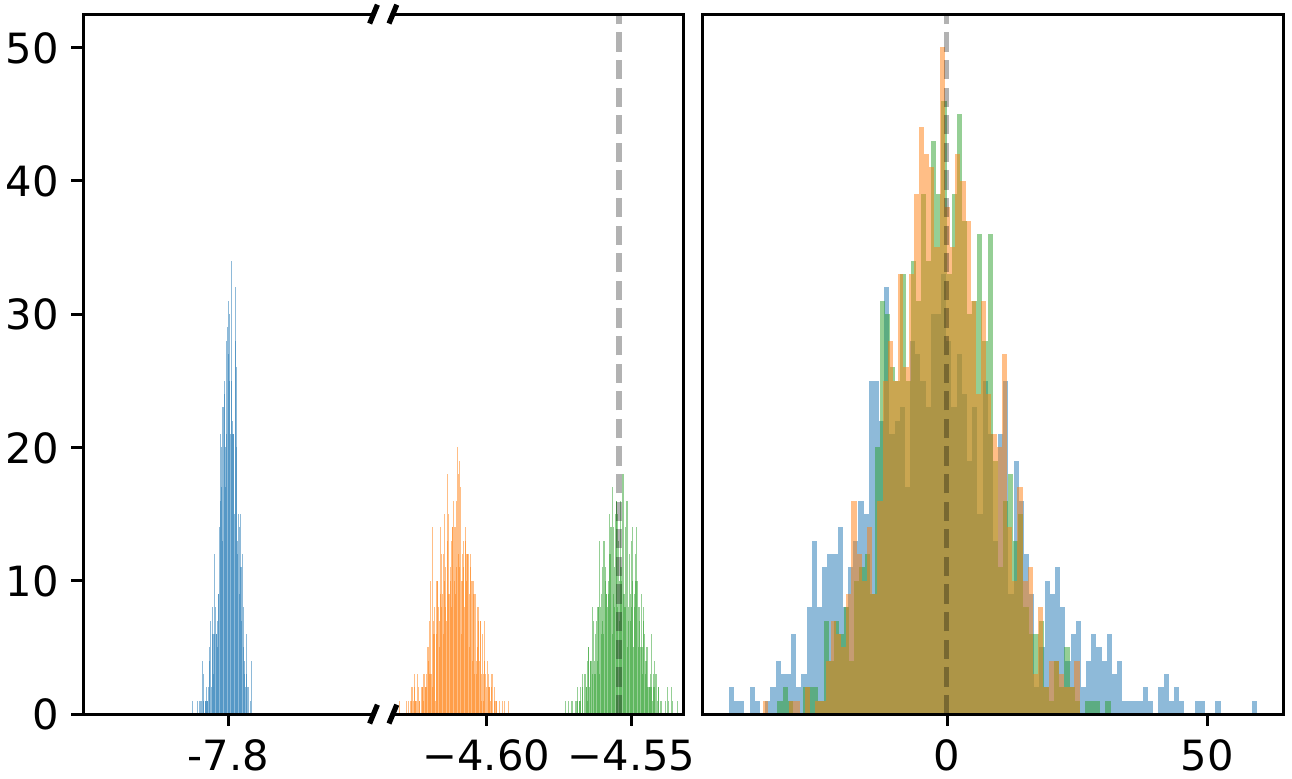}
    \end{subfigure}%
    \vfill
    \begin{subfigure}[b]{0.706\textwidth}
        \centering
        \includegraphics[width=\textwidth]{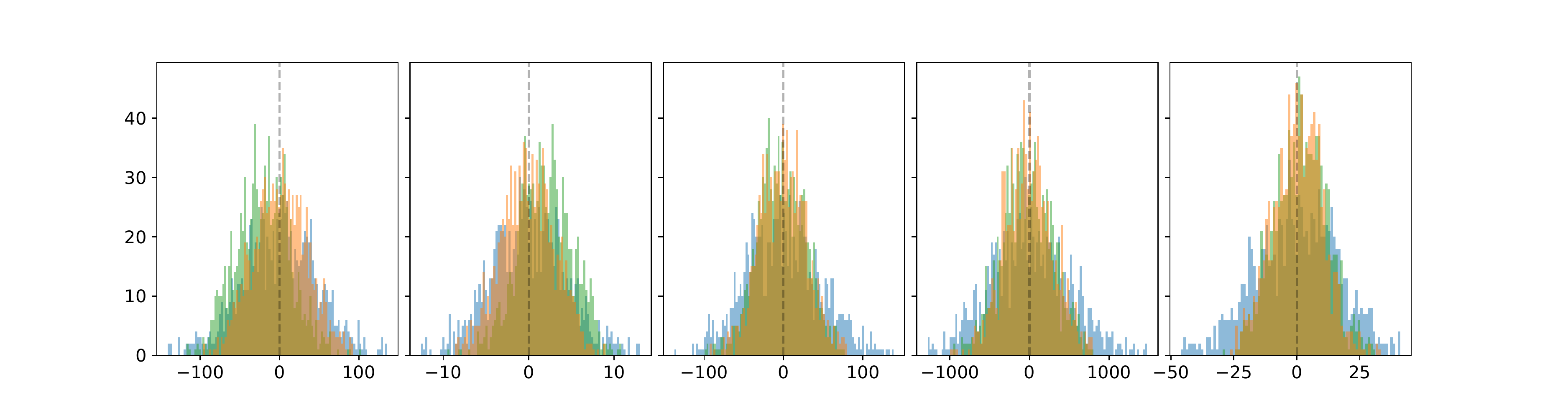}
    \end{subfigure}%
    \hfill
    \begin{subfigure}[b]{0.294\textwidth}
        \centering
        \includegraphics[width=\textwidth]{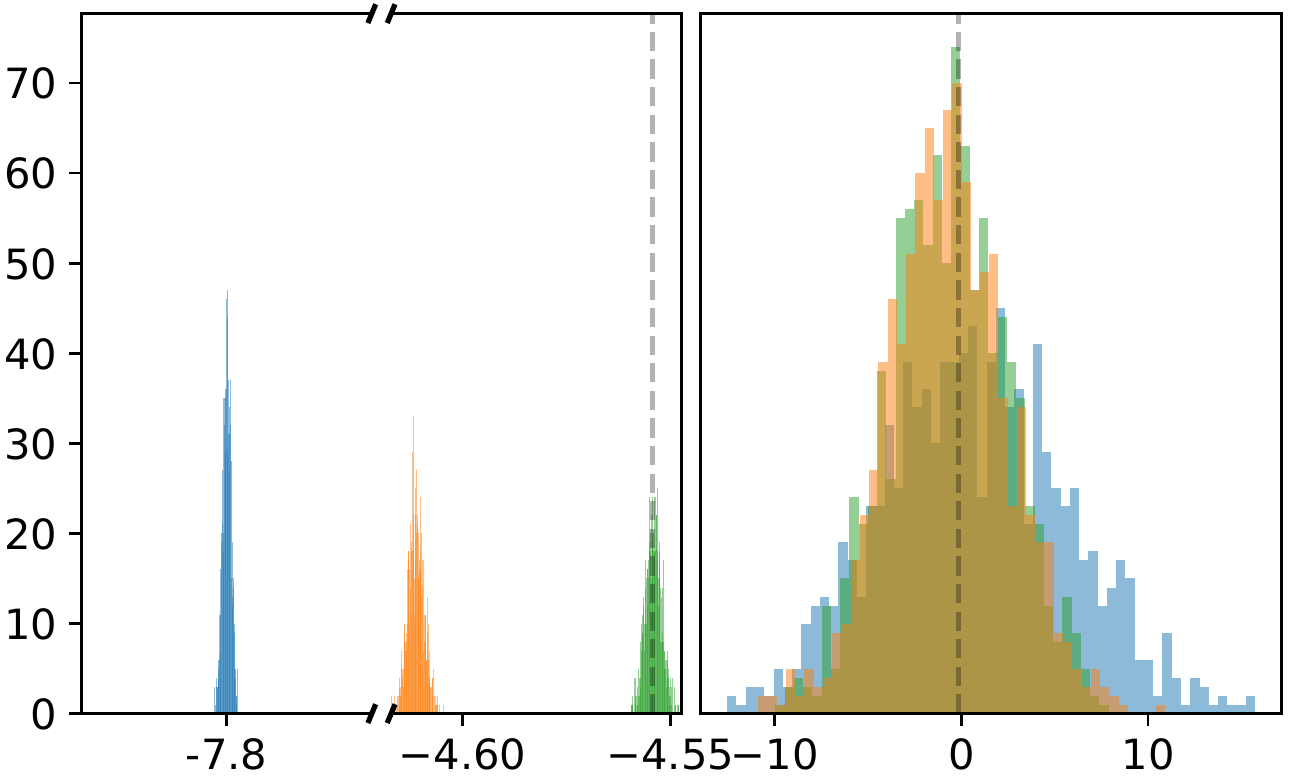}
    \end{subfigure}%
    \vfill
    \begin{subfigure}[b]{0.706\textwidth}
        \centering
        \includegraphics[width=\textwidth]{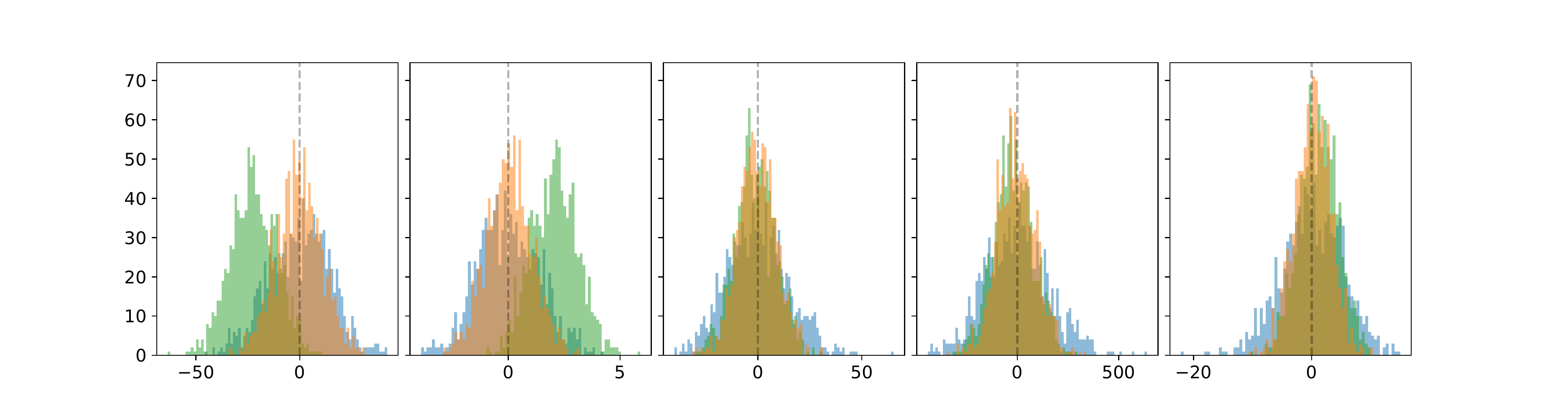}
    \end{subfigure}%
    \hfill
    \begin{subfigure}[b]{0.294\textwidth}
        \centering
        \includegraphics[width=\textwidth]{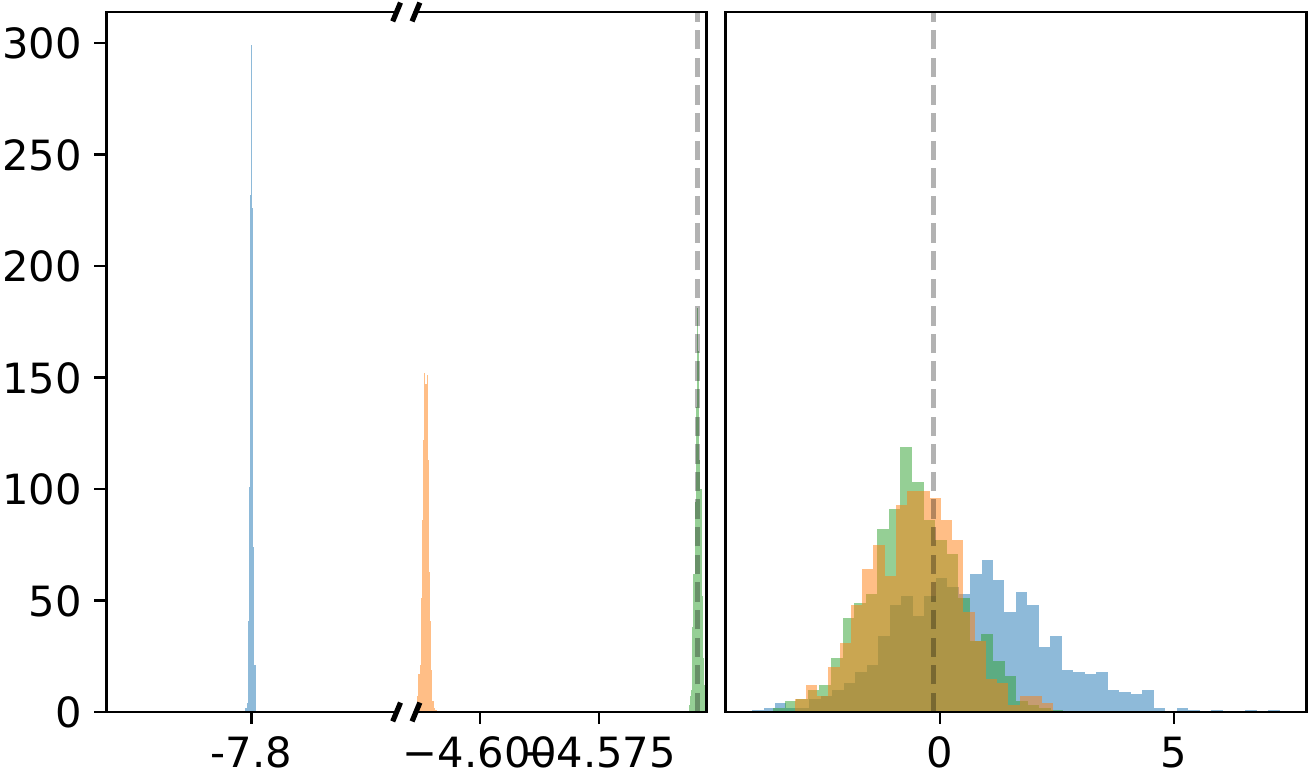}
    \end{subfigure}%
    \vfill
    \begin{subfigure}[b]{0.705\textwidth}
        \centering
        \includegraphics[width=\textwidth]{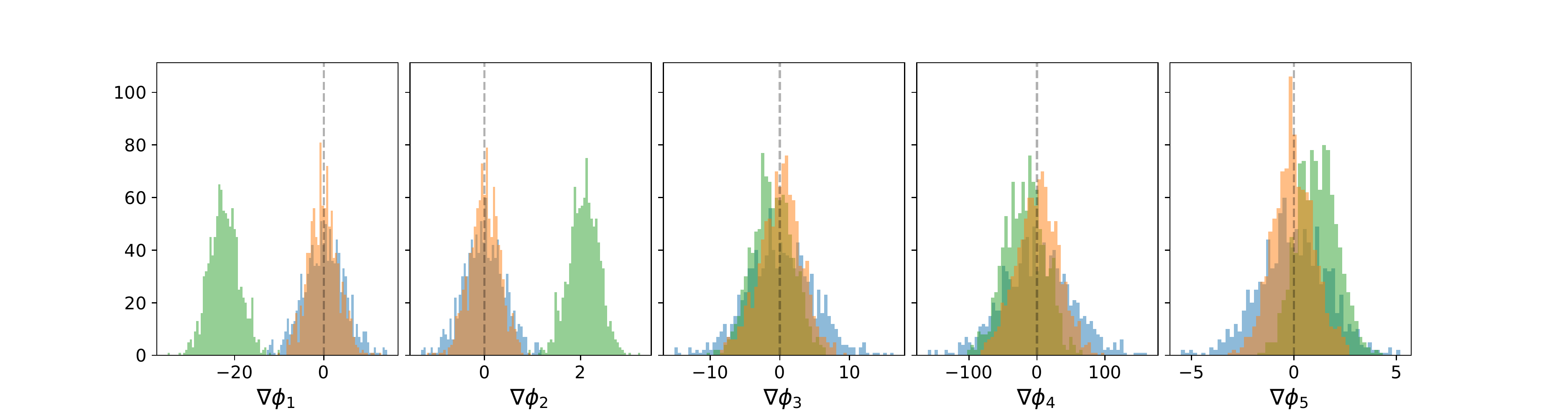}
    \end{subfigure}%
    \hfill
    \begin{subfigure}[b]{0.295\textwidth}
        \centering
        \includegraphics[width=\textwidth]{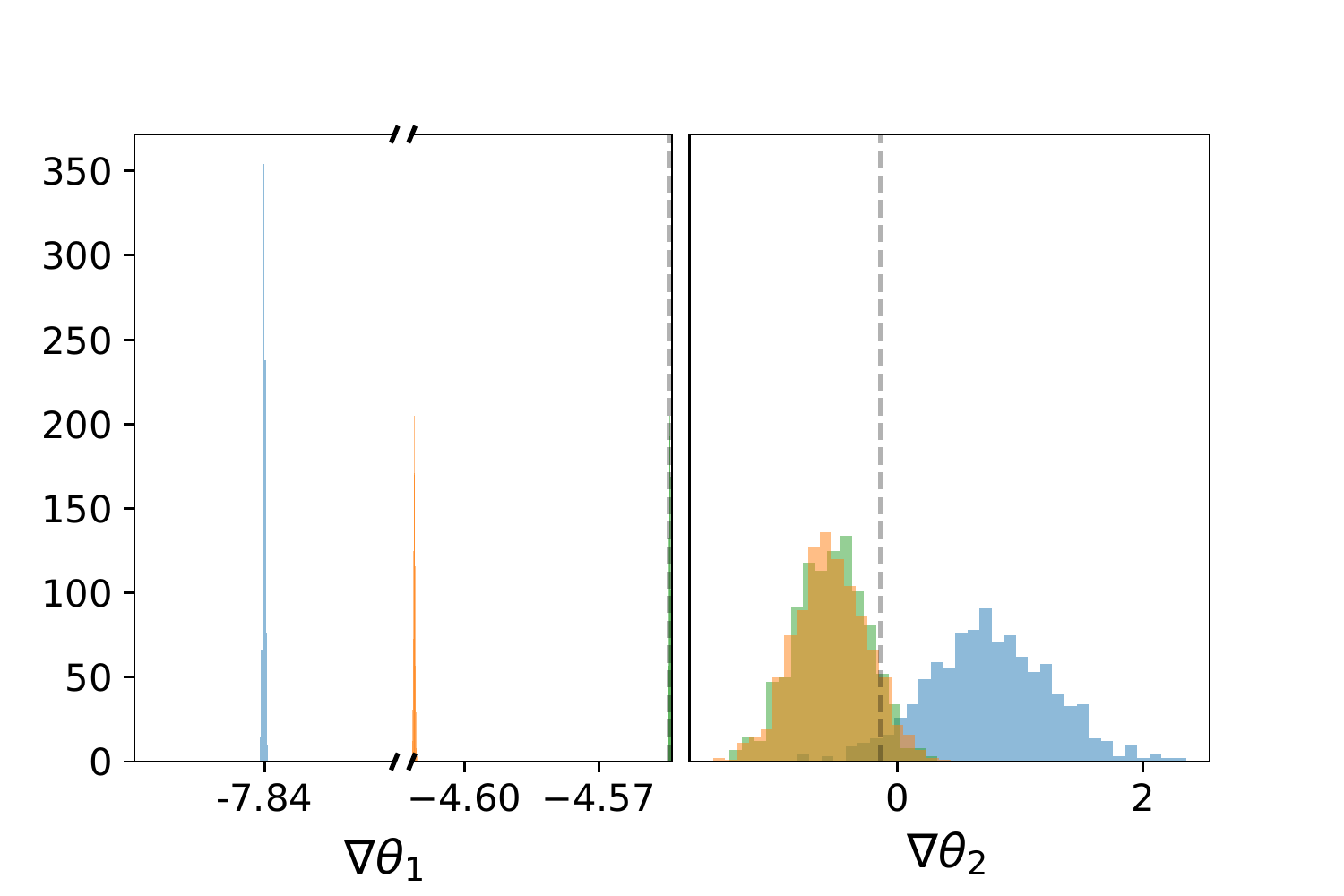}
    \end{subfigure}%
    \vspace{-2mm}
    \caption{Gradient estimates of IWAE, AESMC, MCFO with respect to generative and proposal parameters at their optima with smaller variance of observation model, $\Sigma_R$, under different $K = 10, 100, 1000, 10000, 100000$ from top to bottom respectively. }
    \label{fig:gradient_phi_optimum_2}\vspace{-2mm}
\end{figure}\vspace{-2mm}
\vspace{-2mm}
\begin{figure}[!htb]
    \vspace{-2mm}
    \centering
    \begin{subfigure}[b]{0.305\textwidth}
        \centering
        \includegraphics[width=\textwidth]{figures/gradient_estimate/nonopt_theta_opt_phi/0522_082927/legend.pdf}
    \end{subfigure}%
    \vfill
    \begin{subfigure}[u]{0.305\textwidth}
        \centering
        \includegraphics[width=\textwidth]{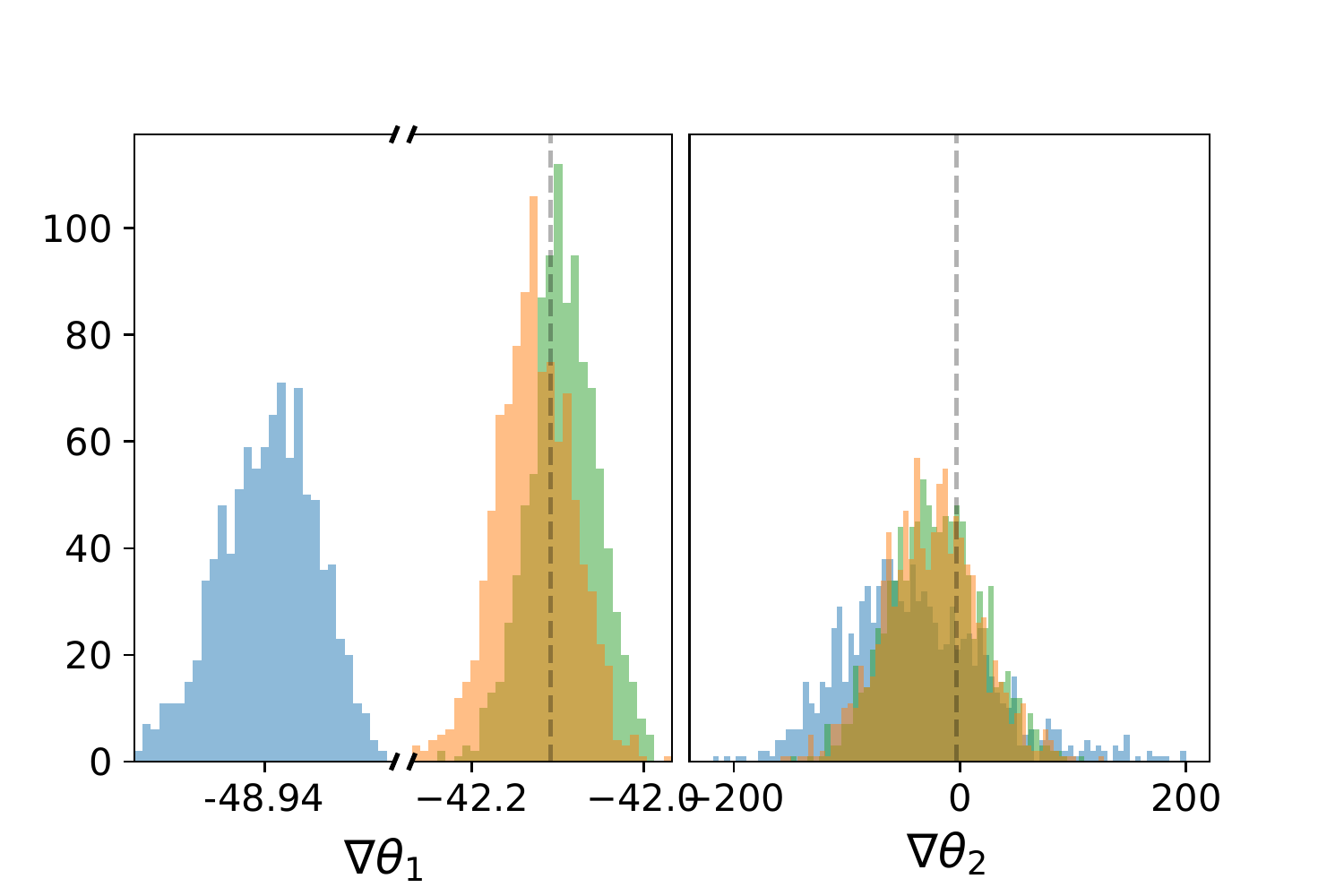}
        \vspace{-7mm}
        \caption{K=10}
    \end{subfigure}%
    \medskip
    \begin{subfigure}[u]{0.30\textwidth}
        \centering
        \includegraphics[width=\textwidth]{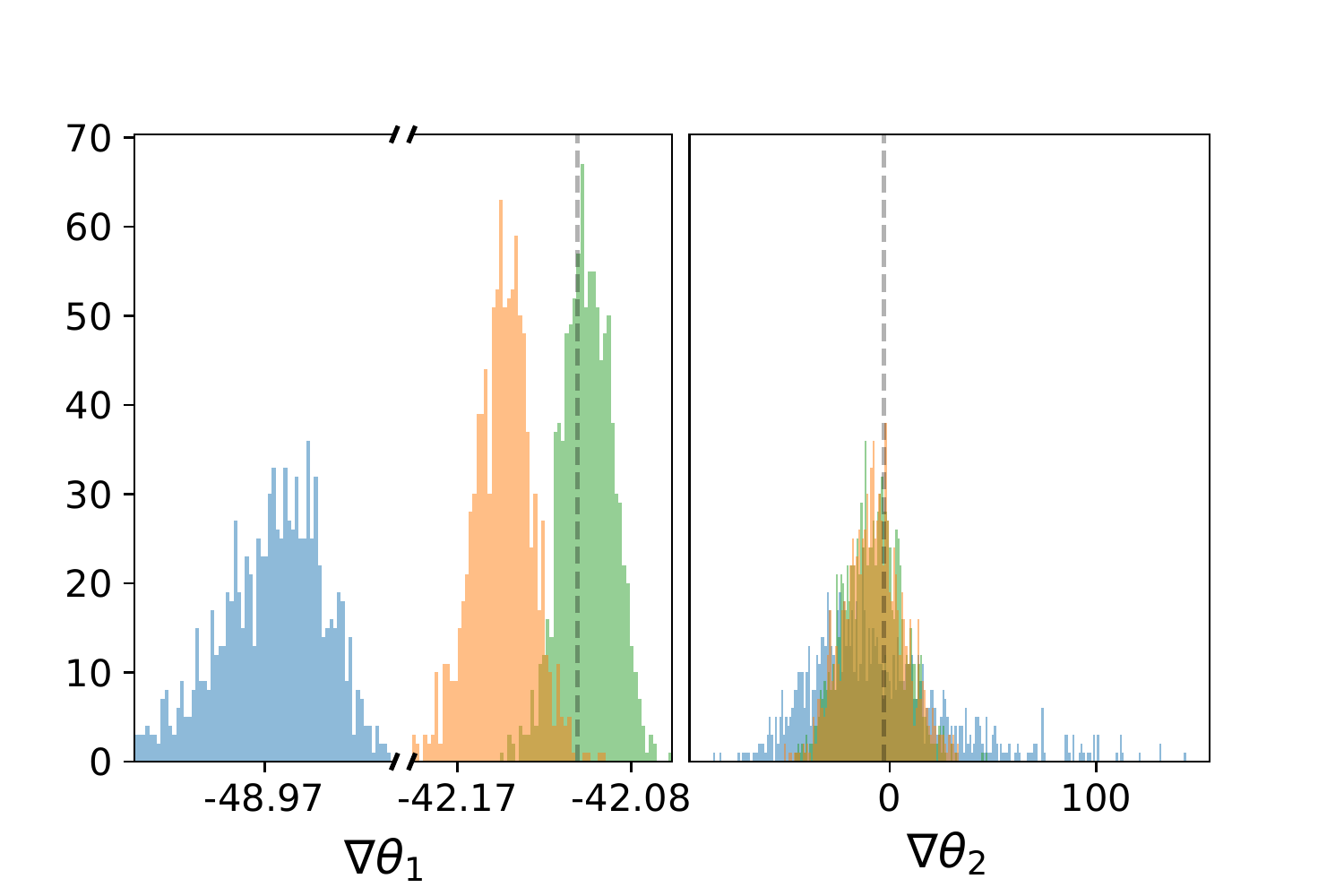}
        \vspace{-7mm}
        \caption{K=100}
    \end{subfigure}%
    \medskip
    \begin{subfigure}[u]{0.30\textwidth}
        \centering
        \includegraphics[width=\textwidth]{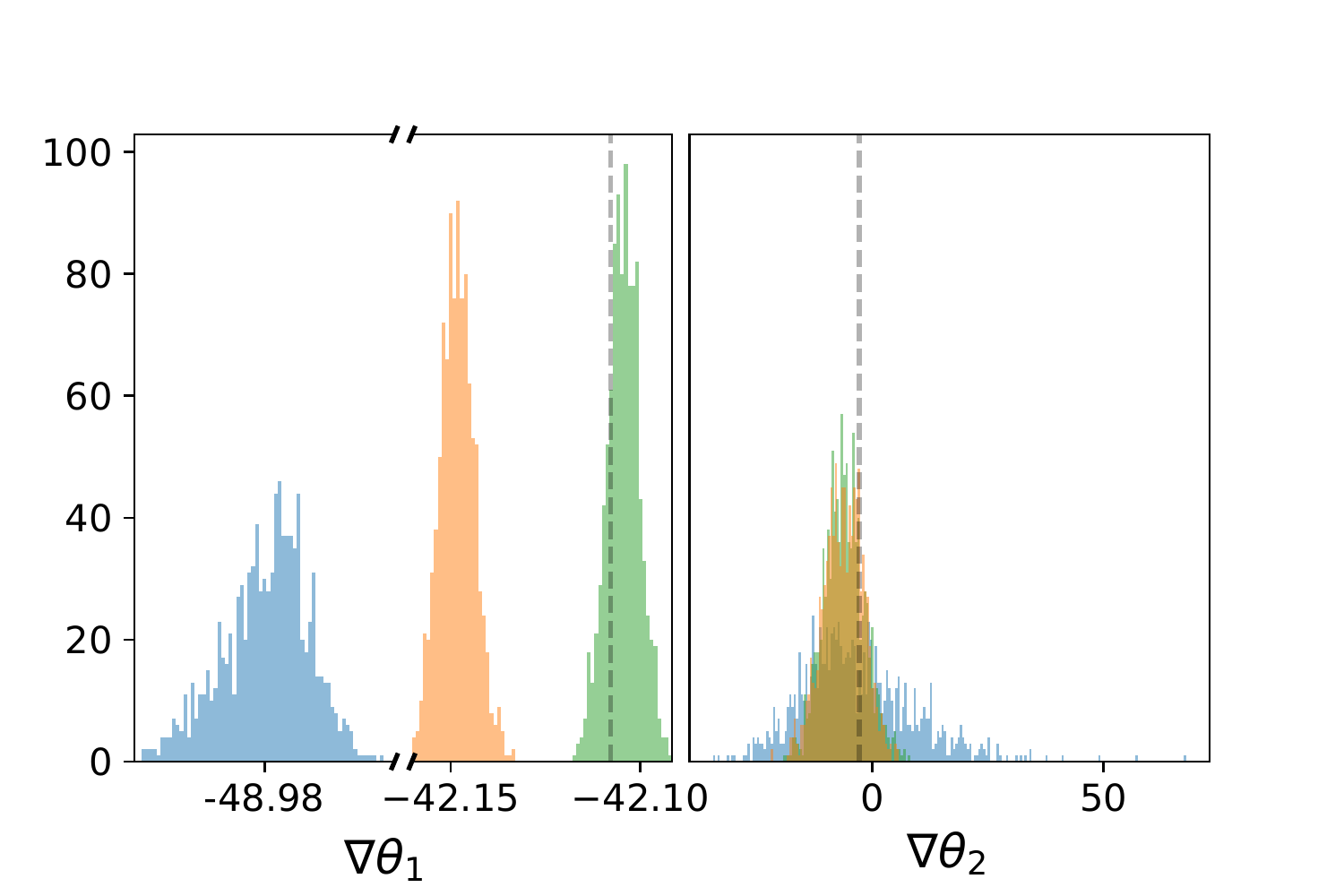}
        \vspace{-7mm}
        \caption{K=1000}
    \end{subfigure}%
    \vspace{-5mm}
    \vfill
    \begin{subfigure}[u]{0.315\textwidth}
        \centering
        \includegraphics[width=\textwidth]{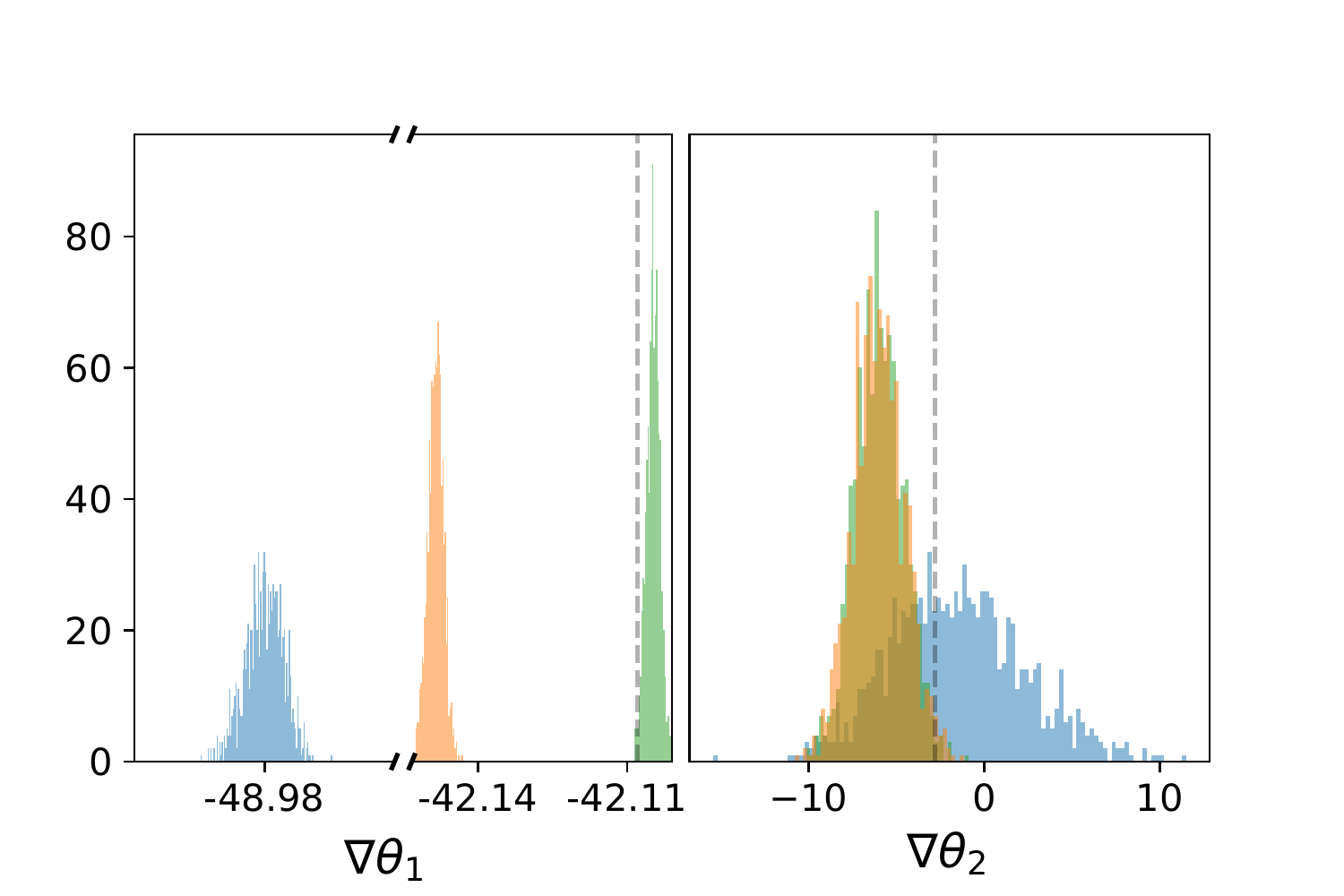}
        \vspace{-7mm}
        \caption{K=10000}
    \end{subfigure}%
    \medskip
    \centering
    \begin{subfigure}[u]{0.335\textwidth}
        \centering
        \includegraphics[width=\textwidth]{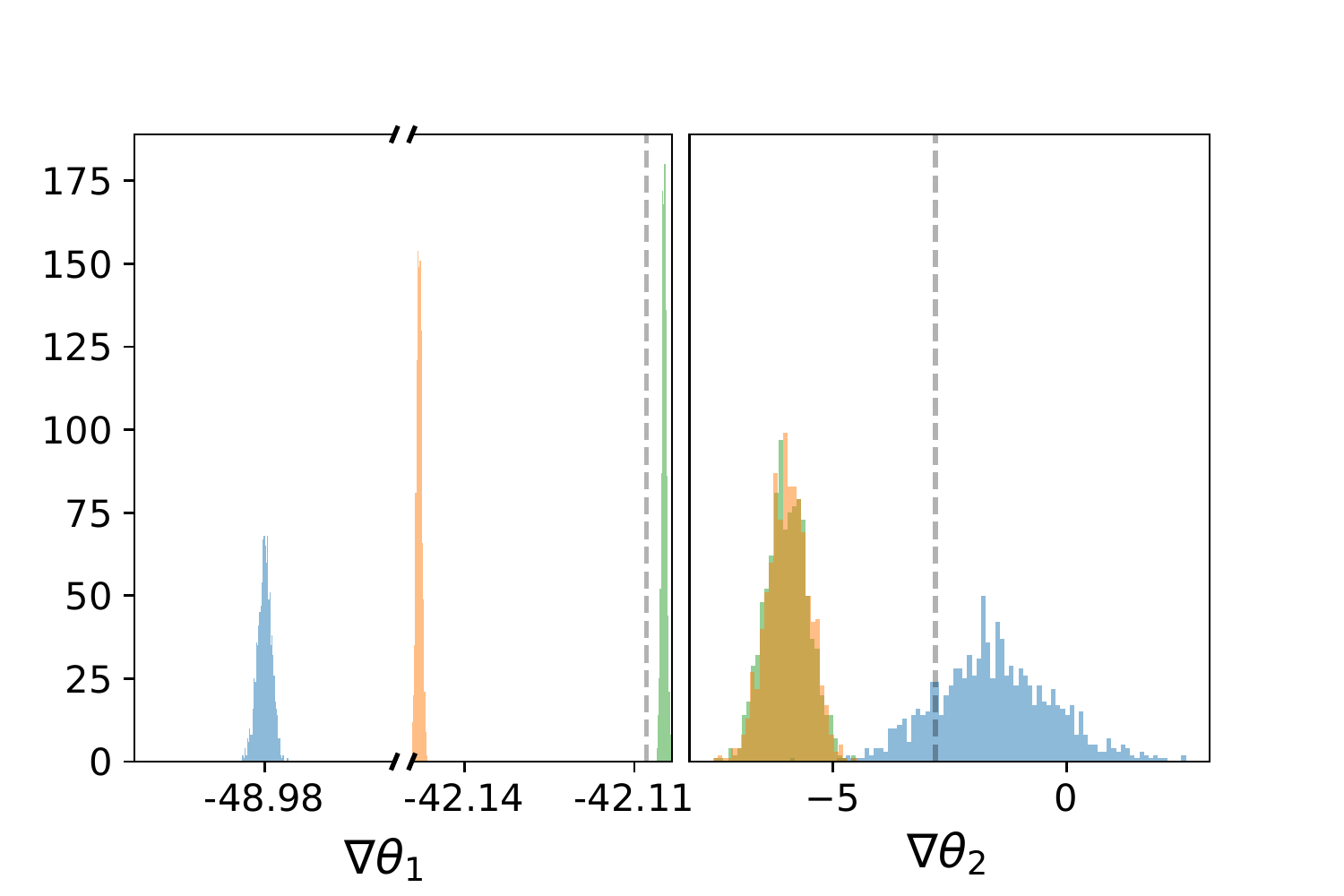}
        \vspace{-7mm}
        \caption{K=100000}
    \end{subfigure}%
    \vspace{-2mm}
    \caption{Gradient estimates of IWAE, AESMC, MCFO with respect to generative parameters, $\boldtheta$ at non-optima with different numbers of samples $K$ as 10, 100, 1000, 10000, 100000 respectively. }
    \label{fig:gradient_theta_nonoptimum}\vspace{-3mm}
\end{figure}

\subsection{Experiment on LGSSM}

\subsubsection{Gradient Estimate on LGSSM}
\label{sec:appendix_gradient_estimate_ex}
Figure \ref{fig:gradient_phi_optimum_more} shows gradient estimates with the numbers of samples $K=\{10000, 100000 \}$ at the same optima as Figure 1 when LGSSM parameters are set as: $\theta_1=0.9$, $\theta_2=10$, $\mu_0 = 0.5$, $\sigma_0 = 1.0$, $\Sigma_Q = 1.0 $, $\Sigma_R = 1.0$. The bias in the estimates of $\triangledown \boldphi$ by AESMC is distinguished and does not reduce with increasing $K$, while that by IWAE and MCFO is close to the true gradient. The variance of estimates on both $\boldtheta$ and $\boldphi$ decreases at the order of $\mathcal{O}(1/K)$ (the standard deviation decreases at $\mathcal{O}(1/\sqrt{K})$ in Table \ref{tab:gradient_stats_phi}) while the bias reduces substantially only when $K$ is small i.e. from 10 to 100. 

It is found that the bias and variance reduction of gradient estimates are not uniform in the parameter space. For instance, Figure \ref{fig:gradient_phi_optimum_2} shows the gradient estimates when observation variance $\Sigma_R$ is set as 0.01, which is lower than Figure 1 and Figure \ref{fig:gradient_phi_optimum_more}. The variance of gradient estimates in Figure \ref{fig:gradient_phi_optimum_2} have higher variance on $\triangledown \boldphi$ and $\triangledown \theta_2$ but lower on $\triangledown \theta_1$ for all three methods. Although the bias on $\triangledown \boldphi$ is less distinguishable than in Figure 1, it cannot be removed with increasing the number of samples.

Except for the gradient estimates at the parameters' optima, Figure \ref{fig:gradient_theta_nonoptimum} illustrates the estimates of gradient $\triangledown \boldtheta$ at non-optima, which of the generative parameters $\boldtheta$ are set to $\theta_1 = 0.6, \theta_2 = 8$. To compare to the analytic gradients, the local optima of proposal parameters are used for all surrogate objectives. The estimates by IWAE have higher variance and bias in $\theta_1$, but lower bias in $\theta_2$. Estimates by AESMC and MCFO have similar bias and variance on $\theta_2$ and $\theta_1$.

\begin{figure*}[!htb]
    \centering
    \begin{subfigure}[b]{0.33\textwidth}
        \centering
        \includegraphics[width=\textwidth]{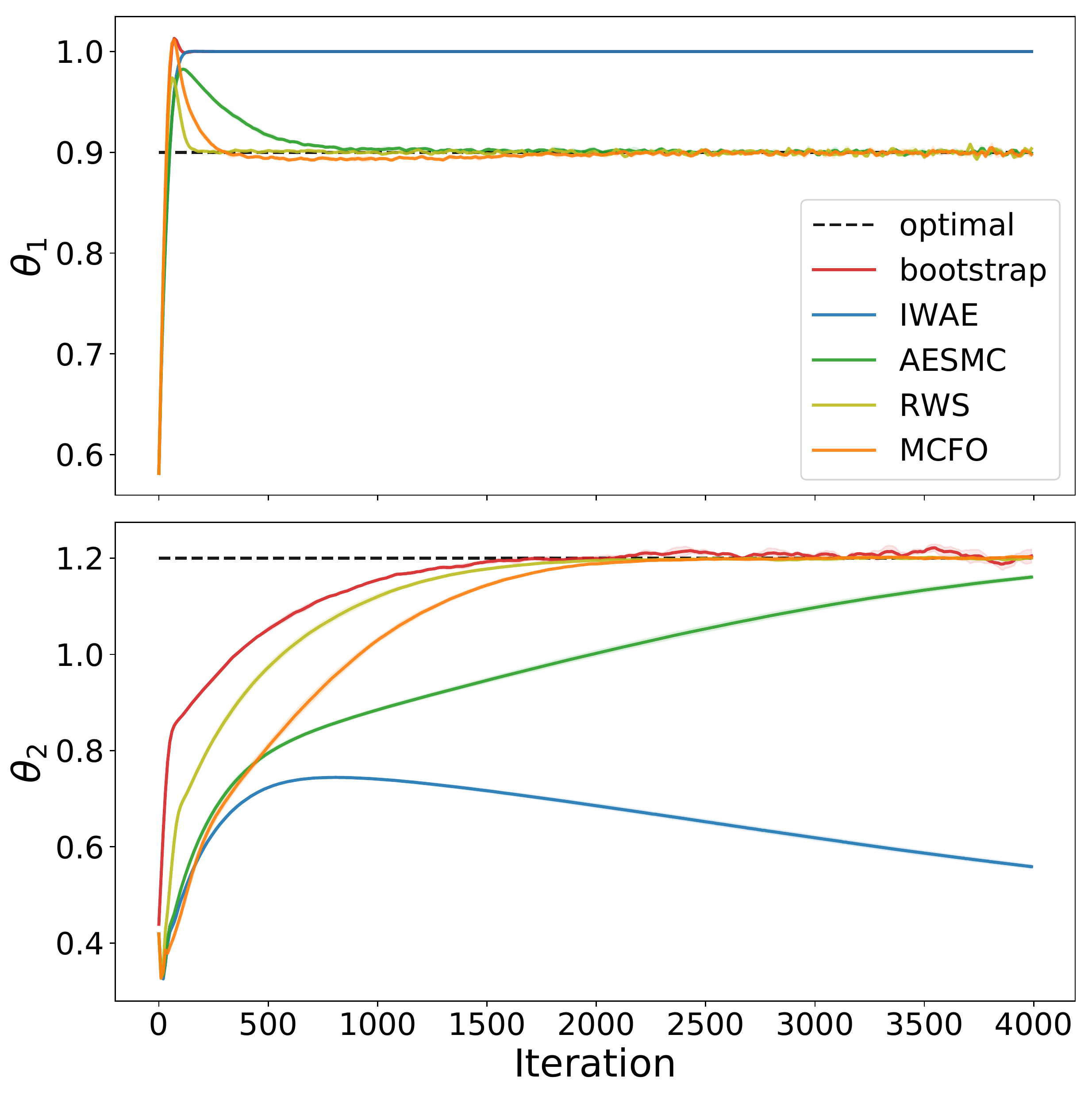}
    \end{subfigure}%
    \hfill
    \begin{subfigure}[b]{0.33\textwidth}
        \centering
        \includegraphics[width=\textwidth]{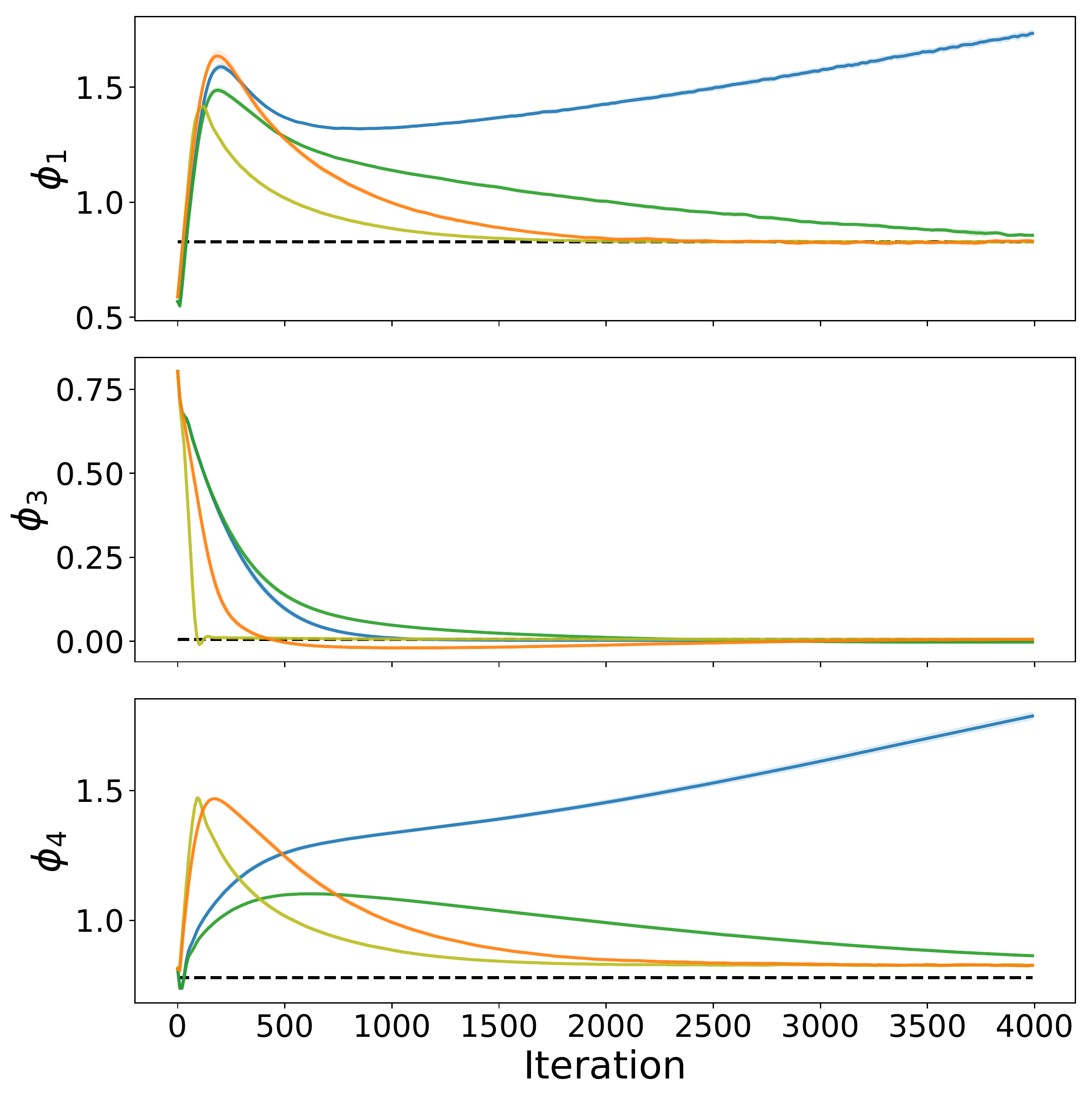}
    \end{subfigure}%
    \begin{subfigure}[b]{0.33\textwidth}
        \centering
        \includegraphics[width=\textwidth]{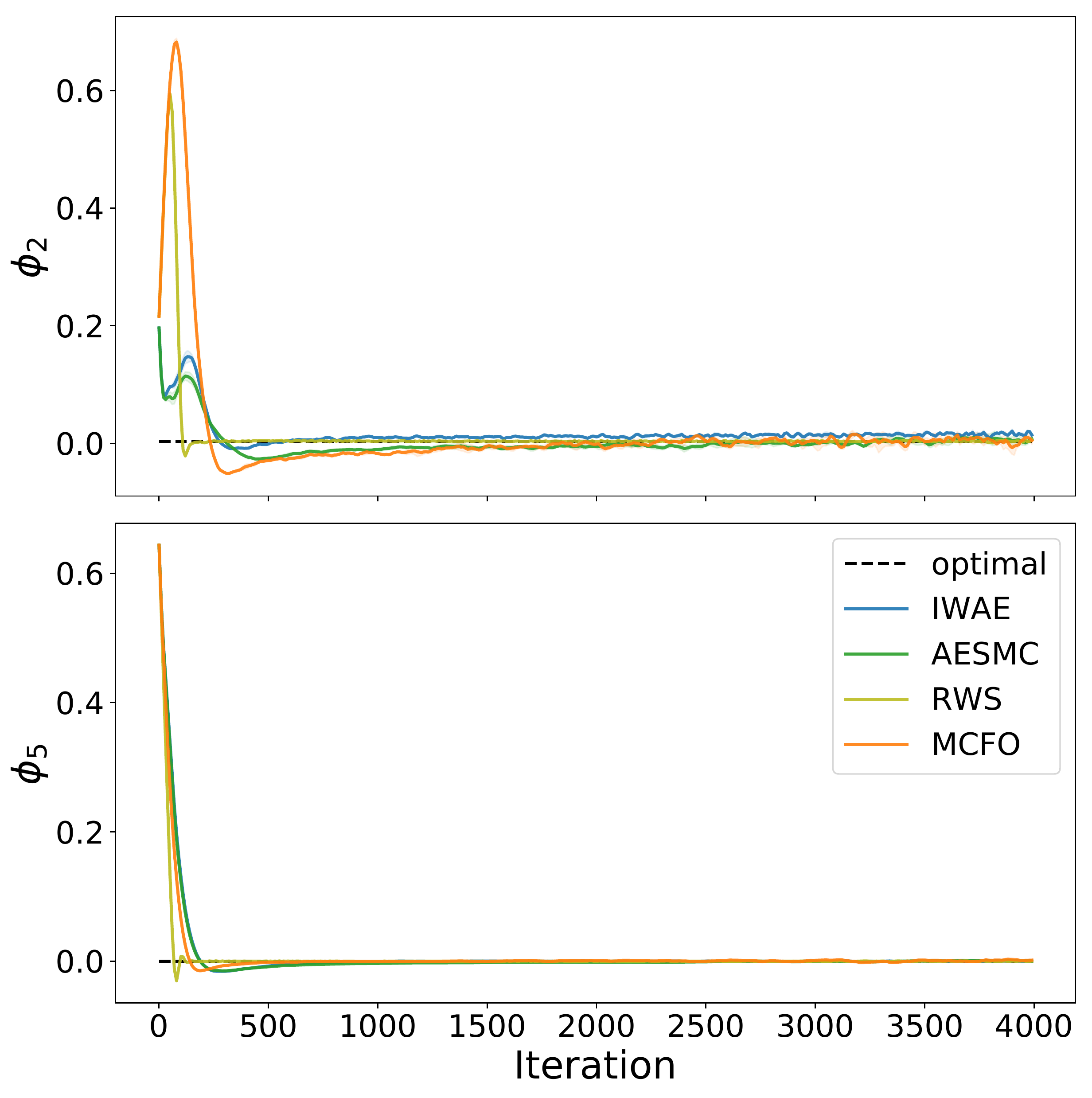}
    \end{subfigure}%
    \caption{Parameters of generative and proposal, $\boldtheta$ and $\boldphi$ during the training as Figure 2.}\vspace{-3mm}
    \label{fig:LGSSM_proposal_bias}
\end{figure*}

\begin{figure}[!htb]
    %\vspace{-4mm}
    \centering
    \begin{subfigure}[b]{0.35\textwidth}
        \centering
        \includegraphics[width=\textwidth]{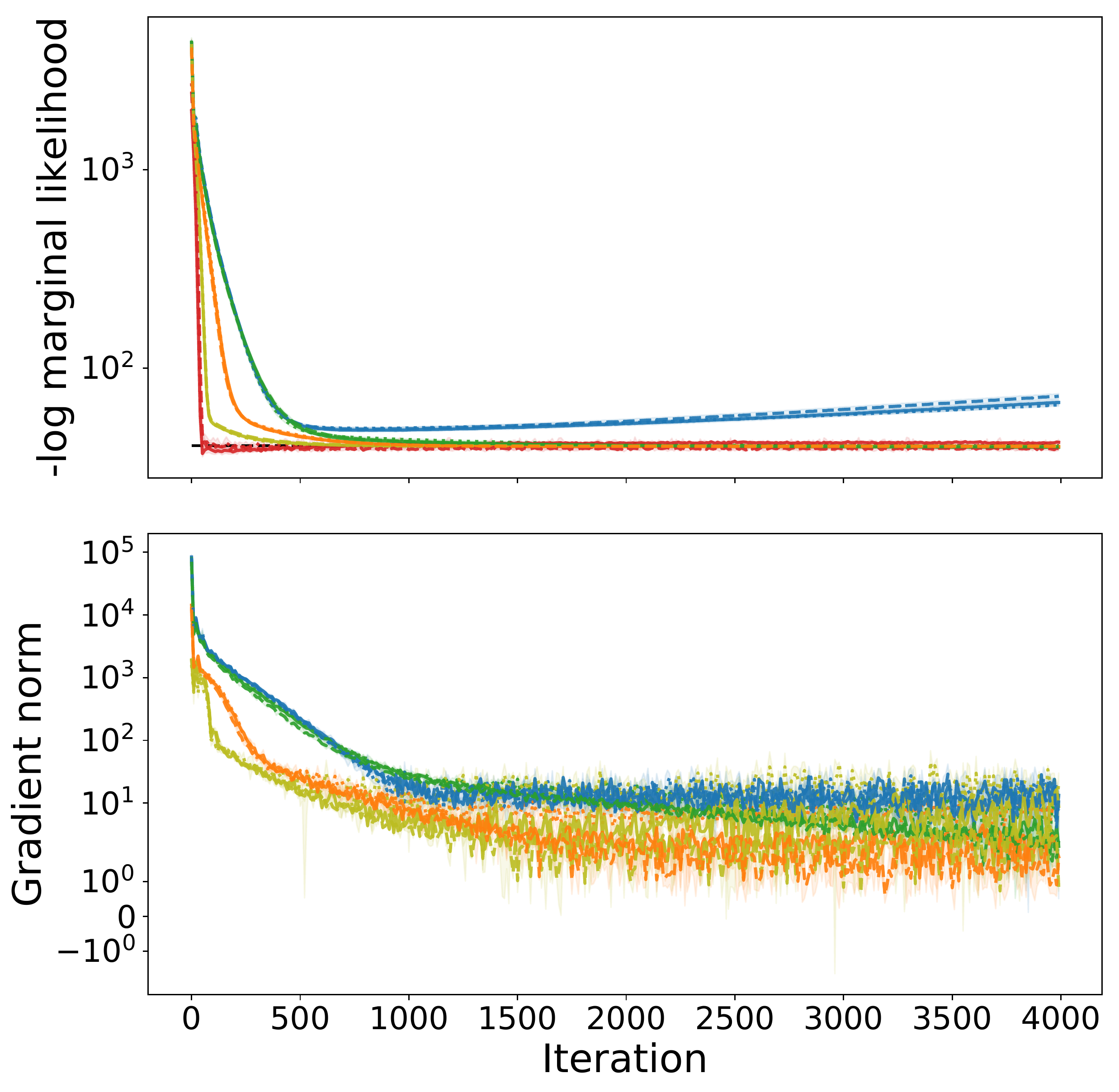}
    \end{subfigure}%
    \medskip
    \begin{subfigure}[b]{0.35\textwidth}
        \centering
        \includegraphics[width=\textwidth]{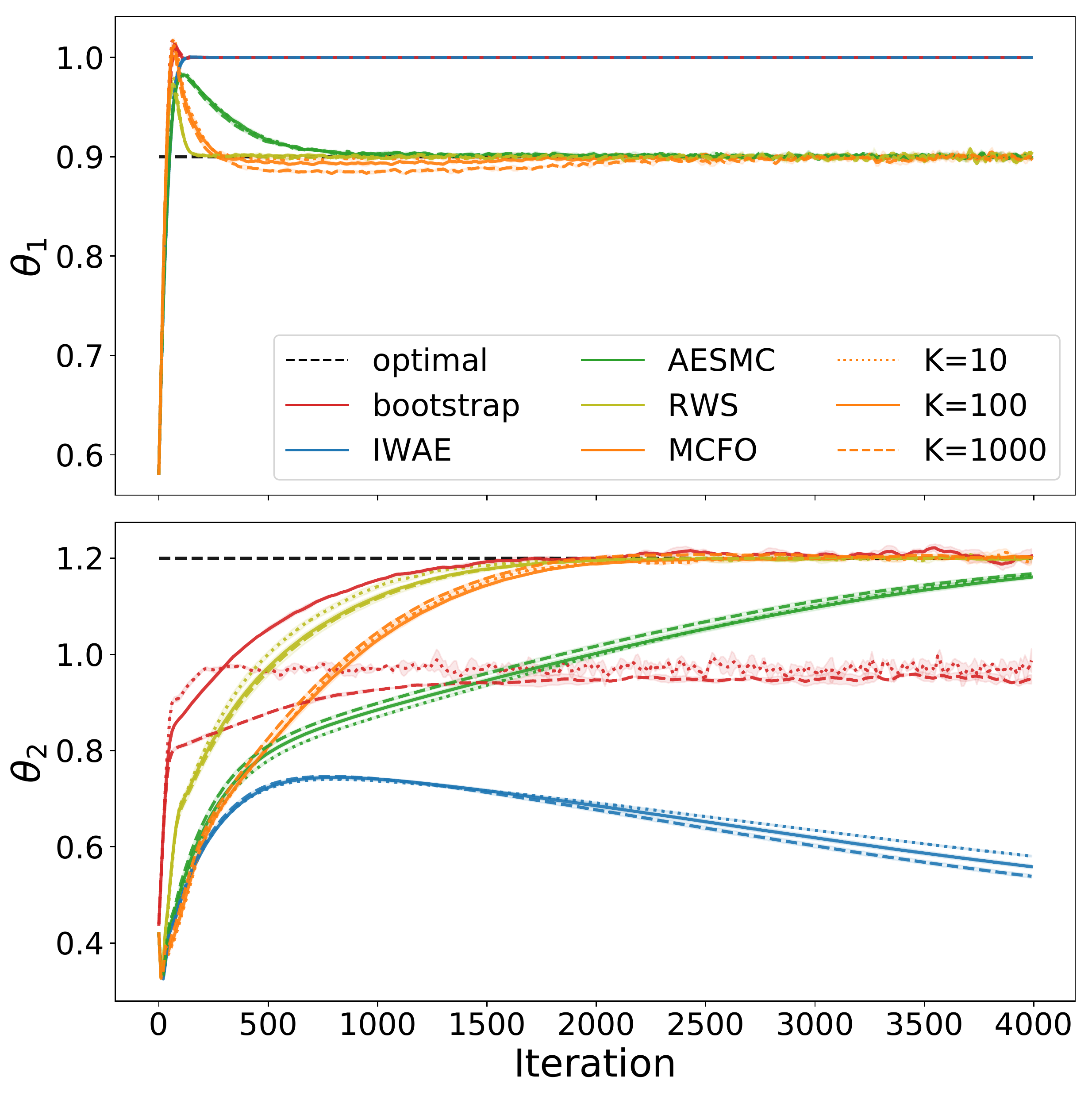}
    \end{subfigure}%
    \vfill
    \begin{subfigure}[b]{0.35\textwidth}
        \centering
        \includegraphics[width=\textwidth]{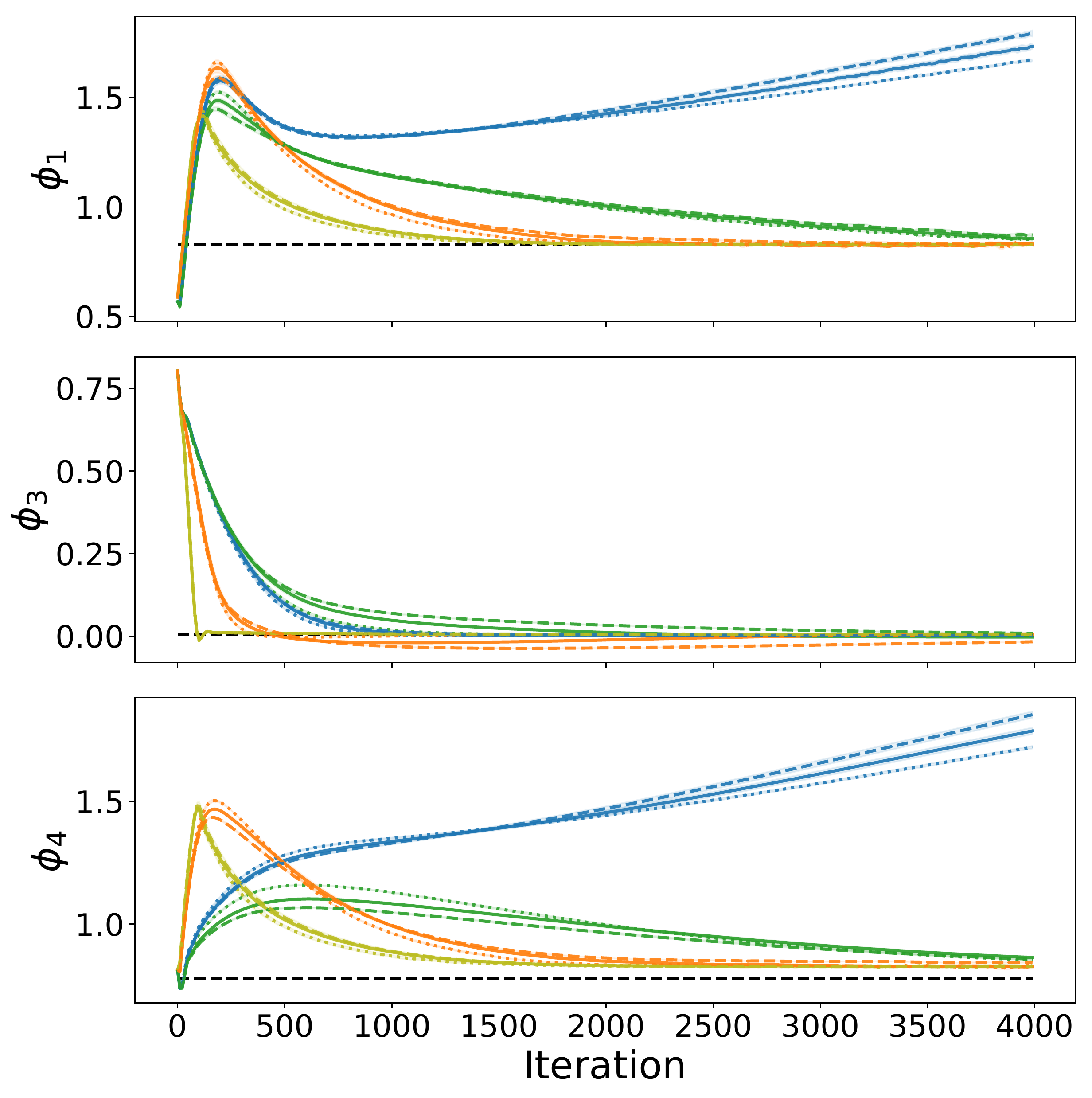}
    \end{subfigure}%
    \medskip
    \begin{subfigure}[b]{0.35\textwidth}
        \centering
        \includegraphics[width=\textwidth]{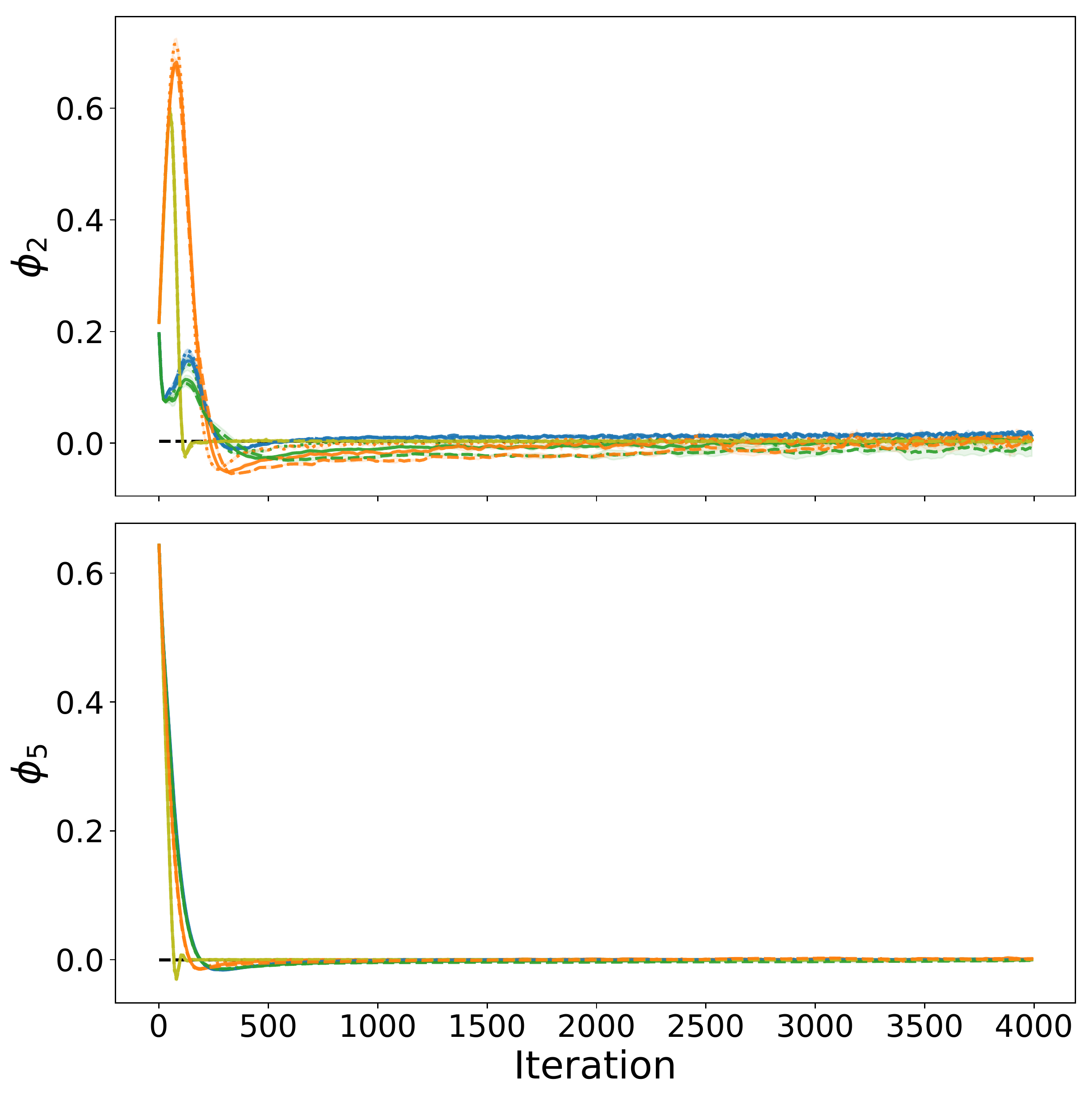}
    \end{subfigure}%
    \caption{Learning of generative and proposal parameters for LGSSM with different numbers of samples. \textit{Top left}: Negative marginal log-likelihoods on test data and gradient norms of all parameters. \textit{Top right}: Generative model parameters $\theta_1$ and $\theta_2$. \textit{Bottom left}: 3 proposal weights $\phi_1$, $\phi_3$ and $\phi_4$. \textit{Bottom right}: 2 bias parameters $\phi_2$ and $\phi_5$. Lines indicate the average of 3 random seed trainings and shaded areas for standard deviation.}\vspace{-3mm}
    \label{fig:learning_ssm_K_more}
\end{figure}

\begin{figure}[!htb]
    %\vspace{-4mm}
    \centering
    \begin{subfigure}[b]{0.35\textwidth}
        \centering
        \includegraphics[width=\textwidth]{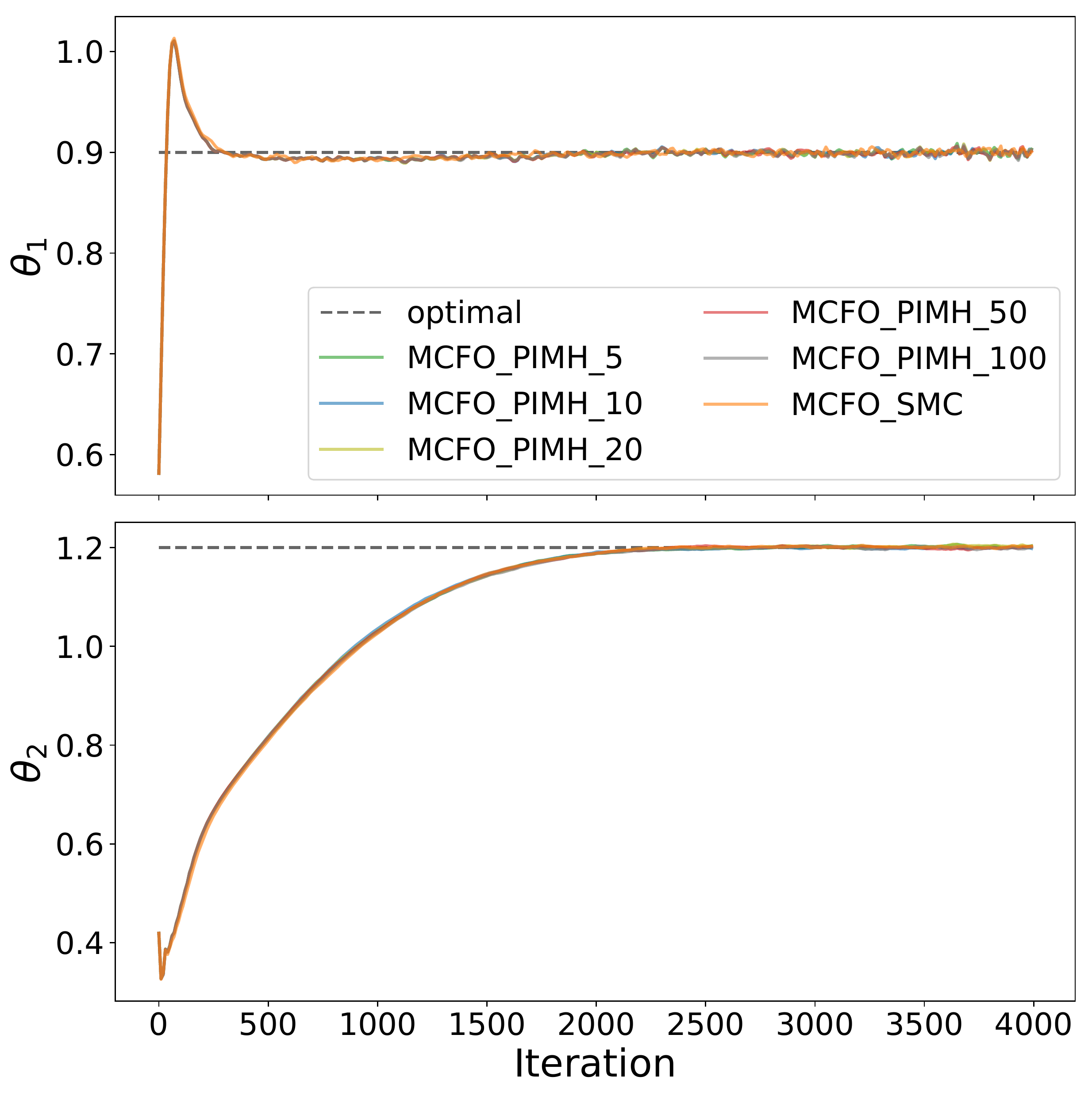}
    \end{subfigure}%
    \medskip
    \begin{subfigure}[b]{0.35\textwidth}
        \centering
        \includegraphics[width=\textwidth]{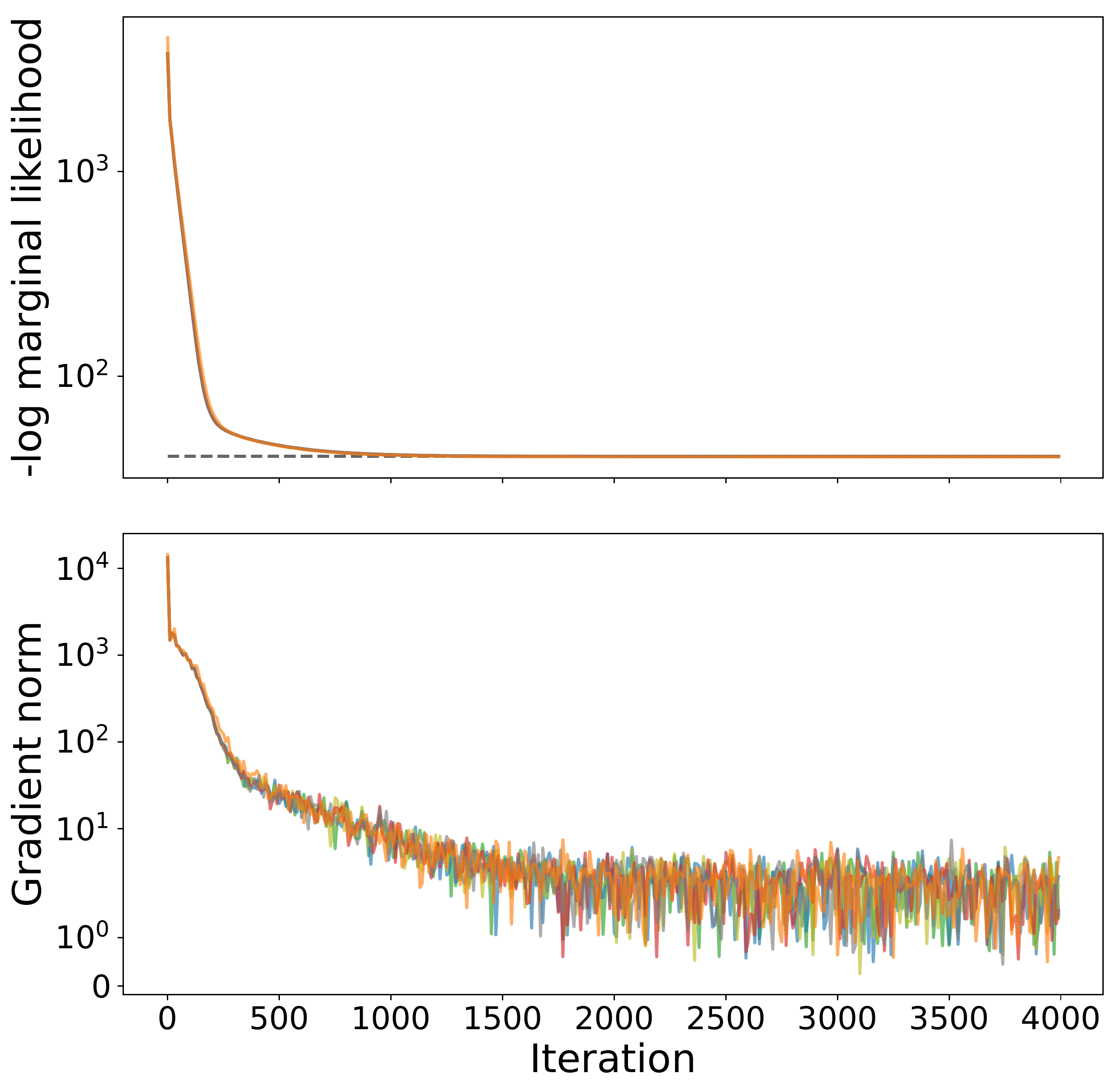}
    \end{subfigure}%
    \vfill
    \begin{subfigure}[b]{0.35\textwidth}
        \centering
        \includegraphics[width=\textwidth]{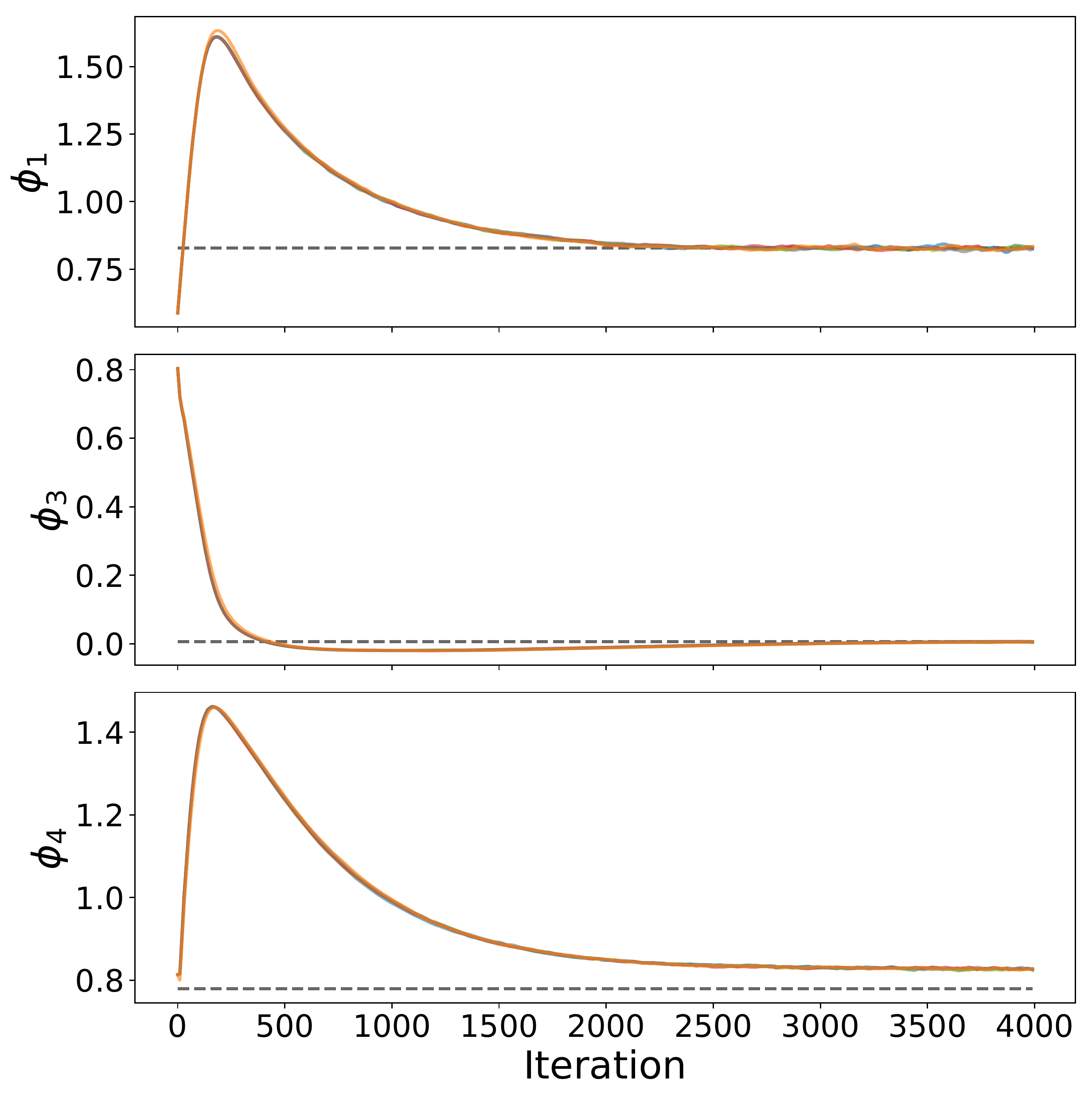}
    \end{subfigure}%
    \medskip
    \begin{subfigure}[b]{0.35\textwidth}
        \centering
        \includegraphics[width=\textwidth]{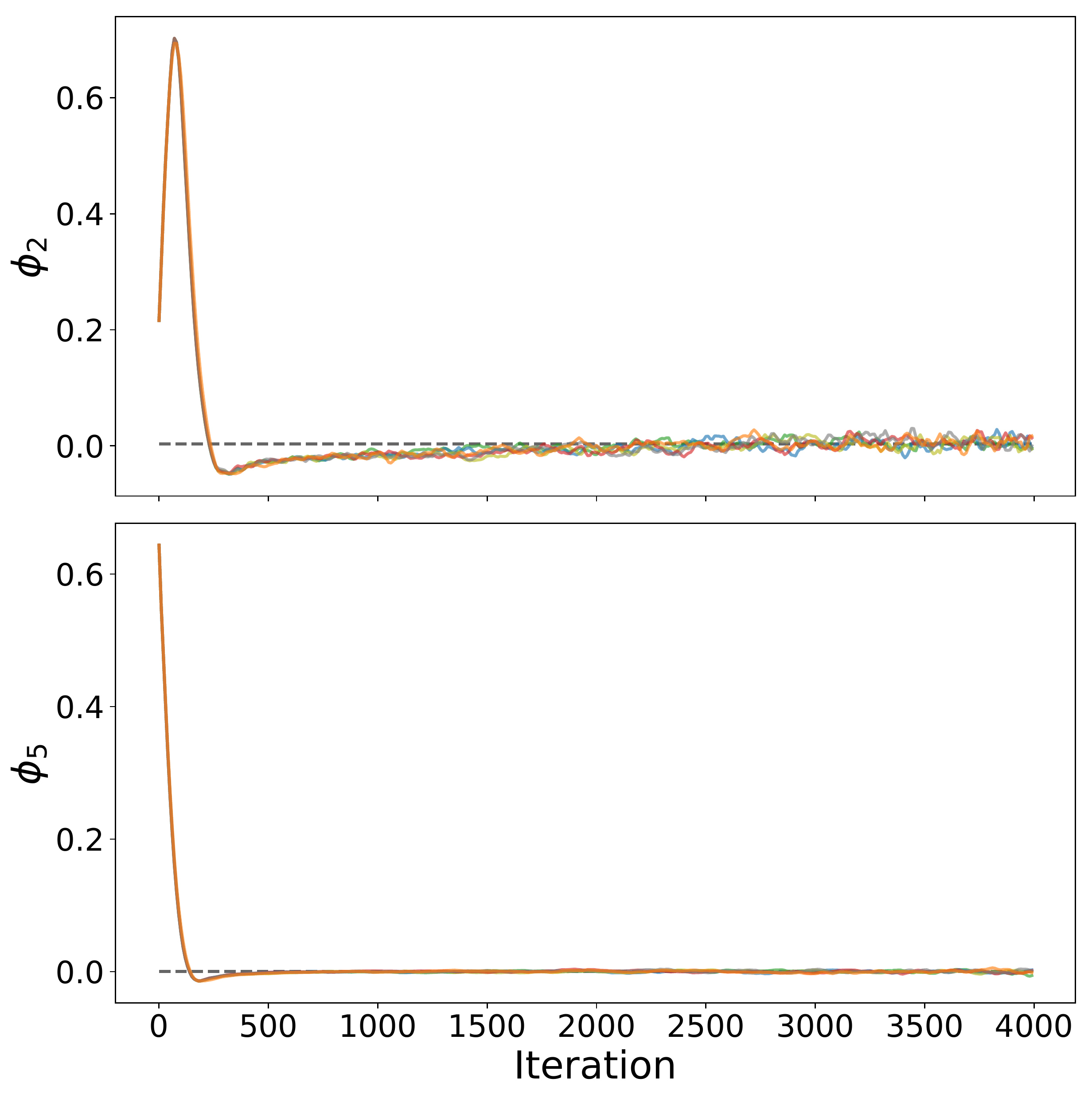}
    \end{subfigure}%
    \caption{\textit{Top left}: Generative model parameters $\theta_1$ and $\theta_2$. \textit{Top right}: Negative marginal log-likelihoods on test data and gradient norms of all parameters. \textit{Bottom left}: 3 proposal weights $\phi_1$, $\phi_3$ and $\phi_4$. \textit{Bottom right}: 2 bias parameters $\phi_2$ and $\phi_5$.  Lines indicate on different samplers and numbers of MH iterations.}\vspace{-3mm}
    \label{fig:learning_pmh}
\end{figure}

\subsubsection{Learning Generative and Proposal Parameters}
\label{sec:appendix_learning_ssm}

The training and test datasets are generated by LGSSM with parameters $\mu_0 = 0.5$, $\sigma_0 = 1$, $\theta_1 = 0.9$, $\theta_2 = 1.2$, $\Sigma_Q = 1.0$, $\Sigma_R = 0.01$. We use \textit{Adam} optimizer with default settings and the learning rate is constant as 0.01 for all methods. The numbers of samples for learning generative and proposal parameters are chosen from 10, 100, 1000. 

Figure \ref{fig:LGSSM_proposal_bias} shows the convergence of parameters, $\boldtheta$ and $\boldphi$, during the same training as Figure 2. MCFO and RWS converge faster than other methods, while IWAE and bootstrap fail to converge to the optimum of parameters. In addition to $K=100$, Figure \ref{fig:learning_ssm_K_more} shows the training when $K=10, 1000$. The convergences of parameters are similar for AESMC, RWS and MCFO with different $K$ while bootstrap and IWAE are more sensitive to $K$. More samples accelerate convergence marginally on $\boldtheta$ while making $\boldphi$ slower to converge. Figure \ref{fig:learning_pmh} shows the comparisons of MCFO-SMC and MCFO-PIMH, using gradient estimates by SMC and PIMH with different numbers of MH iterations. For the simple LGSSM, MCFO-PIMH converges slightly faster that the early overshoot is restricted, compared to MCFO-SMC, but increasing numbers of MH iterations does not make much difference on this simple learning task.

\begin{comment}
\begin{figure}[!htb]
    %\vspace{-4mm}
    \centering
    \begin{subfigure}[b]{0.6\textwidth}
        \centering
        \includegraphics[width=\textwidth]{figures/learning_lgssm/lgssm_Kmult/lgssm_theta.pdf}
        \caption{Generative model parameter $\theta_1$ and $\theta_2$}
    \end{subfigure}%
    \vfill
    \begin{subfigure}[b]{0.5\textwidth}
        \centering
        \includegraphics[width=\textwidth]{figures/learning_lgssm/lgssm_Kmult/lgssm_proposal.pdf}
        \caption{Proposal parameters $\boldphi$}
    \end{subfigure}%
    \caption{Learning generative and proposal parameters jointly}
    \label{fig:learning_ssm_K_more}
\end{figure}
\end{comment}

\subsection{Pendulum Video Sequences}
\label{sec:appendix_pendulum}
To evaluate MCFOs on more general state space models, we create a video sequence dataset of a single pendulum system as inspired by \citep{le2017auto}. Different to \citep{le2017auto}, the SSM defined here has more complex latent trajectories than a random walk. The latent state space model of a single pendulum system is defined with angle $\omega$ and angular velocity $\dot{\omega}$ by:
\begin{equation*}
    \begin{aligned}
        & \dot{\omega}_{t+1} = \frac{3}{ml^2} (mg  \frac{l}{2} \sin \varphi_t) \Delta t + \dot{\omega}_t, \\
        & \omega_{t+1} = \dot{\omega}_{t} \Delta t + \omega_t,\\
        & \begin{bmatrix}
            \dot{\omega}_{t+1}\\
            \varphi_{t+1}
        \end{bmatrix}
        = g(\begin{bmatrix}
        \dot{\omega}_t\\
        \omega_t
        \end{bmatrix}) + \mathcal{N}(0, 0.1^2),
    \end{aligned}
\end{equation*}
where $\Delta t = 0.05$, $g=10$, $m=1$, $l=1$ is specified in \texttt{gym} to create video sequences by projecting the angles of pendulum, $\omega$ to $32\times32$ grayscale images as the Bernoulli distribution of binary observations. 

\begin{figure}
    \centering
    \begin{subfigure}[t]{\textwidth}
        \includegraphics[width=0.32\textwidth]{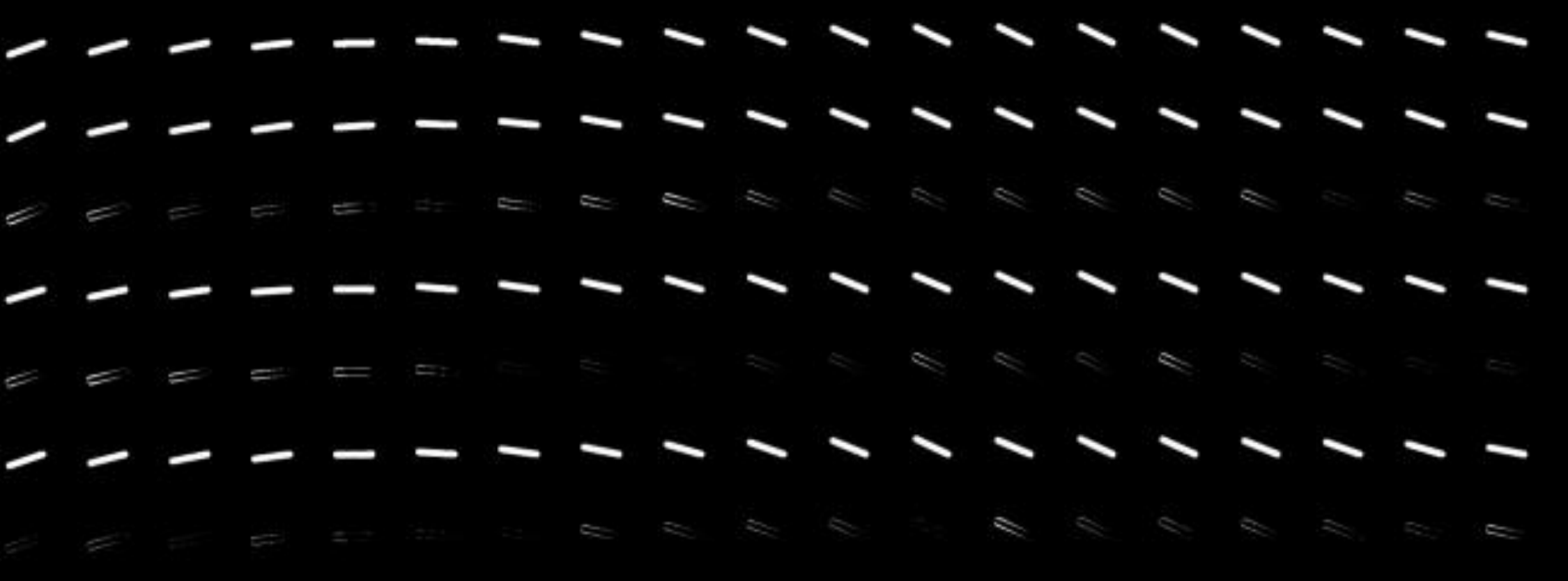}
        \includegraphics[width=0.32\textwidth]{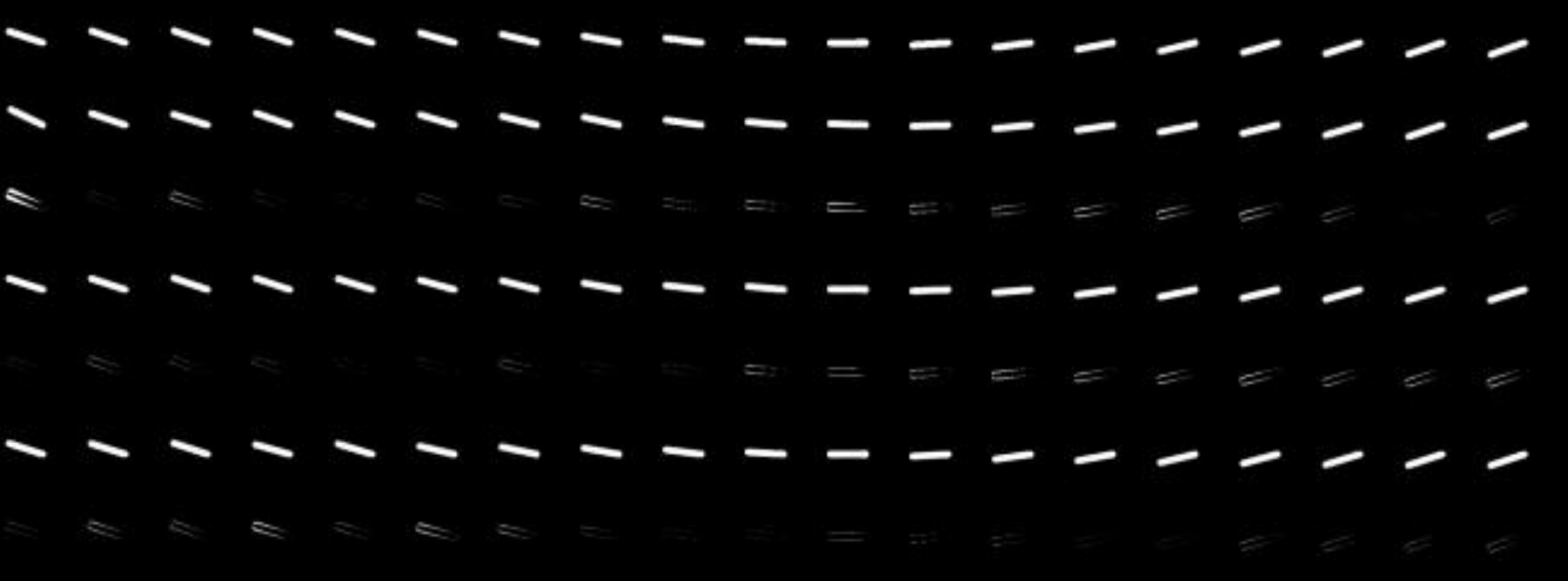}
        \includegraphics[width=0.32\textwidth]{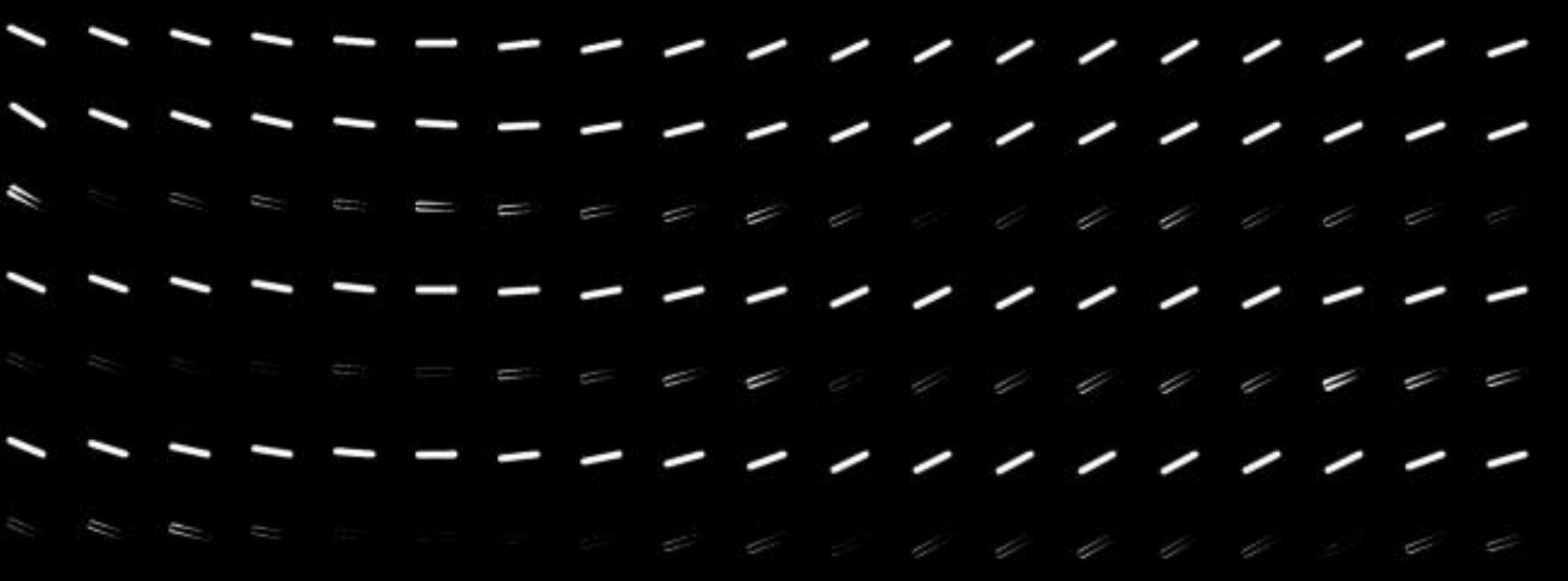}
    \end{subfigure}
    \vfill
    \begin{subfigure}[t]{\textwidth}
        \includegraphics[width=0.32\textwidth]{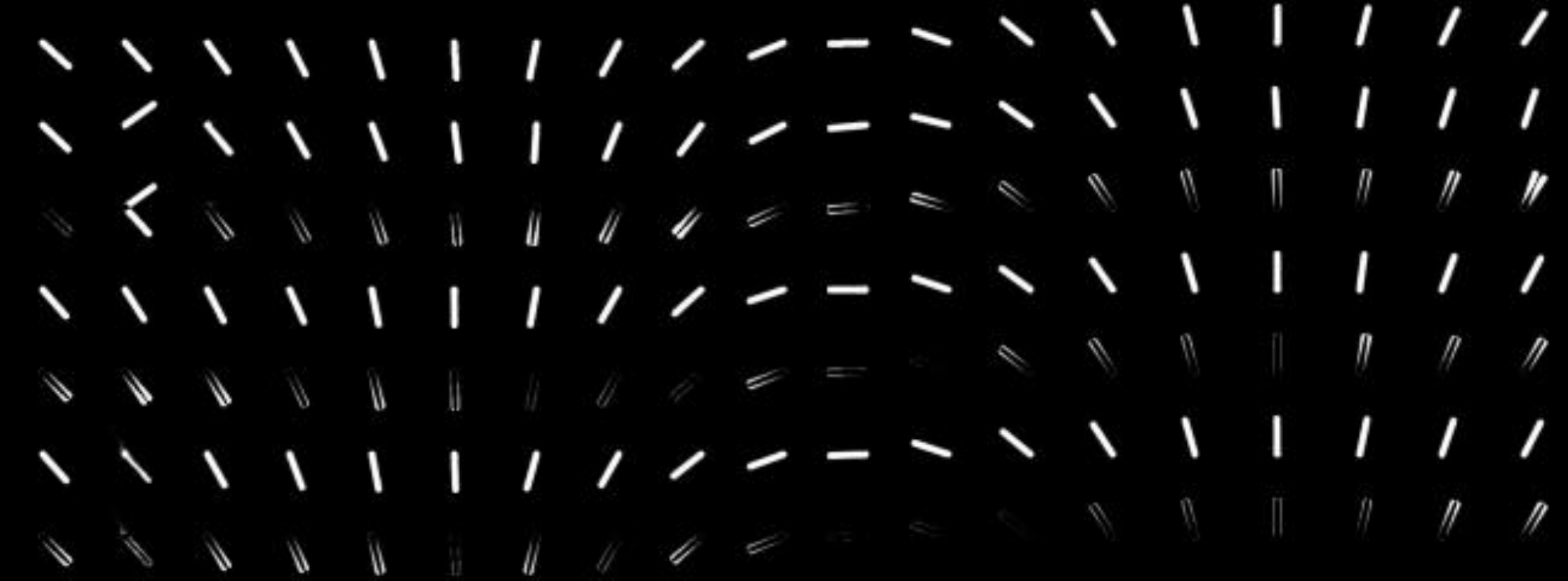}
        \includegraphics[width=0.32\textwidth]{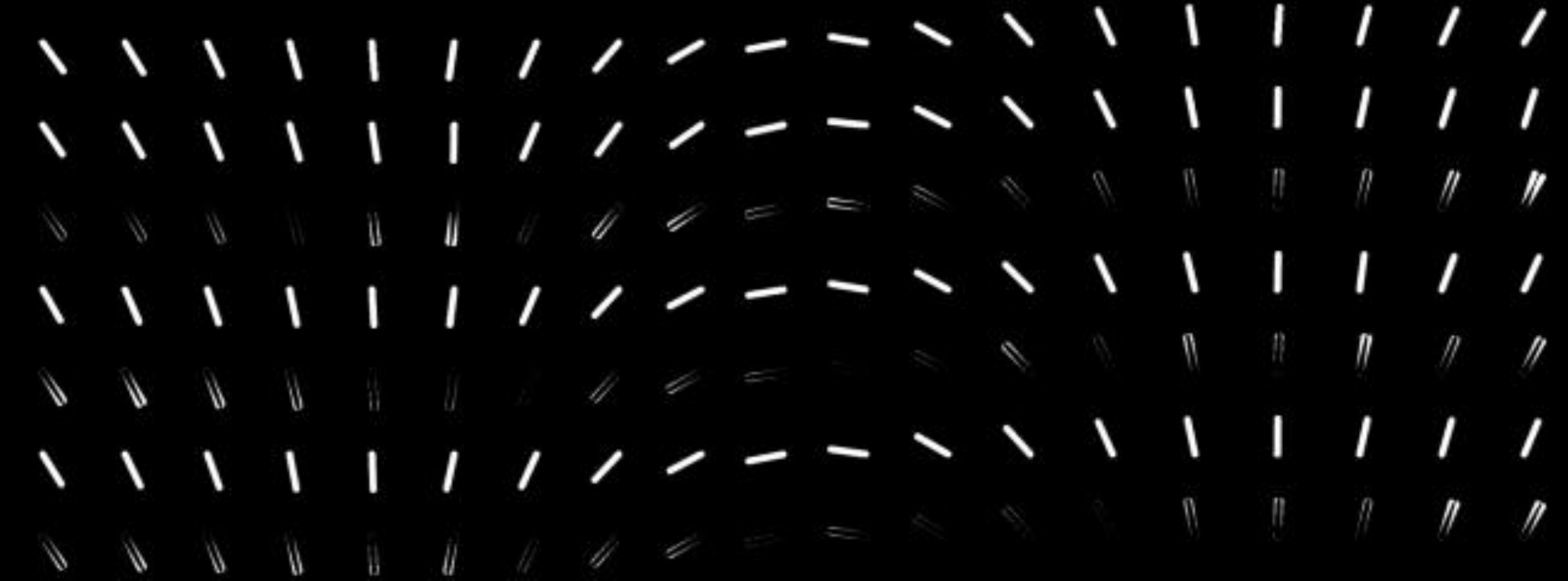}
        \includegraphics[width=0.32\textwidth]{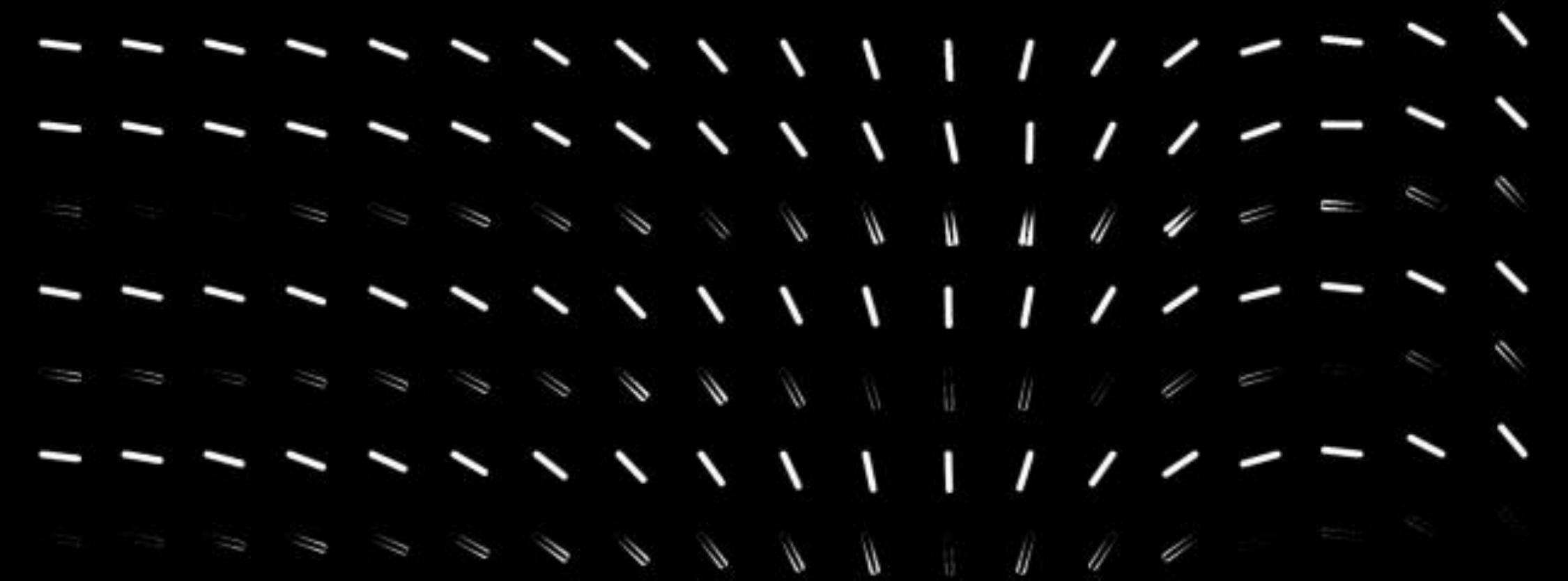}
    \end{subfigure}
    \vfill
    \begin{subfigure}[t]{\textwidth}
        \includegraphics[width=0.32\textwidth]{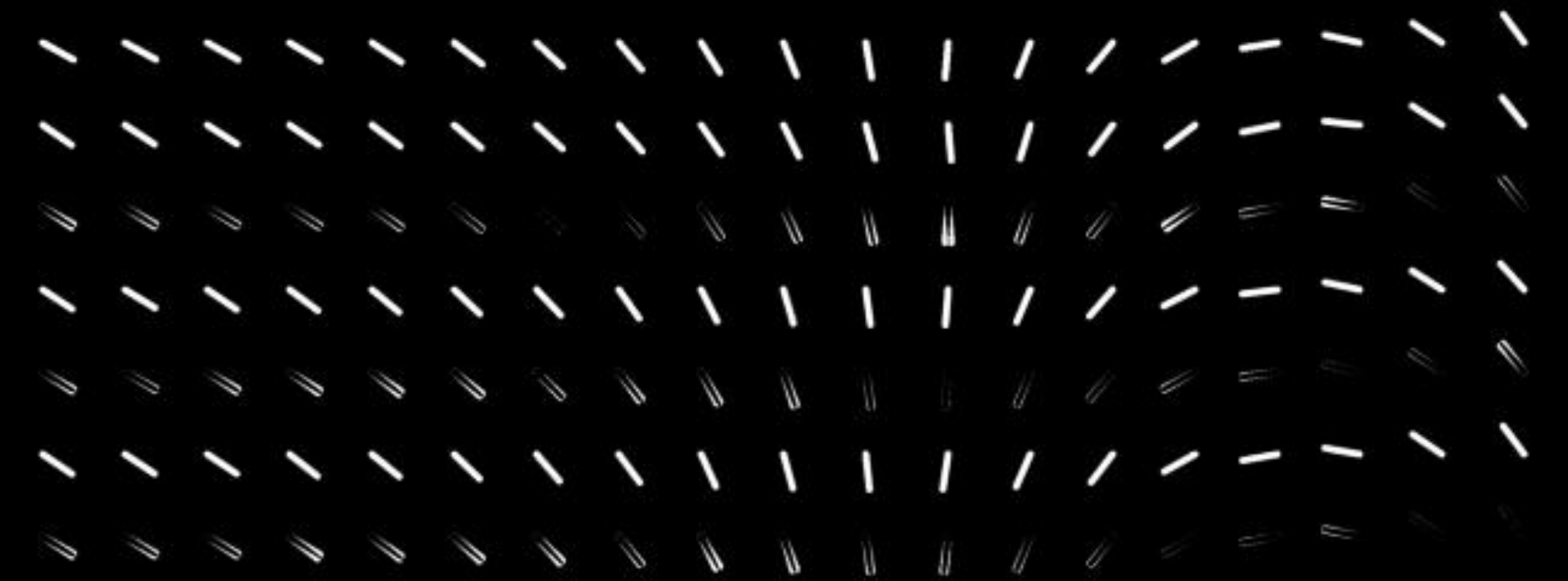}
        \includegraphics[width=0.32\textwidth]{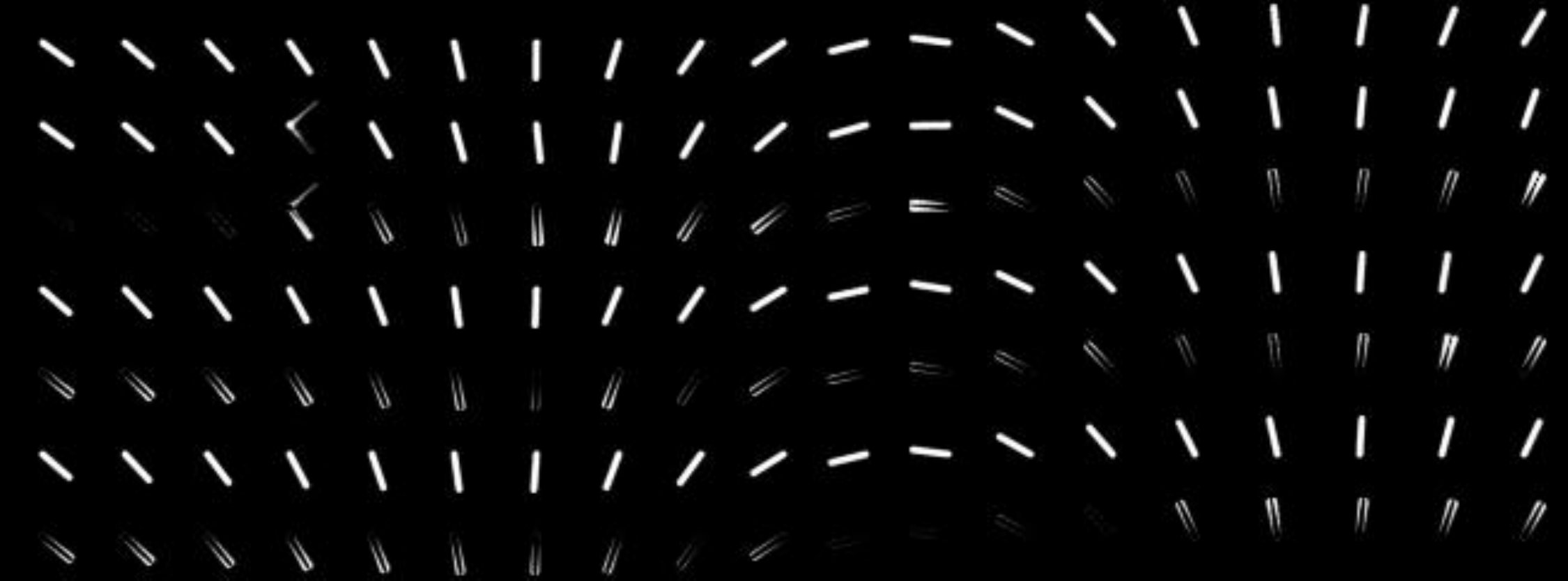}
        \includegraphics[width=0.32\textwidth]{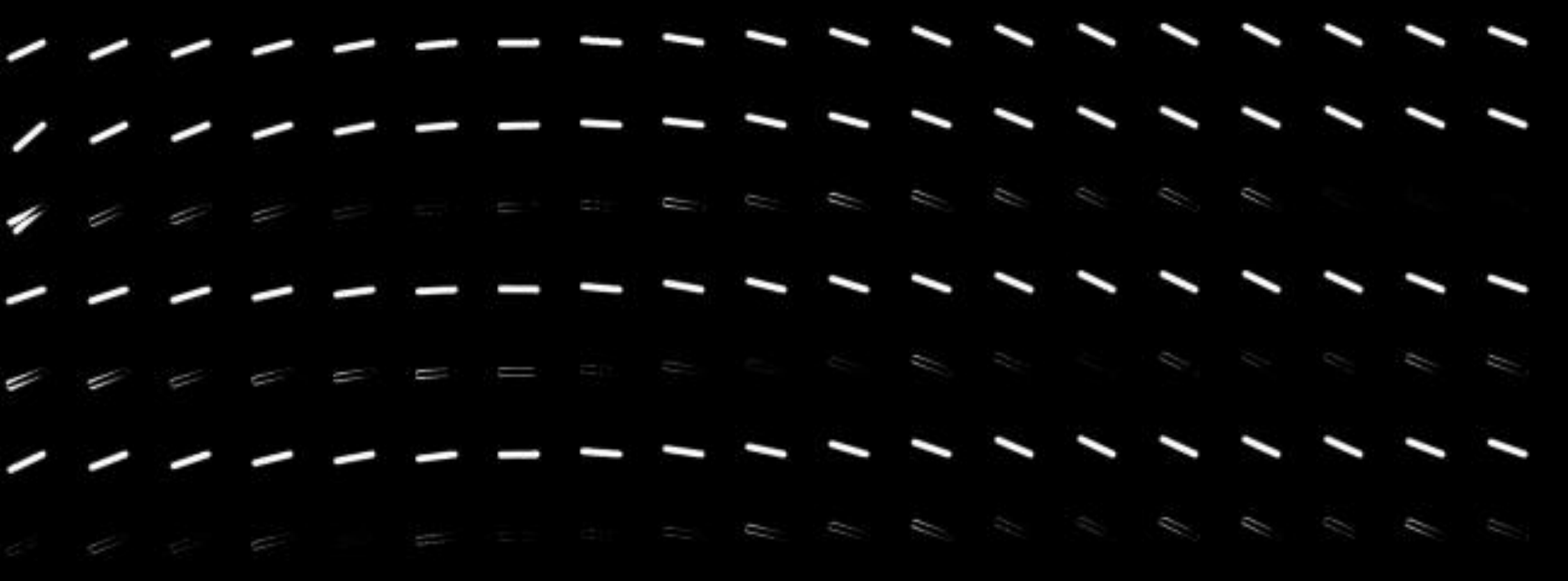}
    \end{subfigure}
    \vfill
    \begin{subfigure}[t]{\textwidth}
        \includegraphics[width=0.32\textwidth]{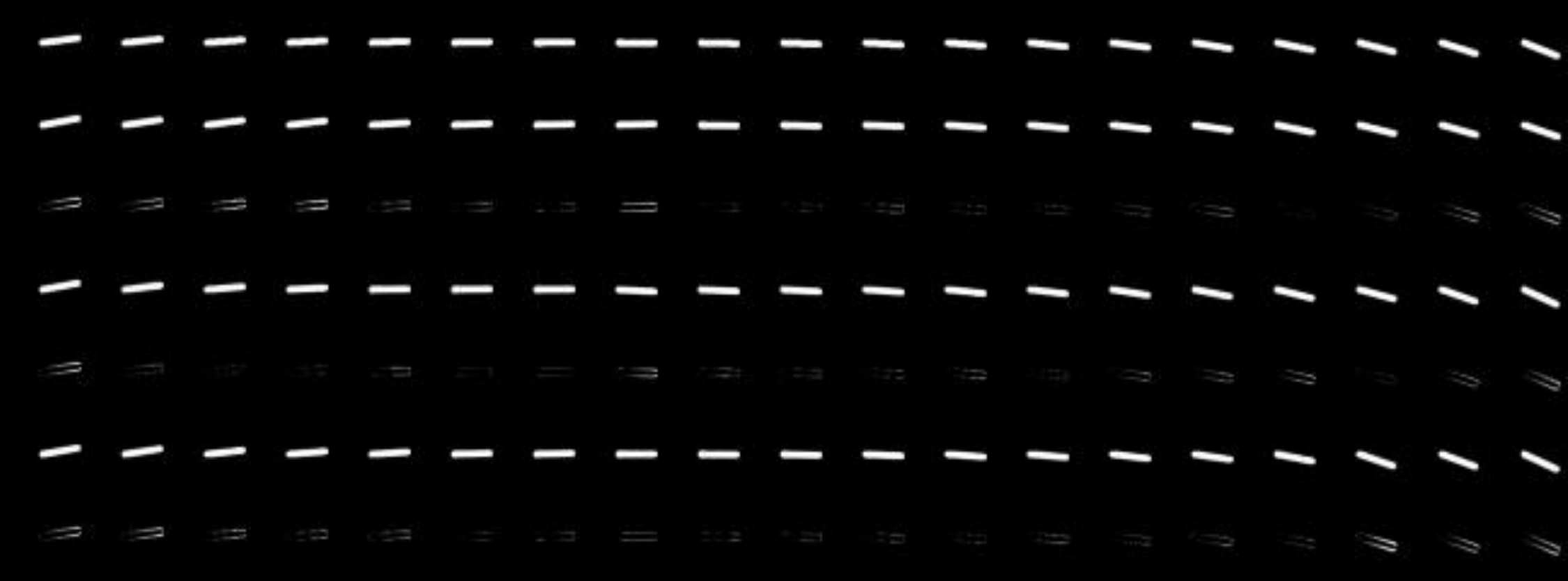}
        \includegraphics[width=0.32\textwidth]{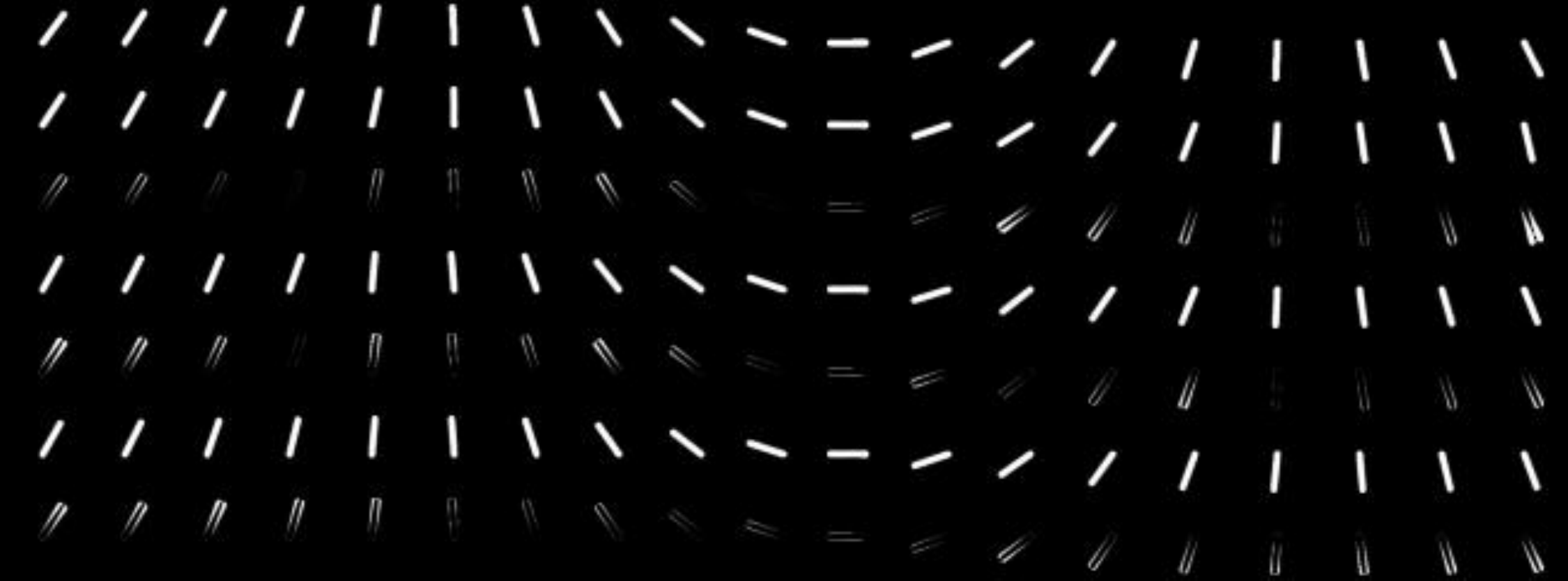}
        \includegraphics[width=0.32\textwidth]{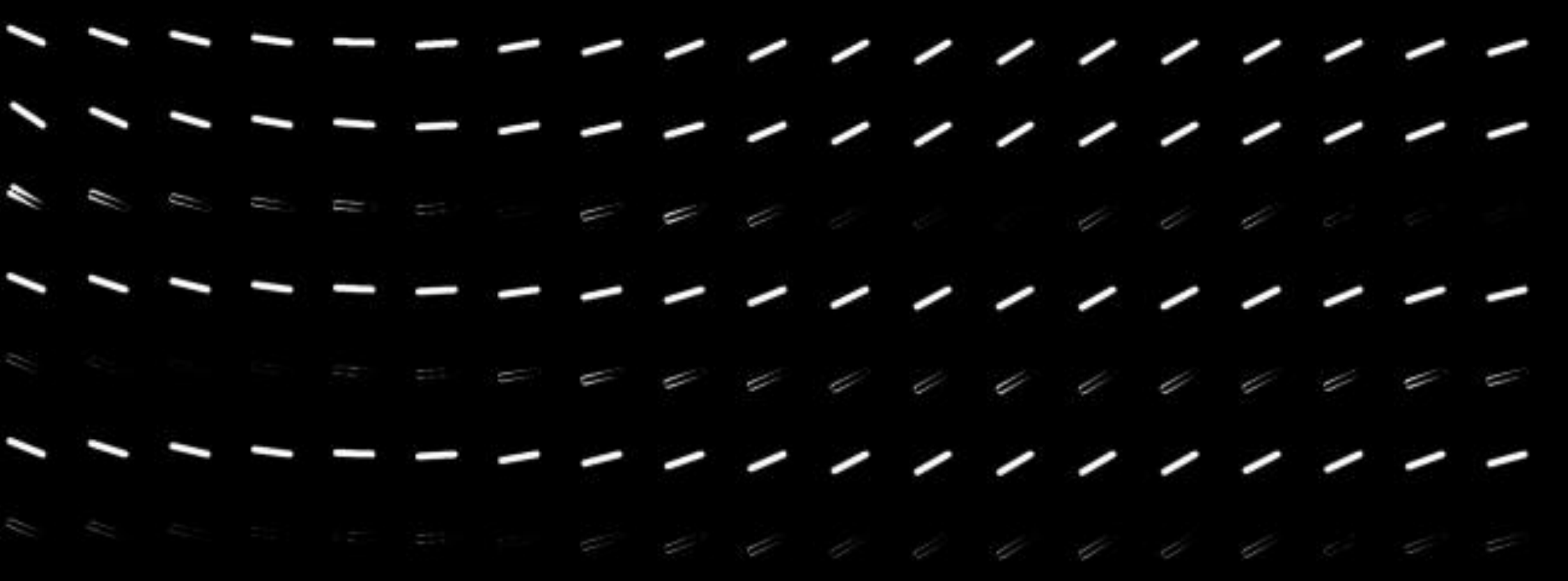}
    \end{subfigure}
    \caption{More sequence examples of filtering and one-step prediction by AESMC, MCFO-SMC and MCFO-PIMH with $K = 100$. The first row is Bernoulli observations, followed by the one-step predictions and the absolute differences of predictions and observations of AESMC, MCFO-SMC and MCFO-PIMH.}
    \label{fig:more_example_filtering_prediction}
\end{figure}

\begin{figure}
    \centering
    \begin{subfigure}[b]{0.4\textwidth}
        \centering
        \includegraphics[width=\textwidth]{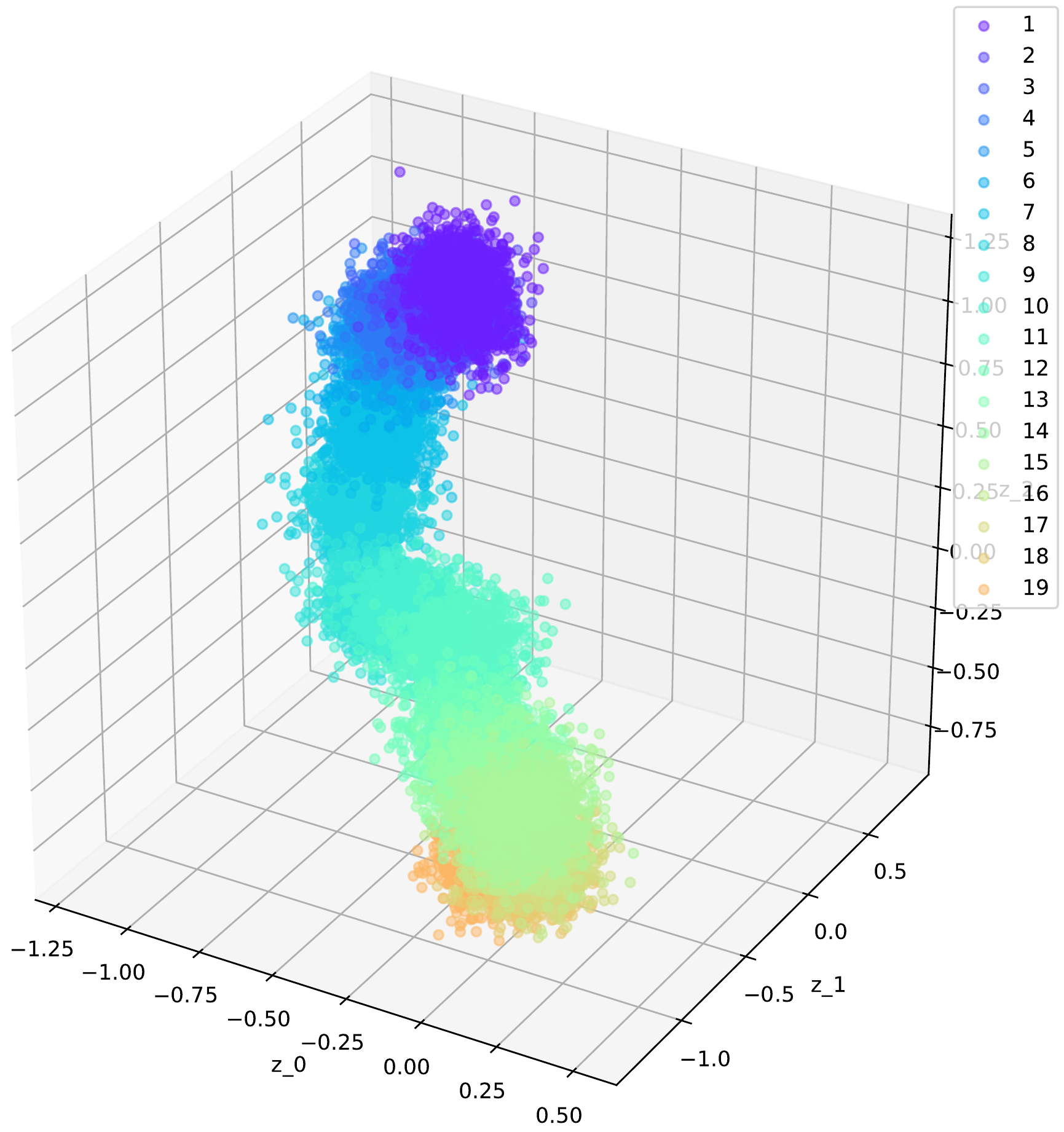}
    \end{subfigure}%
    \medskip
    \begin{subfigure}[b]{0.4\textwidth}
        \centering
        \includegraphics[width=\textwidth]{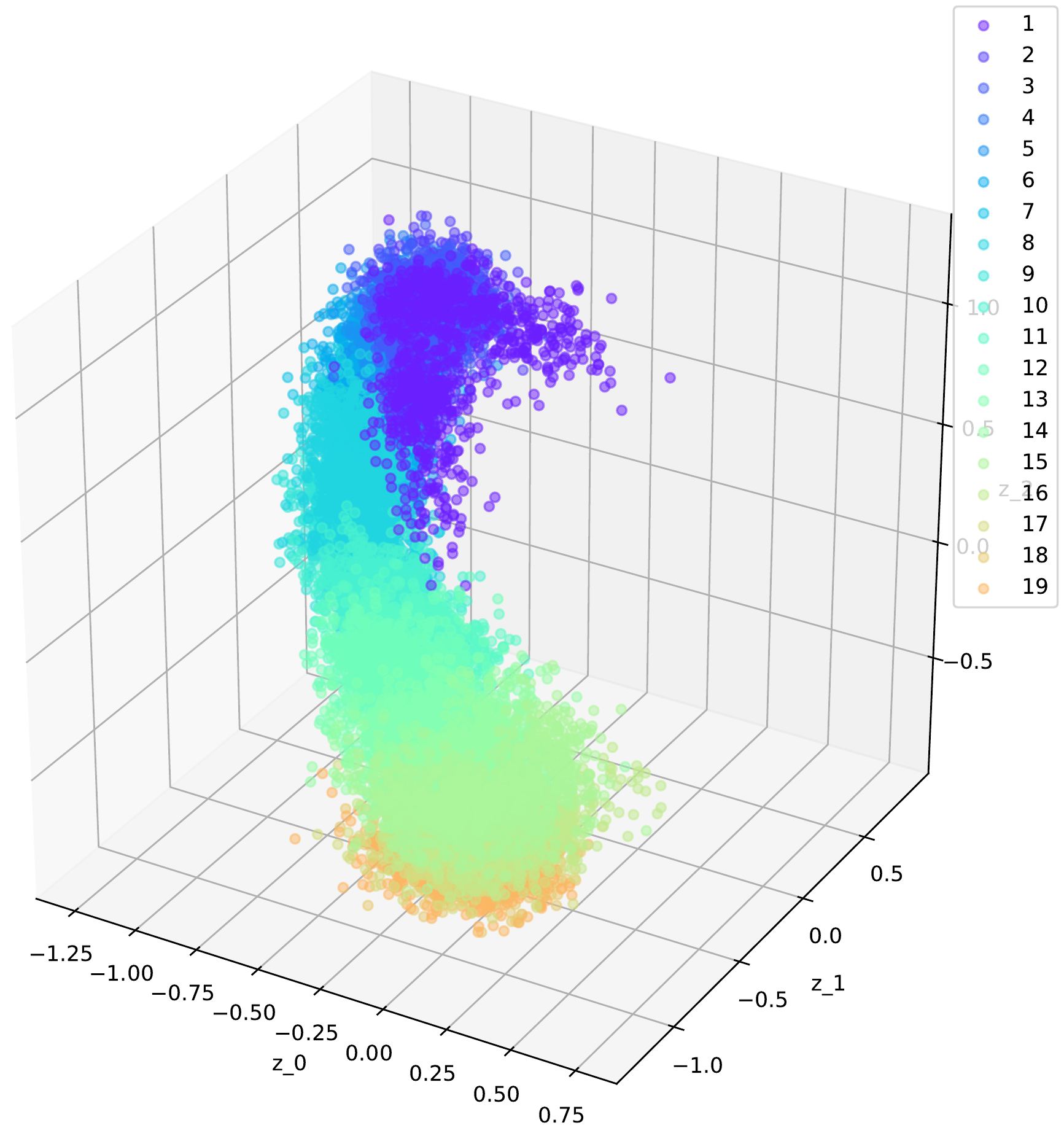}
    \end{subfigure}%
    \vfill
    \begin{subfigure}[b]{0.4\textwidth}
        \centering
        \includegraphics[width=\textwidth]{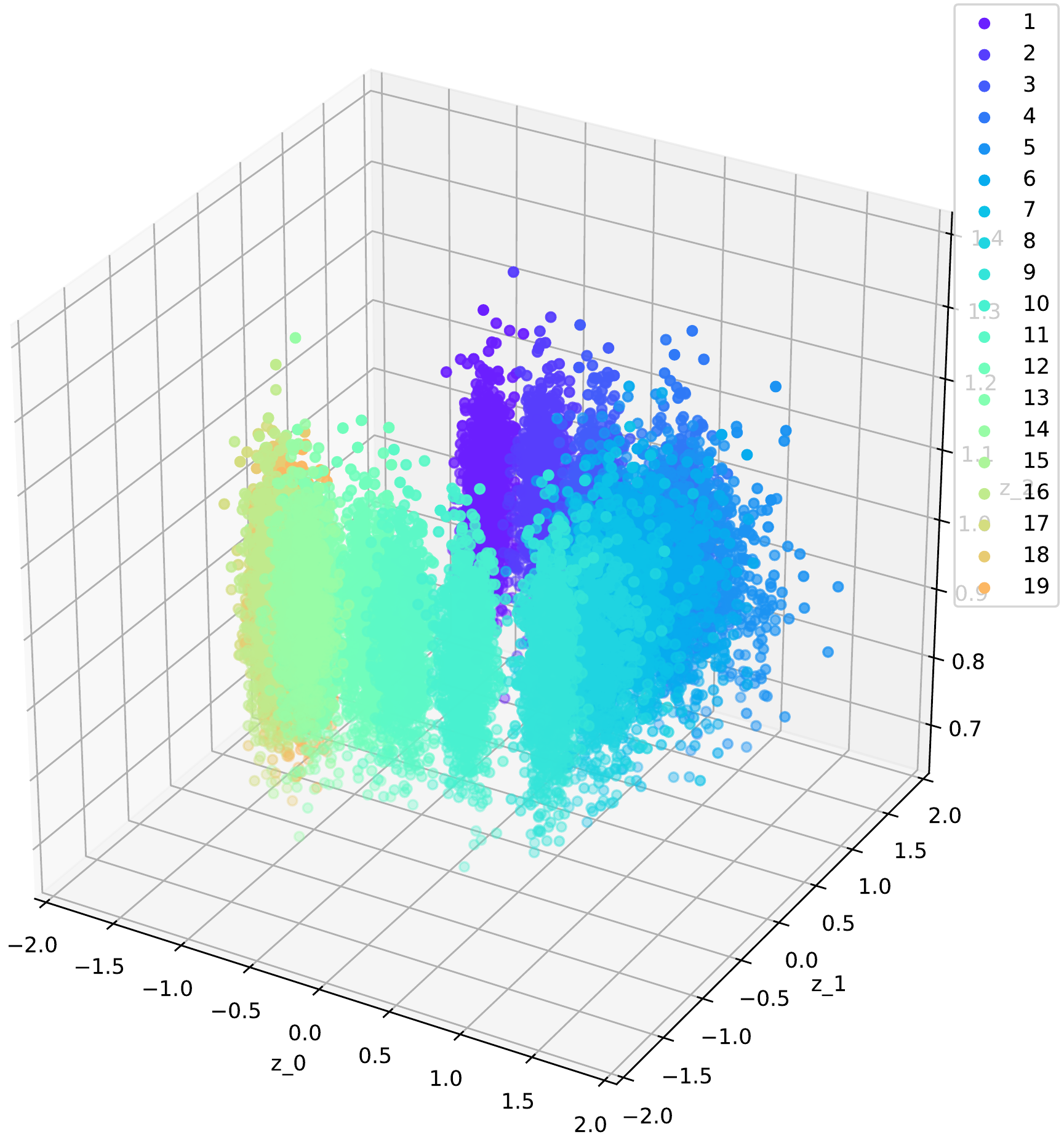}
    \end{subfigure}%
    \medskip
    \begin{subfigure}[b]{0.4\textwidth}
        \centering
        \includegraphics[width=\textwidth]{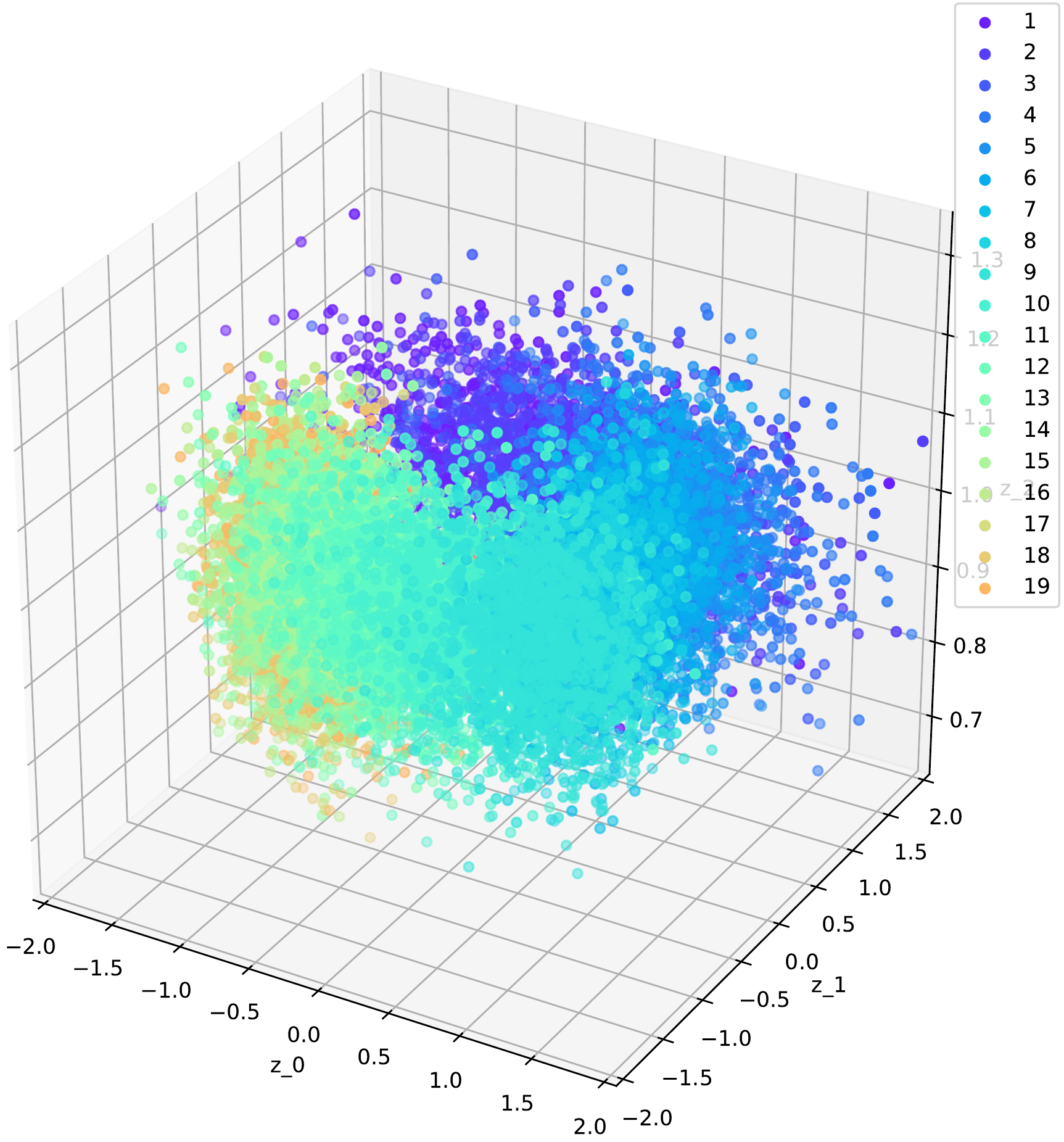}
    \end{subfigure}%
    \caption{Visualization of 1000 samples of filtering and one-step ahead prediction in latent space; \textit{left} is filtering inference samples and \textit{right} is one-step ahead prediction samples. \textit{Top}: AESMC-100. \textit{Bottom}: MCFO-SMC-100.}
    \label{fig:latent_filtering_prediction}
\end{figure}

All training by different objectives use the same generative and proposal model definitions. The latent dimension is set as 3 and all distributions on the latent variables are assumed to be Gaussian. The transition model $p_{\boldtheta}(\mathbf{z}_t|\mathbf{z}_{t-1})$ is specified as $ \mathcal{N}(\mathbf{z}_t; \boldsymbol{\mu}_{\boldtheta}(\mathbf{z}_{t-1}), \boldsymbol{\sigma}_{\boldtheta}^2(\mathbf{z}_{t-1}))$ with its mean and variance parameterized by Multi-layer perceptrons (MLPs) with 1 hidden layer of 64 units. The observation model $p_{\boldtheta}(\mathbf{x}_t|\mathbf{z}_{t})$ is assumed to be factorized Bernoulli distribution $\text{Bern}(\mathbf{x}_t; \boldsymbol{\mu}_{\boldtheta}(\mathbf{z}_{t}))$, its mean is parameterized by a 4-hidden-layer MLP with 16, 64, 512, 2048 units. Instead of mixing  $\mathbf{z}_{t-1}$ and $\mathbf{x}_t$ at first layer of NN, the proposal model $q_{\boldphi}(\mathbf{z}_t|\mathbf{z}_{t-1}, \mathbf{x}_t) \sim \mathcal{N}(\mathbf{z}_t; \boldsymbol{\mu}_{\boldphi}(\mathbf{z}_{t-1}, \mathbf{x}_t), \boldsymbol{\sigma}_{\boldphi}^2(\mathbf{z}_{t-1}, \mathbf{x}_t))$, first encodes $\mathbf{x}_t$ to low dimensional representation by a 2-hidden-layer MLP with 128, 32 units, then concatenates it with $\mathbf{z}_{t-1}$ to map to mean and variance by a 1-hidden-layer MLP with 32 units. 

MCFO-SMC uses SMC while MCFO-PIMH uses PIMH with 5 MH iterations in gradient estimation. \textit{Adam} optimizer with default parameter settings is used for all objectives. The initial learning rate is $1e-2$ and it decays every 10 iterations by a factor of 0.998 with the lowest rate as $1e-4$. For all MCFO-SMC and MCFO-PIMH, the number of samples to optimize proposal models is set to 10 while the numbers of samples for generative models are experimented by $\{10, 20, 50, 100, 200, 500, 1000 \}$.

For evaluation of learned generative and proposal models, though the estimate negative log-likelihood is a common performance metric to report as in Section 4.2 and 4.4, it is not appropriate on this video sequences since it does not reflect how well the uncertainty is predicted on the observations comparing to the Bernoulli mean of observations. For instance, considering a binary pixel is sampled as 1 from a Bernoulli with mean as 0.6 and two models that predict mean as 0.7 and 0.9, the cross-entropy in the log-likelihood will favor the later model, however, first model predicts mean more correctly. To address this, we instead implement another widely used one-step-ahead prediction error, defined as the sum of L2 norms between the ground truth and one-step ahead predictions, $\sum_{t=2}^T \Vert\tilde{\mathbf{x}}_t^{pred}-\mathbf{x}_t^{GT}\Vert_2$. The prediction in observation space, $\tilde{\mathbf{x}}_t^{pred}$, is defined by the Bernoulli mean of the predicted mean in latent space, $\Tilde{\mathbf{z}}^{pred}_t \approx \mathbb{E}_{p(\mathbf{z}_t|\mathbf{z}_{t-1})p(\mathbf{z}_{t-1}|\mathbf{x}_{1:t-1})}[\mathbf{z}_t]$, which is estimated by the transition model and the inference by SMC.

Figure \ref{fig:more_example_filtering_prediction} shows more examples of filtering reconstructions and one-step ahead predictions by MCFO-SMC, MCFO-PIMH and AESMC as Figure 3. Qualitatively, MCFO-PIMH and MCFO-SMC have relatively lower prediction deviation from ground truth, compared to AESMC. BBesides predictions in observation space, we investigate the latent samples from the learned posterior and predict one-step ahead by learned transition model. Figure \ref{fig:latent_filtering_prediction} shows that the manifold of latent variables by MCFO-SMC is more regulated than that of AESMC since MCFO implicitly regularize to learn more sample efficient proposal models and thus simpler generative models. For MCFO-SMC, the representation of the pendulum system angle is mainly encoded in the dimensions $z_0$ and $z_1$ while this disentanglement cannot be observed in AESMC. Table \ref{table:pendulum_evaluation_morek} shows the evaluations on the models learned by more number of samples $K$ as Table 2. When K increases to 200, 500, 1000, no statistically significant difference is observed, compared to 100. 

\begin{table*}[!htb]
  \centering
  \small
    \begin{tabular}{lllll}
    %\cmidrule(r){1-2}
    Method  & $K$    %& Estimated NLL 
    & Prediction & ESS\\
    \hline
    AESMC &\multirow{3}{*}{200}%&  \textbf{592.30 $\pm$ 0.15}  
    & 37.65 $\pm$ 0.23 & 80.64 $\pm$ 0.97 \\
    MCFO-SMC & %&  613.98 $\pm$ 0.18 
    & 34.24 $\pm$ 1.56 & 128.97 $\pm$ 2.06 \\
    MCFO-PIMH & %&  612.06 $\pm$ 0.19  
    & \textbf{33.80 $\pm$ 0.99} & \textbf{129.85 $\pm$ 1.75} \\
    \hline
    AESMC &\multirow{3}{*}{500}%&  \textbf{592.11 $\pm$ 0.14}  
    & 36.93 $\pm$ 0.27 & 78.51 $\pm$ 1.09 \\
    MCFO-SMC & %&  613.71 $\pm$ 0.16 
    & 34.21 $\pm$ 1.58 & 128.80  $\pm$ 2.10 \\
    MCFO-PIMH & %&  612.07 $\pm$ 0.20 
    & \textbf{33.74 $\pm$ 0.89} & \textbf{130.01 $\pm$ 2.11} \\
    \hline
    AESMC &\multirow{3}{*}{1000}%&  \textbf{591.91 $\pm$ 0.17}  
    & 36.73 $\pm$ 0.29 & 79.89 $\pm$ 1.89 \\
    MCFO-SMC & %&  613.48 $\pm$ 0.19 
    & 34.18 $\pm$ 1.72 & \textbf{126.99  $\pm$ 3.26} \\
    MCFO-PIMH & %&  612.05 $\pm$ 0.23 
    & \textbf{33.62 $\pm$ 1.06} & 125.08 $\pm$ 2.37
  \end{tabular}
  \caption{One-step prediction errors and ESS on the test set of generative and proposal models learned by AESMC, MCFO-SMC, MCFO-PIMH with $K=200,500,1000$, evaluated by SMC with 1000 samples averaged over last 1000 iterations.}
  \label{table:pendulum_evaluation_morek}
\end{table*}

\subsection{Polyphonic Music Sequences}
\label{sec:appendix_polyphonic_music}

To evaluate MCFOs on non-Markovian data of high dimensions and complex temporal dependencies, we choose four polyphonic music datasets \citep{boulanger2012modeling}. Each dataset contains a varied length of sequences, and we pre-process all notes to 88-dimensional binary vectors that space the whole range of piano from A0 to C8. To model temporal dependencies of polyphonic music sequences, the same VRNN \citep{chung2015recurrent} is implemented for MCFO-SMC, MCFO-PIMH, FIVO and IWAE reported in Table 3.

All distributions of latent and hidden states, $\mathbf{z}$ and $\mathbf{h}$, are assumed to be Gaussian, while the observation distributions are factorized Bernoulli distributions of 88 dimensions. Both dimensions of hidden and latent variables are set as 64. VRNN uses Long Short Term Memory (LSTM) to update the deterministic hidden state with internal memory states by $(\mathbf{h}_t, \mathbf{c}_t)  = \mathbf{LSTM}(\mathbf{z}_{t}, \mathbf{x}_t, \mathbf{h}_{t-1}, \mathbf{c}_{t-1})$. With the hidden states, the generation process of VRNN is factorized as $\prod_{t=1}^T p_{\boldtheta}(\mathbf{z}_t|\mathbf{h}_{t-1}) p_{\boldtheta}(\mathbf{x}_t|\mathbf{z}_t, \mathbf{h}_{t-1})$. The transition $p_{\boldtheta}(\mathbf{z}_t|\mathbf{h}_{t-1})$ and emission $p_{\boldtheta}(\mathbf{x}_t|\mathbf{z}_t, \mathbf{h}_{t-1})$ are parameterized by MLPs with a hidden layer of the same size as the dimension of latent variables. A function, $\varphi(\cdot)$, is implemented by an MLP to better encode the observations and the proposal, $q_{\phi}(\mathbf{z}_t|\mathbf{h}_{t-1},\varphi(\mathbf{x}_t))$, models $\mathbf{h}_{t-1}$, the encoded history of observations $\mathbf{x}_{1:t-1}$, and encoded new observation $\varphi(\mathbf{x}_t)$ by a MLP with a hidden layer of the same size as the number of latent variables. 

All the trainings of VRNN by MCFO-SMC, MCFO-PIMH, FIVO, IWAE uses \textit{Adam} optimizer with default parameter settings. The initial learning rate is set as $1e-4$ and decays to the lowest rate as $1e-6$.

\end{document}